\documentclass{article}
\PassOptionsToPackage{numbers,compress}{natbib}
\usepackage[final]{neurips_2026}

\usepackage{microtype}
\usepackage{graphicx}
\usepackage{booktabs}
\usepackage{hyperref}
\usepackage{url}
\usepackage{amsmath,amssymb,amsfonts,mathtools}
\usepackage{amsthm}
\usepackage{multirow}
\usepackage{subcaption}
\usepackage[capitalize,noabbrev]{cleveref}
\crefname{algocf}{Algorithm}{Algorithms}
\usepackage{xspace}
\usepackage{xcolor}
\usepackage[ruled,vlined,linesnumbered]{algorithm2e}
\usepackage{float}
\usepackage{bbm}
\usepackage{minitoc}
\usepackage{titletoc}
\newcommand\DoToC{%
  \startcontents
  \printcontents{}{1}{\textbf{Table of Contents}\vskip3pt\hrule\vskip5pt}
  \vskip3pt\hrule\vskip5pt
}
\usepackage{enumitem}
\setlist{noitemsep,topsep=2pt,parsep=0pt,partopsep=0pt}

\hypersetup{hidelinks}
\usepackage[T1]{fontenc}
\usepackage[utf8]{inputenc}
\usepackage{amsmath,amssymb,mathtools}
\usepackage{enumitem}
\usepackage{parskip}
\usepackage{caption}

\setlist[itemize]{leftmargin=1.2em,itemsep=0.2em,topsep=0.3em}

\newcommand{\Qbar}{\overline{Q}}
\newcommand{\Vhat}{\widehat{V}}

\newcommand{\dir}{\mathrm{Dir}}

\newcommand{\EE}{\mathbb{E}}
\newcommand{\PP}{\mathbb{P}}

\DeclareMathOperator*{\argmin}{arg\,min}

\DeclareMathOperator*{\plim}{plim}
\SetKwFunction{SelectAction}{SelectAction}
\SetKwFunction{SimulateV}{SimulateV}
\SetKwFunction{SimulateQ}{SimulateQ}
\SetKwFunction{Rollout}{Rollout}
\SetKwFunction{Sel}{SelectAction}
\SetKwFunction{SimV}{SimulateV}
\SetKwFunction{SimQ}{SimulateQ}
\SetKwProg{Fn}{}{:}{}
\DontPrintSemicolon
\SetAlgoNlRelativeSize{0}
\theoremstyle{plain}
\newtheorem{theorem}{Theorem}[section]
\newtheorem{proposition}[theorem]{Proposition}
\newtheorem{lemma}[theorem]{Lemma}
\newtheorem{corollary}[theorem]{Corollary}
\theoremstyle{definition}
\newtheorem{definition}[theorem]{Definition}
\newtheorem{assumption}[theorem]{Assumption}
\theoremstyle{remark}
\newtheorem{remark}[theorem]{Remark}

\newcounter{theorembox}

\DeclareMathOperator{\CVaR}{CVaR}
\DeclareMathOperator{\VaR}{VaR}
\newcommand{\CVaRdisc}{\CVaR}
\newcommand{\bE}{\mathbb{E}}
\newcommand{\bP}{\mathbb{P}}

\newcommand{\algcatso}{\texttt{CATSO}\xspace}
\newcommand{\algpatso}{\texttt{PATSO}\xspace}

\newcommand{\maybeincludegraphics}[2][width=\linewidth]{%
  \IfFileExists{#2}%
    {\includegraphics[#1]{#2}}%
    {\fbox{\begin{minipage}[c][2.6cm][c]{0.88\linewidth}%
        \centering\tiny
        \textbf{Figure missing}\\[3pt]%
        \texttt{\path{#2}}%
     \end{minipage}}}%
}

\title{Beyond Expected Values: Risk-Sensitive Planning with Distributional Monte-Carlo Tree Search}

\author{
Quynh Trinh$^{*,1}$ \qquad
Trung Nguyen$^{*,1}$ \qquad
Nam Nguyen$^{1,2}$ \qquad
Tuan Dam$^{1, \dagger}$ \\[3pt]
$^{1}$Hanoi University of Science and Technology, Hanoi, Vietnam \\
$^{2}$Center for AI Research, VinUniversity, Hanoi, Vietnam \\
$^{*}$Equal contribution \\
$^{\dagger}$Correspondence to: Tuan Dam <tuandq@soict.hust.edu.vn>
}

\begin{document}
\maketitle
\begin{abstract}
Monte-Carlo Tree Search (MCTS) typically stores only scalar mean estimates, which makes it poorly suited for tail-risk objectives such as conditional value-at-risk (CVaR). We introduce \textbf{Categorical Thompson Sampling with Optimism (CATSO)} and \textbf{Particle Thompson Sampling with Optimism (PATSO)}, two distributional MCTS algorithms that maintain finite-support empirical backup laws at Q-edges. \algcatso represents Q-edge laws with categorical atoms and Dirichlet Thompson sampling, while \algpatso represents them with adaptive particles. Both algorithms select actions using the CVaR of a Thompson-sampled Q-edge law plus a polynomial optimism bonus, and propagate scalar continuation values through visit-weighted averages of child Q-edge CVaR scores. For the finite-depth nested CVaR planning target, we prove an $O(n^{-1/2})$ bound on the root value error after $n$ rollouts, which upper-bounds the risk-sensitive simple regret. \algcatso adds a fixed-grid discretization bias, while capped \algpatso adds a tunable Wasserstein compression bias. Experiments on stochastic planning benchmarks show that distributional Q-edges improve lower-tail return and reduce catastrophic outcomes when mean-optimal and risk-sensitive behavior differ.
\end{abstract}

\section{Introduction}
\label{sec:introduction}

Monte-Carlo Tree Search (MCTS) drives some of the most capable planning
systems ever built, from AlphaGo~\citep{silver2016mastering} and
AlphaZero~\citep{silver2018general} to MuZero~\citep{schrittwieser2020mastering}.
Classical variants such as UCT~\citep{kocsis2006improved},
Fixed-Depth-MCTS~\citep{shah2022journal}, and Power-UCT~\citep{pmlr-v244-dam24a}
store a scalar value estimate at each edge, usually an empirical mean. This
is sufficient when the objective is expected return, but it is inadequate for
risk-sensitive planning. In safety-critical domains such as autonomous
driving, medical treatment planning, and financial decision making, two actions
can have similar means but very different lower-tail behavior.

A simple example illustrates the problem. Suppose a root action $a_1$ returns
$5$ deterministically, while another action $a_2$ returns $10$ with probability
$1/2$ and $0$ with probability $1/2$. Both actions have the same expected value,
but their lower-tail conditional value-at-risk (CVaR) is different: for small
$\tau$, $a_1$ has CVaR equal to $5$, whereas $a_2$ has CVaR equal to $0$.
A scalar mean estimate cannot distinguish these two return laws once their
means are known. A risk-sensitive planner, however, should prefer $a_1$ when
the objective is to protect against poor outcomes. This requires preserving
distributional information at the Q-edges of the search tree.

This paper develops distributional MCTS algorithms for finite-depth
risk-sensitive planning. We introduce \textbf{Categorical Thompson Sampling
with Optimism} (\algcatso) and \textbf{Particle Thompson Sampling with
Optimism} (\algpatso). Both algorithms maintain finite-support empirical
backup laws at Q-edges. \algcatso represents each Q-edge law using categorical
atoms with Dirichlet Thompson sampling, while \algpatso represents it using
adaptive weighted particles. At action selection time, both algorithms draw a
Thompson-sampled Q-edge law, compute its lower-tail CVaR, and add the
polynomial optimism bonus of~\citet{shah2022journal} to ensure sufficient
exploration.

The algorithms keep V-nodes scalar by backing up a visit-weighted average of child Q-edge CVaR scores, so Q-edges store distributions while V-nodes propagate scalar risk-sensitive summaries. This avoids distributional V-node convolutions while still letting lower-tail information drive the search. Our main result proves an $O(n^{-1/2})$ root value-error bound against the finite-depth nested CVaR target (Section~\ref{sec:theory}), which yields an $O(n^{-1/2})$ risk-sensitive simple regret. \algcatso incurs an additional finite-grid discretization bias, and capped \algpatso adds a tunable Wasserstein compression bias.

Our contributions are:
\begin{itemize}
    \item We introduce \algcatso and \algpatso, two non-parametric distributional
    MCTS algorithms that maintain empirical backup laws at Q-edges.
    \item We propose a CVaR-sensitive scalar V-backup that aggregates
    visit-weighted child Q-edge CVaR scores rather than child Q-edge means.
    \item We prove an $O(n^{-1/2})$ root value-error bound for the finite-depth nested CVaR target after $n$ rollouts, which directly upper-bounds the risk-sensitive simple regret.
    \item We characterize the additional approximation terms introduced by finite \algcatso grids and capped \algpatso particle compression.
    \item We evaluate the algorithms on stochastic planning benchmarks where
    mean-optimal and CVaR-optimal behavior differ.
\end{itemize}

Existing distributional MCTS methods such as DNG-MCTS~\citep{NIPS2013_846c260d}
and D2NG-POMCP~\citep{bai2014thompson} rely on parametric return models, such
as Gaussian or mixture distributions. In contrast, \algcatso and \algpatso are
non-parametric: we estimate Q-edge backup laws directly from backed-up
samples. This makes our methods suitable for planning problems with multimodal,
skewed, or heavy-tailed return distributions, where parametric approximations
can obscure the lower-tail behavior that CVaR is designed to measure.

\section{Related Work}
\label{sec:related}

\textbf{MCTS with non-asymptotic guarantees and distributional backups.} 
While UCT~\cite{kocsis2006improved} established asymptotic convergence for MCTS, the breakthrough finite-sample analysis is due to Shah et al.~\cite{shah2022journal}, who modeled each tree node as a non-stationary bandit and proved $\mathcal{O}(n^{-1/2})$ root value convergence using a polynomial exploration bonus; Dam et al.~\cite{pmlr-v244-dam24a} extended this framework to power-mean backups. However, these analyses rely on scalar point estimates and do not capture tail risk. Existing distributional MCTS methods place parametric posteriors over returns—such as Normal-Gamma models in DNG-MCTS~\cite{NIPS2013_846c260d, bai2014thompson} or Gaussian Wasserstein barycenters in W-MCTS~\cite{pmlr-v267-dam25c}—or address heuristic multi-objective search~\cite{hayes2021risk}. Parametric assumptions can obscure lower-tail behavior in multimodal or heavy-tailed domains. Most recently, Zhang et al.~\cite{zhang2026tail} proposed CVaR-MCTS and a distributionally robust variant (also termed Wasserstein-MCTS, distinct from the Gaussian barycenter method of Dam et al.~\cite{pmlr-v267-dam25c}) that optimize \emph{expected reward subject to a static cumulative-cost CVaR constraint} with PAC tail-safety and regret guarantees. In contrast, CATSO and PATSO maintain \emph{non-parametric} empirical backup laws at Q-edges while keeping V-nodes scalar, directly optimizing the unconstrained \emph{nested finite-depth CVaR of reward} with an $\mathcal{O}(n^{-1/2})$ finite-budget root value-error guarantee.

\textbf{Thompson sampling in stationary vs.\ non-stationary CVaR bandits.} 
Thompson sampling has been combined with MCTS over scalar or Gaussian posteriors~\cite{NIPS2013_846c260d, greshler2024a}. In the stationary multi-armed bandit literature, Riou and Honda~\cite{riou2020bandit} established logarithmic regret for Dirichlet Thompson sampling on bounded-support rewards, which Baudry et al.~\cite{baudry2021optimal} extended to support-aware stationary CVaR bandits via multinomial grids (M-CVTS) and adaptive particles (B-CVTS). Our local bandit analysis builds on this Dirichlet convex-polytope machinery, but stationary CVaR bandit guarantees do not transfer directly to MCTS: the backed-up samples $\overline{Q}_t(s_h,a) = r_t(s_h,a) + \gamma \hat{V}^\tau(s_{h+1})$ have a non-stationary law that drifts as the searched subtree evolves. Controlling this coupled evolution requires pairing Dirichlet CVaR Thompson sampling with a polynomial optimism bonus and establishing our three-stage proof roadmap (local non-stationary CVaR bandit concentration, one-step $W_1$-Lipschitz CVaR lifting across Q-edges, and backward induction).

\textbf{Risk-sensitive RL} 
Beyond tree search, risk-sensitive RL has been studied through policy gradients~\cite{tamar2015optimizing, saferl:cvar:chow2015}, distributional RL~\cite{dabney2018distributional, keramati2020being}, robust MDPs~\cite{iyengar2005robust, nilim2005robust}, and Wasserstein DRO~\cite{esfahani2018data}. Recent theoretical work has established provably efficient episodic regret and PAC sample-complexity bounds across two primary objectives: \emph{static trajectory CVaR}, including near-minimax-optimal learning~\cite{wang2023near} and reward-free CVaR exploration~\cite{ni2024risk}, and \emph{nested (iterated) CVaR}, including tabular episodic learning~\cite{du2023provably}, function approximation with human feedback~\cite{chen2024provably}, and generative-model sample complexity~\cite{deng2025near}. Whereas these works bound global episodic regret or offline sample complexity across full MDP state spaces, our work addresses \emph{online finite-budget Monte-Carlo tree search} from a given root state $s_0$, proving non-asymptotic $\mathcal{O}(n^{-1/2})$ root value-error and simple-regret guarantees for nested CVaR planning.
\section{Setting}
\label{sec:setting}

\paragraph{MDP and finite-depth planning.}
We consider a discounted MDP $\mathcal M=\langle\mathcal S,\mathcal A,\mathcal R,\mathcal P,\gamma\rangle$, where $\mathcal S$ is the state space, $\mathcal A$ is a finite action set, $\mathcal R(\cdot\mid s,a)$ is a stochastic reward kernel with support in $[0,R_{\max}]$ \footnote{Note that this assumption is for the sake of simplicity, the analysis remains valid if one assumes any closed bounded reward interval $[R_{\min}, R_{\max}]$ by shifting the interval up by $R_{\min}$ constant and the fact that $\CVaR_\tau(X + c) = \CVaR_\tau(X) + c$ for any constant $c$. 
}, $\mathcal P(\cdot\mid s,a)$ is a stochastic transition kernel over $\mathcal S$, and $\gamma\in[0,1)$ is the discount factor. MCTS plans from a root state $s_0$ by simulating depth-$H$ trajectories, using a fixed rollout policy $\pi_{\mathrm{ro}}$ at depth $H$; we write $G_H^{\pi_{\mathrm{ro}}}(s) = \sum_{k \ge 0}\gamma^kr_k$ for the resulting discounted return. For the analysis, every internal node has finitely many legal actions and every encountered transition kernel has finite support.

\paragraph{Lower-tail CVaR.}
For a real-valued $X$ and $\tau\in(0,1]$, the lower-tail conditional value-at-risk is
\begin{equation}
\label{eq:cvar}
  \CVaR_\tau(X):=\frac1{\tau}\int_0^\tau F_X^{-1}(u)\,du,
\end{equation}
the average of the worst $\tau$-fraction of outcomes, with $\CVaR_1(X)=\mathbb E[X]$. We also write $\CVaR_\tau(P)$ for laws $P$; for finitely supported $P$, this is the discrete formula \eqref{eq:disc-cvar}. The same $\tau$ is used for selection, backup, and evaluation throughout. Although we state the theory for lower-tail CVaR, this is without loss of generality: upper-tail risk can be handled by applying the same framework to $-X$, equivalently using $\overline{\CVaR}_\tau(X):=-\CVaR_\tau(-X)$.

\paragraph{Finite-depth nested CVaR target.}
Set $\widetilde V^\tau(s_H):=\CVaR_\tau(G_H^{\pi_{\mathrm{ro}}}(s_H))$ and, for $h=H-1,\ldots,0$, define the Bellman-backup law
\begin{equation}
\label{eq:cvar-backup-law}
  \mathcal D^\tau(s_h,a):=\mathcal L\left(R(s_h,a,S')+\gamma\widetilde V^\tau(S')\right),\quad S'\sim\mathcal P(\cdot\mid s_h,a),
\end{equation}
with reward sampled from $\mathcal R(s_h,a,S')$, and
\begin{equation}
\label{eq:cvar-q-target}
  \widetilde Q^\tau(s_h,a):=\CVaR_\tau\left(\mathcal D^\tau(s_h,a)\right),
  \qquad
  \widetilde V^\tau(s_h):=\max_{a\in\mathcal A(s_h)}\widetilde Q^\tau(s_h,a).
\end{equation}
Write $\widetilde V^\tau(s_0):=\widetilde V_0^\tau(s_0)$. At $\tau=1$, \eqref{eq:cvar-q-target} reduces to the standard finite-depth risk-neutral Bellman recursion; for $\tau<1$, it is a nested time-consistent CVaR target, distinct from the static $\CVaR_\tau(\sum_{t\ge0}\gamma^t r_t)$. All guarantees are stated relative to \eqref{eq:cvar-q-target}.

\paragraph{Algorithmic root value estimate.}
After $n$ rollouts, each visited Q-edge $(s,a)$ maintains an empirical backup law $\widehat{\mathcal D}_{T_{s,a}(n)}(s,a)$, where $T_{s,a}(n)$ is its visit count, with empirical CVaR score $\widehat c_{\tau,n}(s,a):=\CVaR_\tau(\widehat{\mathcal D}_{T_{s,a}(n)}(s,a))$. The scalar V-node estimate is the visit-weighted average
\begin{equation}
\label{eq:root-v-estimate-setting}
  \widehat V_{T_s(n)}^\tau(s):=\sum_{a\in\mathcal A(s)}\frac{T_{s,a}(n)}{T_s(n)}\widehat c_{\tau,n}(s,a),
\end{equation}
with each action initialized before backup; $\widehat V_n^\tau(s_0)$ is the root estimate returned after $n$ rollouts.

\paragraph{Risk-sensitive simple regret.}
The risk-sensitive simple regret relative to \eqref{eq:cvar-q-target} is
\begin{equation}
\label{eq:sr-tau}
  \mathrm{SR}_\tau(n):=\mathbb E\left[\widetilde V^\tau(s_0)-\widehat V_n^\tau(s_0)\right],
\end{equation}
and our main theorem controls the stronger absolute error
\begin{equation}
\label{eq:absolute-root-error}
  \mathcal E_\tau(n):=\mathbb E\left[\left|\widehat V_n^\tau(s_0)-\widetilde V^\tau(s_0)\right|\right]\ge \mathrm{SR}_\tau(n),
\end{equation}
so any $O(n^{-1/2})$ bound on $\mathcal E_\tau(n)$ transfers to $\mathrm{SR}_\tau(n)$.

\section{Algorithms}
\label{sec:algorithms}

We replace the scalar mean at each Q-edge with an empirical backup law, used both for risk-sensitive exploration (Thompson-sampled CVaR scores) and risk-sensitive backup (empirical Q-edge CVaR scores). V-nodes remain scalar, propagating a visit-weighted average of child Q-edge CVaR values.

\subsection{Distributional Q-Edges}
\label{sec:qedge-representation}

On a rollout, the scalar sample stored at Q-edge $(s_h,a)$ is
\begin{equation}
\label{eq:qbar-alg}
  \overline Q_t(s_h,a):=r_t(s_h,a)+\gamma\widehat V^\tau(s_{h+1}),
\end{equation}
where $\widehat V^\tau(s_{h+1})$ is the current scalar CVaR-sensitive continuation value. \algcatso and \algpatso differ only in how they represent the empirical law of these samples.

\noindent\textbf{Categorical Thompson Sampling with Optimism} (\algcatso) represents the Q-edge law as a categorical distribution on a fixed ordered grid
$
  Z_N(s_h,a)=\{z_0(s_h,a)<\cdots<z_N(s_h,a)\}.
$
It maintains Dirichlet parameters
$
  \alpha^0(s_h,a),\ldots,\alpha^N(s_h,a),
$ giving the posterior-center law
\begin{equation}
\label{eq:catso-law}
  \widehat{\mathcal D}_t(s_h,a)=\sum_{i=0}^{N}p_i(s_h,a)\delta_{z_i(s_h,a)},\quad
  p_i(s_h,a)=\frac{\alpha^i(s_h,a)}{\sum_{\ell=0}^{N}\alpha^\ell(s_h,a)}.
\end{equation}
When a backed-up sample $\overline Q_t(s_h,a)$ is observed, \algcatso assigns it
to the nearest grid atom and increments the corresponding Dirichlet parameter.

The theoretical analysis uses a fixed grid chosen before search.  In experiments, we also consider an adaptive-grid implementation, described in Appendix~\ref{app:experiments} (Figure~\ref{fig:catso-patso}), which expands the support range when new backed-up samples fall outside the current interval.  This adaptive variant is treated as an implementation heuristic, introducing the projection bias of Remark~\ref{rem:catso_discretization_bias}, and is not the object of the fixed-grid Dirichlet posterior-tail proof.

\noindent\textbf{Particle Thompson Sampling with Optimism} (\algpatso) uses adaptive weighted particles. With current support $\mathcal S(s_h,a)=\{x_1,\ldots,x_M\}$ and weights $\alpha^j$,
\begin{equation}
\label{eq:patso-law}
  \widehat{\mathcal D}_t(s_h,a)=\sum_{j=1}^{M}\bar p_j(s_h,a)\delta_{x_j(s_h,a)},\quad
  \bar p_j(s_h,a)=\frac{\alpha^j(s_h,a)}{\sum_{\ell=1}^{M}\alpha^\ell(s_h,a)}.
\end{equation}
Uncapped \algpatso represents the empirical law exactly; capped \algpatso keeps at most $U$ particles per edge, adding the compression bias of Remark~\ref{rem:capped_patso_bias}.

\subsection{Discrete CVaR Computation}
\label{sec:disc-cvar-alg}

Both algorithms only evaluate CVaR of finitely supported laws. For $P=\sum_{i=1}^{m}w_i\delta_{x_i}$ with $x_1\le\cdots\le x_m$ and $\sum_i w_i=1$, set $c_j:=\sum_{i=1}^{j}w_i$, $c_0:=0$, $k:=\min\{j:c_j\ge\tau\}$. Then
\begin{equation}
\label{eq:disc-cvar}
  \CVaR_\tau(P)=\frac1\tau\left(\sum_{i=1}^{k-1}w_ix_i+(\tau-c_{k-1})x_k\right).
\end{equation}
Since \algcatso atoms and \algpatso particles are stored sorted, this costs linear time in the support size.

\subsection{CVaR Thompson Selection}
\label{sec:selection}

At node $s_h$, each action draws a Dirichlet sample over its current support — atoms $\{z_i\}$ for \algcatso, particles $\{x_j\}$ for \algpatso — yielding the Thompson-sampled CVaR score $\widetilde c_\tau(s_h,a):=\CVaR_\tau(\sum_i L_i(s_h,a)\delta_{\cdot})$. The selected action is
\begin{equation}
\label{eq:selection}
  a_t\in\arg\max_{a\in\mathcal A(s_h)}\left\{\widetilde c_\tau(s_h,a)+B(n,s_h,a)\right\},
  \qquad
  B(n,s_h,a):=C\frac{T_{s_h}(n)^{1/4}}{(T_{s_h,a}(n)\vee 1)^{1/2}}.
\end{equation}
The polynomial bonus follows the finite-sample MCTS analysis of \citet{shah2022journal}: Thompson sampling drives distributional exploration, while the bonus prevents premature underexploration as descendant values still evolve.

\subsection{Pseudocode}
\label{sec:pseudocode}

Figure~\ref{fig:catso-patso} gives complete pseudocode. The two algorithms share tree traversal, rollout, and scalar V-backup; they differ only in the Q-edge representation and update. \algcatso initializes every Q-edge with the unit Dirichlet prior $(1,\ldots,1)$; \algpatso starts with an empty particle set, with at least one forced visit per action before any Dirichlet draw so that $\alpha(s,a)$ is nonempty.

\begin{figure}[t]
\centering
\noindent
\begin{minipage}[t]{0.5\linewidth}
\begin{algorithm}[H]
\caption{\algcatso}
\label{alg:catso}
\scriptsize
\DontPrintSemicolon
\SetAlgoNoEnd
\Fn{\Sel{$s_h$}}{
    \For{$a\in\{1,2,\dots,A\}$}{
        $L(s_h,a)\sim\dir(\alpha^0(s_h,a),\ldots,\alpha^N(s_h,a))$
        $\widehat c^\tau(s_h,a)\gets\CVaRdisc_\tau\bigl([z_0(s_h,a),\ldots,z_N(s_h,a)],L(s_h,a)\bigr)$
    }
    $a=\arg\max_a\bigl\{\widehat c^\tau(s_h,a)+B(n,s_h,a)\bigr\}$\\
    \Return $a$
}
\Fn{\SimV{$s_h,t$}}{
    $a=\Sel(s_h)$\\
    $\SimQ(s_h,a,t)$\\
    $T_{s_h}(t)=T_{s_h}(t)+1$\\
    \For{$a'\in\{1,2,\dots,A\}$}{
        $p_i(s_h,a')\gets \alpha^i(s_h,a')/\sum_{\ell=0}^{N}\alpha^\ell(s_h,a')$\\
        $\widehat c^\tau(s_h,a')\gets \CVaRdisc_\tau    \bigl([z_0(s_h,a'),\ldots,z_N(s_h,a')],p(s_h,a')\bigr)$
    }
    $\Vhat^\tau_{T_{s_h}(t)}(s_h) \gets \sum_{a'} \tfrac{T_{s_h,a'}(t)}{T_{s_h}(t)} \widehat c^\tau(s_h,a')$
}
\Fn{\SimQ{$s_h,a,t$}}{
    $s_{h+1}\sim\mathbb P(\cdot|s_h,a)$,
    $r_t(s_h,a)\sim\mathcal R(s_h,a,s_{h+1})$\\
    \eIf{node $s_{h+1}$ not expanded}{
        $\Rollout(s_{h+1})$
    }{
        $\SimV(s_{h+1},t)$
    }
    $T_{s_h,a}(t)=T_{s_h,a}(t)+1$\\
    $\Qbar_t(s_h,a)=r_t(s_h,a)+\gamma\Vhat^\tau_{T_{s_{h+1}}(t)}(s_{h+1})$\\
    $i \leftarrow \argmin_i|\Qbar_t(s_h,a) - z_i(s_h, a)| $\\
    $\alpha_i(s_h,a) \leftarrow \alpha_i(s_h,a) + 1 $\\
    Update $p(s_h,a)=[p_0(s_h,a),\ldots,p_N(s_h,a)]$
}
\end{algorithm}
\end{minipage}%
\hfill
\begin{minipage}[t]{0.5\linewidth}
\begin{algorithm}[H]
\caption{\algpatso}
\label{alg:patso}
\scriptsize
\DontPrintSemicolon
\SetAlgoNoEnd
\Fn{\Sel{$s_h$}}{
    \For{$a\in\{1,2,\dots,A\}$}{
        $L(s_h,a)\sim\dir(\alpha(s_h,a))$\\
        $\widehat c^\tau(s_h,a)\gets\CVaRdisc_\tau\bigl(\mathcal S(s_h,a),L(s_h,a)\bigr)$
    }
    $a=\arg\max_a\bigl\{\widehat c^\tau(s_h,a)+B(n,s_h,a)\bigr\}$\\
    \Return $a$
}
\Fn{\SimV{$s_h,t$}}{
    $a=\Sel(s_h)$\\
    $\SimQ(s_h,a,t)$\\
    $T_{s_h}(t)=T_{s_h}(t)+1$\\
    \For{$a'\in\{1,2,\dots,A\}$ such that $T_{s_h,a'}(t)>0$}{
        $\bar p_j(s_h,a')
        \gets
        \alpha^j(s_h,a')/\sum_\ell \alpha^\ell(s_h,a')$\\
        $\widehat c^\tau(s_h,a') \gets \CVaRdisc_\tau
        \bigl(\mathcal S(s_h,a'),\bar p(s_h,a')\bigr)$
    }
    $\Vhat^\tau_{T_{s_h}(t)}(s_h) \gets \sum_{a'} \tfrac{T_{s_h,a'}(t)}{T_{s_h}(t)} \widehat c^\tau(s_h,a')$
}
\Fn{\SimQ{$s_h,a,t$}}{
    $s_{h+1}\sim\mathbb P(\cdot|s_h,a)$,
    $r_t(s_h,a)\sim\mathcal R(s_h,a,s_{h+1})$\\
    \eIf{node $s_{h+1}$ not expanded}{
        $\Rollout(s_{h+1})$
    }{
        $\SimV(s_{h+1},t)$
    }
    $T_{s_h,a}(t)=T_{s_h,a}(t)+1$\\
    $\Qbar_t(s_h,a)=r_t(s_h,a)+\gamma\Vhat^\tau_{T_{s_{h+1}}(t)}(s_{h+1})$\\
    \eIf{$\Qbar_t(s_h,a)\in\mathcal S(s_h,a)$}{
        $\alpha^j(s_h,a)\mathrel{+}=1$\quad\tcp*[h]{at the matching atom $x_j$}
    }{
        $\mathcal S(s_h,a)\gets\mathcal S(s_h,a)\cup\{\Qbar_t(s_h,a)\}$\\
        $\alpha(s_h,a)\gets(\alpha(s_h,a),1)$
    }
}
\end{algorithm}
\end{minipage}
\caption{Comparison of \algcatso (left) and \algpatso (right). Both methods share rollout, expansion, and \texttt{SimulateV} (which backs up scalar V-nodes using a visit-weighted average of child CVaR). They differ in their Q-edge representation (\texttt{SimulateQ}): \algcatso uses categorical Dirichlet atoms, whereas \algpatso uses adaptive weighted particles. \texttt{SelectAction} relies on the CVaR of a Thompson-sampled Q-edge plus a polynomial bonus $B(n,s,a)$ (Eq.~\eqref{eq:selection}). For expected-value planning, $\tau=1$.}
\label{fig:catso-patso}
\end{figure}

\subsection{CVaR-Average V-Backup}
\label{sec:backup-recommendation}
Define the empirical CVaR score at each Q-edge:
\begin{equation}
\label{eq:empirical-cvar-score}
\widehat c_{\tau,t}(s_h,a):=\CVaR_\tau\left(\widehat{\mathcal D}_t(s_h,a)\right),
\end{equation}
and back up at V-nodes a visit-weighted average of child Q-edge CVaR scores, so that the parent Q-edge sample is
\begin{equation}
\label{eq:cvar-v-backup}
  \widehat V_{T_{s_h}(t)}^\tau(s_h):=\sum_{a\in\mathcal A(s_h)}\frac{T_{s_h,a}(t)}{T_{s_h}(t)}\widehat c_{\tau,t}(s_h,a),
  \qquad
  \overline Q_t(s_h,a)=r_t(s_h,a)+\gamma\widehat V_{T_{s_h}(t)}^\tau(s_{h+1}).
\end{equation}
Distributions thus live only at Q-edges; V-nodes store scalar risk-sensitive summaries, avoiding distributional convolutions while still propagating lower-tail information up the tree. Note that \eqref{eq:cvar-v-backup} is an arithmetic average of CVaRs, not the CVaR of the mixture $\sum_a w_a\widehat{\mathcal D}(s,a)$, which differ in general; we use the former because V-nodes are scalar by design. At $\tau=1$, $\CVaR_1$ equals the mean, so \eqref{eq:cvar-v-backup} reduces to the standard average-mean V-backup $\widehat V_{T_{s_h}(t)}^1(s_h)=\sum_a\tfrac{T_{s_h,a}(t)}{T_{s_h}(t)}\widehat Q_{{T_{s_h,a}(t)}}(s_h,a)$ with $\widehat Q(s_h,a)=\mathbb E_{\widehat{\mathcal D}_t(s_h,a)}[X]$, recovering expected-value planning.

\subsection{Root Recommendation}
\label{sec:root-recommendation}
After $n$ rollouts, the planner returns the scalar root value $\widehat V_n^\tau(s_0)$ from \eqref{eq:cvar-v-backup} and the empirical root Q-edge CVaR action
\begin{equation}
\label{eq:cvar-root-recommendation}
  \widehat a_n^{\CVaR}\in\arg\max_{a\in\mathcal A(s_0)}\widehat c_{\tau,n}(s_0,a).
\end{equation}
\begin{remark}
$\widehat V_n^\tau(s_0)$ is the object controlled by the main convergence theorem, while \eqref{eq:cvar-root-recommendation} is the risk-sensitive root action used in experiments.
\end{remark}

\section{Theoretical Analysis}
\label{sec:theory}

We prove that \algcatso and \algpatso converge to the finite-depth nested CVaR target $\widetilde V^\tau(s_0)$ in three steps: (i) a single-node CVaR bandit analysis; (ii) a one-step CVaR lifting lemma showing that child V-concentration implies plug-in Q-edge CVaR concentration; (iii) backward induction over the finite-depth tree. The main quantity controlled is the scalar root estimate
\[
  \widehat V_{T_s(n)}^\tau(s)=\sum_{a\in\mathcal A(s)}\frac{T_{s,a}(n)}{T_s(n)}\widehat c_{\tau,n}(s,a),
  \quad
  \widehat c_{\tau,n}(s,a)=\CVaR_\tau\!\left(\widehat{\mathcal D}_{T_{s,a}(n)}(s,a)\right),
\]
whose absolute error upper-bounds the simple regret~\eqref{eq:sr-tau}. The proof of step (i) is given in Appendix \ref{app:local_bandits}, where we analyze the local CVaR bandit problem and derive the \algcatso/\algpatso suboptimal-arm pull bounds. Step (ii) is proved in Appendix \ref{app:one_step_convergence} through the one-step CVaR lifting lemma. Step (iii) is completed in the rest of Appendix \ref{app:mcts_convergence}, which apply backward induction separately to \algcatso and \algpatso and then translate the root value-error bound into the risk-sensitive simple-regret guarantee.

\subsection{CVaR-Objective Non-Stationary Bandits}
\label{sec:cvar_bandit}

Consider $K$ arms with bounded non-stationary payoffs $\{R_{a,t}\}\subset[0,R]$ and empirical law $\widehat F_{a,m}:=\tfrac1m\sum_{t=1}^m\delta_{R_{a,t}}$. Assume each $\widehat F_{a,m}$ stabilizes to a limiting law $F_a$, and write $c_a^\tau:=\CVaR_\tau(F_a)$, $c_\star^\tau:=\max_a c_a^\tau$, $a_\tau^\star\in\arg\max_a c_a^\tau$, $\widehat c_{a,m}^\tau:=\CVaR_\tau(\widehat F_{a,m})$. The node-level scalar backup is the visit-weighted average
\begin{equation}
\label{eq:cvar-average-bandit-backup}
  \widehat c_n^\tau:=\sum_{a=1}^K\frac{T_a(n)}{n}\widehat c_{a,T_a(n)}^\tau,
\end{equation}
which reduces to the standard average-mean backup at $\tau=1$.

\begin{definition}[Convergence with concentration]
\label{def:concentration}
$(\widehat V_n)_{n\ge1}$ converges with concentration to $V$, written $\plim_{n\to\infty}\widehat V_n=V$, if $\lim_n\mathbb E[\widehat V_n]=V$ and there exists $c>0$ with $\mathbb P(|\widehat V_n-V|>\varepsilon)\le c/(n\varepsilon^2)$ for all $n\ge1$, $\varepsilon>0$.
\end{definition}

\begin{assumption}[Asymptotic CVaR stationarity with concentration]
\label{assump:cvar_main}
For every arm $a$, $\plim_{m\to\infty}\widehat c_{a,m}^\tau = c_a^\tau$, and the CVaR maximizer $a_\tau^\star \in \arg\max_a c_a^\tau$ is unique.
\end{assumption}

\begin{remark}[Role of Assumption~\ref{assump:cvar_main} in MCTS]
\label{rem:cvar_main_mcts}
Assumption~\ref{assump:cvar_main} describes the local bandit problem at each V-node of the MCTS tree: fixing $s_h$, each $a\in\mathcal A(s_h)$ is an arm with empirical score $\widehat c_{a,m}^\tau := \widehat Q_{m}^\tau(s_h,a)$ and limiting value $c_a^\tau := \widetilde Q^\tau(s_h,a)$.
The armwise convergence is therefore not an extra modeling assumption: it is verified by backward induction via Lemma~\ref{lem:cvar_q_conc}. If $\plim_{n\to\infty}\widehat V_{n}^\tau(s_{h+1}) = \widetilde V^\tau(s_{h+1})$ for every child $s_{h+1}$, the lemma yields $\plim_{n\to\infty}\widehat Q_{n}^\tau(s_h,a) = \widetilde Q^\tau(s_h,a)$. The terminal rollout estimator starts the induction, and the one-step lifting propagates it upward through the finite-depth tree. Uniqueness of $a_\tau^\star$ is a standard nondegeneracy condition simplifying the suboptimal-pull bounds; ties can be handled by replacing $a_\tau^\star$ with the optimal set $\mathcal A_\tau^\star := \arg\max_a c_a^\tau$ and treating only actions outside it as suboptimal.
\end{remark}

Both algorithms select using the Thompson-sampled CVaR score $\widetilde c_a^\tau(t):=\CVaR_\tau(\widetilde{\mathcal D}_{a,T_a(t)})$ plus the polynomial bonus,
\begin{equation}
\label{eq:cvar_index_main}
  I(t)\in\arg\max_{a\in[K]}\left\{\widetilde c_a^\tau(t)+C\frac{t^{1/4}}{(T_a(t)\vee1)^{1/2}}\right\}.
\end{equation}
For a probability law $F$ on $[0,R_{\max}]$ and threshold $c\in\mathbb R$, define the upper-tail risk information number
$
\mathcal K_{\inf}^{\tau,+}(F,c)
:=
\inf\bigl\{\mathrm{KL}(F\|G): G\in\mathcal P([0,R_{\max}]),\ \CVaR_\tau(G)\ge c\bigr\},
$
and its grid-restricted version on $\mathcal Z_N=\{z_0<\cdots<z_N\}\subset[0,R_{\max}]$,
$
\mathcal K_{\inf}^{(N),\tau,+}(F,c)
:=
\inf\bigl\{\mathrm{KL}(F\|G): G\in\mathcal P(\mathcal Z_N),\ \CVaR_\tau(G)\ge c\bigr\}.
$ 
\begin{theorem}[\algcatso/\algpatso in CVaR non-stationary bandits]
\label{thm:cvar_bandit_main}
Suppose Assumption~\ref{assump:cvar_main} holds, each arm is sampled once initially, and selection follows~\eqref{eq:cvar_index_main}. Under the fixed-grid \algcatso conditions of Appendix~\ref{sec:bandits-convergence} (resp.\ uncapped \algpatso conditions of Appendix~\ref{sec:patso_cvar_bandit}), there exists $C_{\mathrm B}<\infty$ such that
\[
  \mathbb E\!\left[|\widehat c_n^\tau-c_\star^\tau|\right]\le C_{\mathrm B}n^{-1/2},
  \qquad
  \mathbb E[T_a(n)]\le\frac{(1+\epsilon_0)\log n}{\mathcal K_{\inf}^{\bullet,\tau,+}(F_a,c_\star^\tau)}+O(n^{1/2}),
\]
for every CVaR-suboptimal $a\neq a_\tau^\star$ and $\epsilon_0>0$, where $\bullet=(N)$ for \algcatso and is omitted for \algpatso.
\end{theorem}

The proof rests on the decomposition
\[
  \left|\widehat c_n^\tau-c_\star^\tau\right|\le\sum_{a=1}^K\frac{T_a(n)}{n}\left|\widehat c_{a,T_a(n)}^\tau-c_a^\tau\right|+R\sum_{a\neq a_\tau^\star}\frac{T_a(n)}{n}.
\]
The first term is controlled by Assumption~\ref{assump:cvar_main}; the second by the CVaR suboptimal-arm pull bound — the CVaR-average analogue of the standard scalar MCTS argument.

\subsection{One-Step CVaR Lifting}
\label{sec:cvar_lifting_main}
Why does MCTS require a \emph{non-stationary} bandit formulation together with an \emph{asymptotic CVaR stationarity} condition (Assumption~\ref{assump:cvar_main})? Unlike variation-budget or change-point models for exogenous environment shifts, the non-stationarity in MCTS is \emph{endogenous to the search}: at an internal node $s_h$, pulling action $a$ yields a backed-up sample $\overline{Q}_i = R_i + \gamma \hat{V}_{N_i}^\tau(S_i')$, whose distribution drifts as the descendant value estimate $\hat{V}_{N_i}^\tau(S_i')$ evolves. Because edge samples are not i.i.d.\ draws from a fixed law, stationary CVaR bandit guarantees~\cite{baudry2021optimal, riou2020bandit} cannot be applied recursively as-is. Instead, following the non-stationary MCTS paradigm established for expected-value and power-mean planning~\cite{shah2022journal, pmlr-v244-dam24a}, each node requires that its drifting arm payoffs eventually stabilize with polynomial concentration—ruling out persistent or adversarial drift. Assumption~\ref{assump:cvar_main} is the CVaR analogue of this local stabilization condition, replacing empirical-mean convergence with empirical-CVaR convergence. Crucially, its armwise concentration requirement is an inductive property verified across the tree rather than an external environmental assumption: Lemma~\ref{lem:cvar_q_conc} below provides the induction step, showing that scalar CVaR concentration at child nodes $(s_{h+1})$ implies $W_1$ concentration of the plug-in backup law at $(s_h, a)$, which by CVaR Lipschitzness verifies Assumption~\ref{assump:cvar_main} at depth $h$.
\begin{lemma}[One-step CVaR lifting for plug-in Q-edge laws]
\label{lem:cvar_q_conc}
Fix $(s_h,a)$ and let $(X_i,S_i)_{i\ge1}$ be i.i.d.\ one-step samples from the reward-transition kernel at $(s_h,a)$, with $X_i\in[0,R_{\max}]$ and $S_i\sim P(\cdot\mid s_h,a)$. Assume $P(\cdot\mid s_h,a)$ has finite support $\{s_{h+1}^{(1)},\ldots,s_{h+1}^{(M)}\}$ with $p_m:=P(s_{h+1}^{(m)}\mid s_h,a)>0$, $p_{\min}:=\min_m p_m$, and let
$
  N_m^n:=\sum_{i=1}^n \mathbf 1\{S_i=s_{h+1}^{(m)}\}.
$
For each $s_{h+1}^{(m)}$, let $V^\tau(s_{h+1}^{(m)})$ be the limiting scalar CVaR value and $\widehat V_{k}^\tau(s_{h+1}^{(m)})$ its estimator after $k$ local visits, with
\[
  0\le \widehat V_{k}^\tau(s_{h+1}^{(m)}),\,V^\tau(s_{h+1}^{(m)})\le \frac{R_{\max}}{1-\gamma},
  \qquad
  \underset{k\to\infty}{\mathrm{plim}}\,\widehat V_{k}^\tau(s_{h+1}^{(m)})=V^\tau(s_{h+1}^{(m)}).
\]
Define plug-in and oracle backed-up samples
$
  \bar q_i^{(n)}:=X_i+\gamma\widehat V_{N_{S_i}^n}^\tau(S_i),
  G_i:=X_i+\gamma V^\tau(S_i),
$, where $N_{S_i}^n=N_m^n$ when $S_i=s_{h+1}^{(m)}$, with empirical and population laws
$
  \widehat{\mathcal D}_{n}(s_h,a):=\tfrac1n\sum_{i=1}^n\delta_{\bar q_i^{(n)}},
$
and $D^\tau(s_h,a)$ is defined in~(\ref{eq:cvar-backup-law}). The corresponding CVaR Q-values $\widehat Q_{n}^{\tau}(s_h,a):=\CVaR_\tau(\widehat{\mathcal D}_{n}(s_h,a))$, $Q^\tau(s_h,a):=\CVaR_\tau(\mathcal D^\tau(s_h,a))$. Then
\[
  \underset{n\to\infty}{\mathrm{plim}}\,\widehat Q_{n}^{\tau}(s_h,a)=Q^\tau(s_h,a).
\]
\end{lemma}

\subsection{From Bandits to MCTS}
\label{sec:bandits_to_mcts}

For each depth $h$, set $\widehat Q_{n}^\tau(s_h,a):=\CVaR_\tau(\widehat{\mathcal D}_{n}(s_h,a))$ and $\widehat V_{n}^\tau(s_h):=\sum_{a}\tfrac{T_{s_h,a}(n)}{T_{s_h}(n)}\widehat Q_{h,T_{s_h,a}(n)}^\tau(s_h,a)$.

\begin{theorem}[Convergence of \algcatso and uncapped \algpatso]
\label{thm:mcts_convergence}
Run \algcatso (resp.\ uncapped \algpatso) on a finite-depth tree of depth $H$ with bounded rewards and finite transition support at every Q-edge. Assume the terminal rollout estimator has an $O(n^{-1/2})$ $L^1$ rate, the local CVaR-optimal action at each internal node is unique, and Theorem~\ref{thm:cvar_bandit_main} applies at every internal node. Then there exists $C'>0$ such that
\[
  \mathbb E\!\left[\left|\widehat V_n^\tau(s_0)-\widetilde V^\tau(s_0)\right|\right]\le C'n^{-1/2},
\]
for the fixed-grid \algcatso target and the uncapped \algpatso target. For \algcatso against the unprojected target with grid mesh $\Delta_N$, the bound becomes $C'n^{-1/2}+O(\Delta_N/\tau)$.
\end{theorem}


Since $\mathrm{SR}_\tau(n)\le\mathbb E[|\widehat V_n^\tau(s_0)-\widetilde V^\tau(s_0)|]$, Theorem~\ref{thm:mcts_convergence} yields $\mathrm{SR}_\tau(n)=O(n^{-1/2})$ up to the stated approximation terms.

\begin{remark}[\algcatso discretization bias]
\label{rem:catso_discretization_bias}
The exact $O(n^{-1/2})$ \algcatso statement is for the fixed-grid target; against the unprojected backup law with mesh $\Delta_N$, CVaR Lipschitzness adds an $O(\Delta_N/\tau)$ bias, with depth-dependent constants absorbed.
\end{remark}

\begin{remark}[Capped \algpatso]
\label{rem:capped_patso_bias}
Theorem~\ref{thm:mcts_convergence} concerns uncapped \algpatso. If capped \algpatso satisfies $W_1(\widehat{\mathcal D}^{\mathrm{cap}}_{m,U},\widehat{\mathcal D}^{\mathrm{uncap}}_{m})\le\varepsilon_U(s,a)$, then every CVaR estimate gains an additional $\varepsilon_U(s,a)/\tau$ bias, and the root bound becomes $C'n^{-1/2}+O(\varepsilon_U/\tau)$.
\end{remark}

\subsection{Memory and Computation}
\label{sec:memory}

\paragraph{\algcatso.} Each Q-edge stores $N+1$ atoms with Dirichlet parameters, costing $O(N)$ per edge. Selection draws one Dirichlet vector and computes the discrete CVaR of the sorted categorical law in $O(N)$. Total memory after $n$ rollouts is $O(N|\mathcal E_n|)$ over the explored Q-edges $\mathcal E_n$.

\paragraph{\algpatso.} Uncapped \algpatso stores one weighted particle per distinct backed-up return ($O(m)$ after $m$ visits), representing the empirical Q-edge law exactly. For bounded memory, we cap at $U$ particles per edge: when insertion would exceed $U$, the closest adjacent pair $(x_i,\alpha_i),(x_{i+1},\alpha_{i+1})$ is merged into
\[
\left(\frac{\alpha_i x_i+\alpha_{i+1}x_{i+1}}{\alpha_i+\alpha_{i+1}},\,\alpha_i+\alpha_{i+1}\right),
\]
preserving the first moment. Capped \algpatso uses $O(U)$ memory and $O(U)$ CVaR-selection cost per edge, with the compression bias quantified by Remark~\ref{rem:capped_patso_bias}; increasing $U$ trades memory for a smaller bias.

\section{Experiments}
\label{sec:experiments}

We evaluate whether distributional Q-edge information improves planning when expected return and lower-tail return favor different actions. Each benchmark contains a safe option with modest reward and a risky option with attractive mean but poor lower-tail outcomes — directly testing the paper's central claim that scalar MCTS does not preserve enough information for CVaR. All methods plan from the same initial state with the same rollout policy, search depth, and budget; the resulting root recommendation or value estimate is evaluated by independent Monte-Carlo rollouts. Detailed information about the environments, algorithms, and setup is provided in Appendix ~\ref{app:experiments}. 

\subsection{Experimental Questions}
\label{sec:experimental_questions}

\textbf{Q1 (lower-tail return).} Does Monte-Carlo CVaR return of \algcatso/\algpatso exceed that of scalar baselines? \textbf{Q2 (catastrophic outcomes).} How often do recommended policies encounter environment-specific catastrophic events (bankruptcy, trap cells, severe-cost transitions, cliff falls)? \textbf{Q3 (root estimate quality).} For oracle-evaluable environments, does the root CVaR value gap — the empirical analogue of the simple regret~\eqref{eq:sr-tau} — shrink with planning budget? Another interesting question that one can raise is how sensitive the algorithm performance according to the representation of the empirical distribution, because of the space limit, we postpone the answer and leave it in Appendix \ref{app:experiments}.

\subsection{Environments}
\label{sec:environments}

We evaluate on five stochastic benchmarks designed such that the mean-optimal and CVaR-optimal actions differ. 

\textbf{Autonomous Vehicle \citep{rigter2021risk}:} A road network gridworld where faster routes (like highways) typically incur low costs but carry a risk of rare, extreme penalties (e.g., crashes).
\textbf{Two-Level Risky Treasure:} A two-step environment forcing a choice between a guaranteed safe reward or a risky branch leading to either an extreme payout or a severe penalty.
\textbf{Risky Shortcut Gridworld:} A navigation task offering a long, safe route versus a short, hazardous route where stochastic wind can push the agent off a cliff.
\textbf{Betting Game \citep{bauerle2011markov}:} A repeated fractional betting scenario where aggressive wagers maximize expected growth but risk bankruptcy.
\textbf{Guarded Maze \citep{greenberg2022efficient, luo2023alternative, luo2024simple}:} A maze featuring a short path with a high risk of triggering a severe penalty from a guard, contrasted with a longer, entirely safe alternative.

\subsection{Algorithms and Baselines}
\label{sec:baselines}

We compare \algcatso and \algpatso with scalar MCTS baselines (UCT-style planners and max/power-backup variants), which store scalar Q-edge statistics and select on scalar value estimates. \algcatso and \algpatso instead store empirical Q-edge laws and select via the Thompson-sampled CVaR score plus the polynomial bonus~\eqref{eq:selection}. For scalar baselines, CVaR is applied only at evaluation time — the planner recommends a root action and we estimate its CVaR with independent rollouts — so all methods are compared on the same risk-sensitive objective.

\subsection{Metrics}
\label{sec:metrics}

\paragraph{CVaR value gap.} For oracle-evaluable environments, we compute the quantity $\widehat{\mathrm{Gap}}_\tau(n)=\widetilde V^\tau(s_0)-\widehat V_n^\tau(s_0)$ (lower is better) at the root node. 
\paragraph{Monte-Carlo CVaR return.} Lower-tail CVaR of the return distribution induced by the recommended root action follows (\ref{eq:cvar-root-recommendation}) (higher is better). Note that we use static CVaR for evaluation, that is we take the average among $\tau$-worst trajectories in the Monte-Carlo evaluation. \paragraph{Catastrophic count.} Number of evaluation rollouts that encounter the environment-specific catastrophic event (lower is better; an interpretable safety metric, not a replacement for CVaR).

\subsection{Results}
\label{sec:results}

Figures~\ref{fig:cvar-regret}--\ref{fig:catastrophic-count} summarize results across planning budgets for three environments, since the paper page limit is restricted, we refer to table \ref{tab:final-budget-all-results} in appendix \ref{app:experiments} for full results. \algcatso and \algpatso consistently improve the metrics that depend on lower-tail behavior. Scalar baselines, compressing each Q-edge to scalar statistics, are repeatedly drawn to high-mean branches with poor tails — favoring higher-cost path in Autonomous Vehicle and shorter risky routes in the Guarded Maze and Risky Shortcut Gridworld. By contrast, the distributional methods preserve enough information to identify actions with similar or even worse means but safer lower tails.

The two distributional methods are complementary: \algcatso is stable when the categorical grid is fine but incurs discretization bias on sharp or widely separated modes, whereas \algpatso adapts its support to observed backed-up returns and handles sparse or multimodal laws naturally, with capping trading memory for a small compression bias (Remark~\ref{rem:capped_patso_bias}).

\begin{figure}[t]
\centering
\begin{subfigure}{\textwidth}
  \centering
  \includegraphics[width=0.95\textwidth]{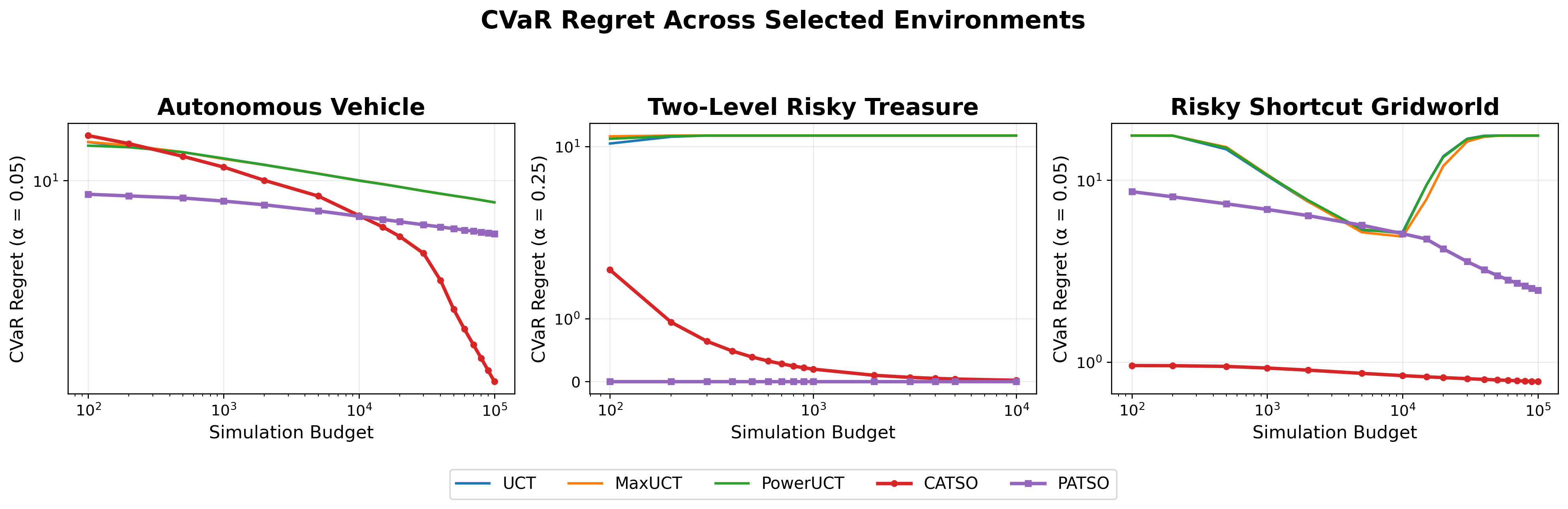}
  \caption{CVaR value gap across environments vs.\ simulation budget. Lower is better.}
  \label{fig:cvar-regret}
\end{subfigure}
\begin{subfigure}{\textwidth}
  \centering
  \includegraphics[width=0.95\textwidth]{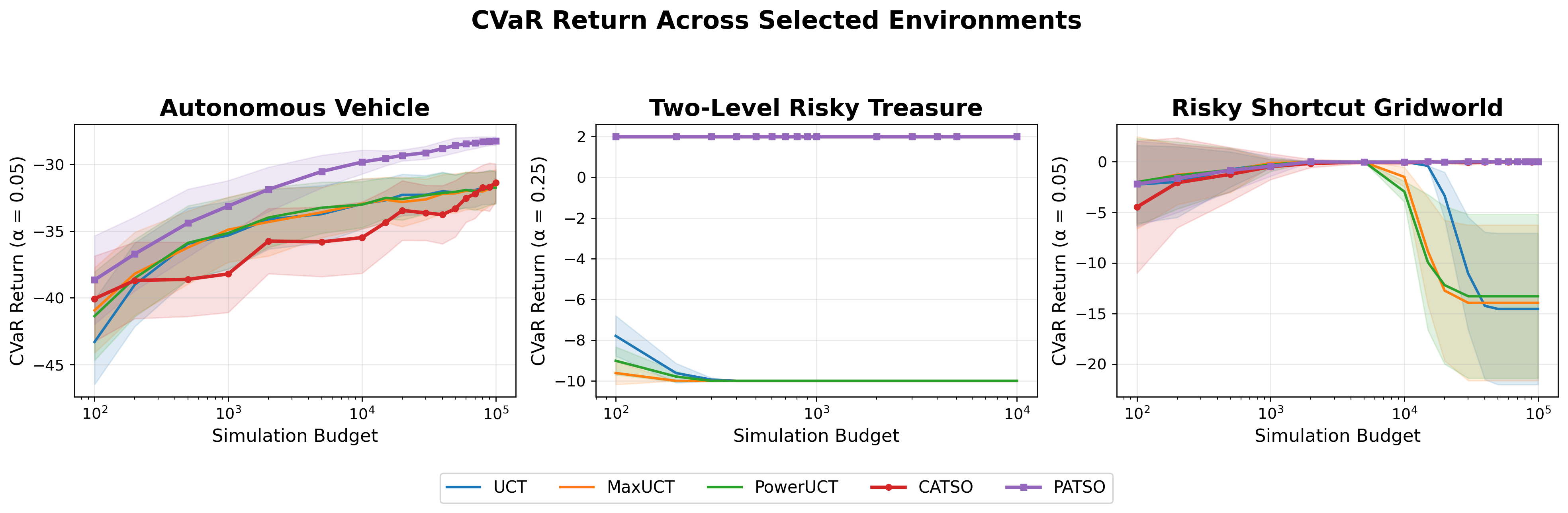}
  \caption{Monte-Carlo CVaR return of the policies recommended by each planner. Higher is better.}
  \label{fig:cvar-return}
\end{subfigure}
\caption{CVaR planning performance across environments: value gap during planning (top) and return of recommended policies (bottom) with $\pm$ standard deviation.}
\label{fig:cvar-combined}
\end{figure}
\begin{figure}[t]
\centering
\includegraphics[width=0.95\textwidth]{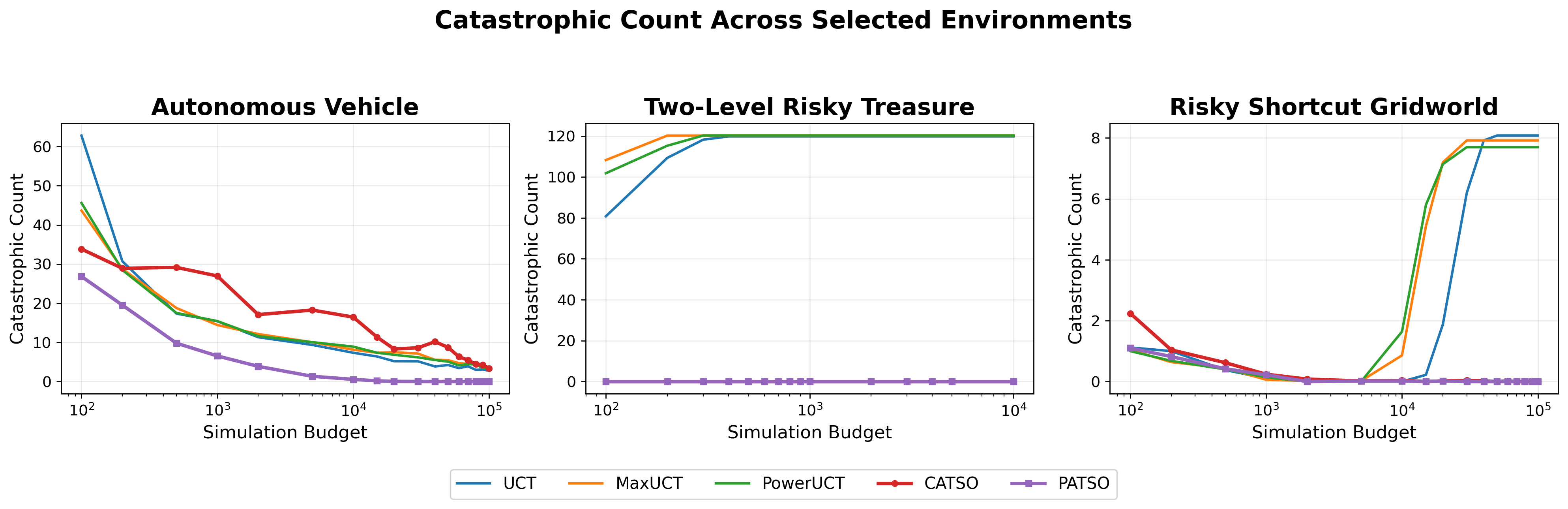}
\caption{Catastrophic-event count across evaluation rollouts. Lower is better.}
\label{fig:catastrophic-count}
\end{figure}

\subsection{Takeaways}
\label{sec:experiment_takeaways}

Three conclusions emerge. (i) Maintaining empirical Q-edge laws is useful in practice, not only in theory. (ii) The gains of \algcatso and \algpatso show up precisely on lower-tail metrics — CVaR value gap, Monte-Carlo CVaR return, catastrophic count. (iii) Scalar baselines remain competitive on expected return, indicating that the improvement reflects preserved tail information, not better exploration.

\section{Discussion}
\label{sec:discussion}

\algcatso and \algpatso show that MCTS can reason about lower-tail risk without storing full return distributions at every node.  The key design choice is to keep V-nodes scalar while maintaining empirical backup laws at Q-edges.  This allows the planner to use CVaR-sensitive Thompson sampling for exploration and a CVaR-average scalar backup for value propagation.  As a result, the methods preserve tail information that scalar mean-based MCTS discards, while retaining the standard $O(n^{-1/2})$ finite-depth convergence rate up to representation bias.

The theory is stated for a finite-depth nested CVaR target.  This target is appropriate for risk-sensitive tree search, but it is not the same as static CVaR of the full infinite-horizon trajectory return.  \algcatso also depends on a finite grid and therefore incurs discretization bias, while capped \algpatso trades memory for a tunable $W_1$ compression bias.  Future work should study distributional V-nodes, adaptive risk levels, learned model dynamics, and extensions to partially observed or continuous-control settings.
\clearpage
\bibliographystyle{plainnat}
\bibliography{main}

\clearpage
\appendix
\appendix
\begin{center}
{\bf \Large{Appendix of ``Beyond Expected Values: Risk-Sensitive Planning with Distributional Monte-Carlo Tree Search''}}
\end{center}
\DoToC
\medskip
\section{Notations}
\label{app:notation}
Table \ref{tab:notation-general} provides an overview of the standard notation adopted throughout this work. Context-specific notation, when required, will be explicitly defined within the corresponding definition.

\begin{table}[ht]
\centering
\small
\setlength{\tabcolsep}{5pt}
\begin{tabular}{@{}p{0.27\linewidth}p{0.68\linewidth}@{}}
\toprule
\textbf{Notation} & \textbf{Meaning} \\
\midrule
$\mathcal M=\langle\mathcal S,\mathcal A,\mathcal R,\mathcal P,\gamma\rangle$
& Discounted MDP with state space $\mathcal S$, action space $\mathcal A$,
reward kernel $\mathcal R$, transition kernel $\mathcal P$, and discount
$\gamma\in(0,1)$. \\

$s_0$ & Root state of the search tree. \\

$s_h$ & State or search-tree node at depth $h$. \\

$H$ & Finite search depth. \\

$\mathcal A(s)$ & Legal actions at state or node $s$. \\

$R(s,a,S')$ & Random reward obtained after taking action $a$ in state $s$ and
transitioning to $S'$. \\

$R_{\min},R_{\max}$ & Lower and upper reward bounds.  In the simplified
presentation, rewards are often assumed to lie in $[0,R_{\max}]$. \\

$\pi_{\mathrm{ro}}$ & Rollout policy used at terminal search depth. \\

$G_H^{\pi_{\mathrm{ro}}}(s)$ & Discounted rollout return from state $s$ at
depth $H$ under $\pi_{\mathrm{ro}}$. \\

$\tau$ & CVaR risk level.  Smaller $\tau$ focuses on a more extreme lower tail. \\

$\CVaR_\tau(X)$ & Lower-tail conditional value-at-risk of random variable $X$. \\

$\CVaR_\tau(P)$ & CVaR of a random variable with probability law $P$. \\

$\delta_x$ & Dirac point mass at $x$. \\

$\mathcal L(Y)$ & Probability law of random variable $Y$, i.e.,
$\mathcal L(Y)(B)=\mathbb P(Y\in B)$. \\

$W_1(P,Q)$ & 1-Wasserstein distance between probability laws $P$ and $Q$. \\

$\plim_{n\to\infty}\widehat V_n=V$
& Convergence with concentration as in Definition~\ref{def:concentration}. \\
\bottomrule
\end{tabular}
\caption{General notation.}
\label{tab:notation-general}
\end{table}

\section{CVaR and Wasserstein Preliminaries}
\label{app:cvar_wasserstein}
  \subsection{Discrete CVaR Formula}
  \label{sec:cvar_derivation}
\paragraph{Derivation for \eqref{eq:cvar}} Recall the quantile function
$F_X^{-1}(u)=\inf\{x:F_X(x)\ge u\}$ for $u\in(0,1]$.
By the inverse-transform identity, if $U\sim\mathrm{Unif}(0,1)$ then
$F_X^{-1}(U)\stackrel{d}{=}X$, so conditioning on $U\le\tau$ samples
the lower $\tau$-tail of $X$. Averaging yields
\begin{equation}\label{eq:cvar-cond}
\CVaR_\tau(X)
=\mathbb{E}\left[F_X^{-1}(U)\big|\,U\le\tau\right]
=\frac{1}{\tau}\int_0^\tau F_X^{-1}(u)\,du,
\end{equation}
which, when $F_X$ is continuous and strictly increasing, reduces to the
familiar $\mathbb{E}[X\mid X\le\VaR_\tau(X)]$. An equivalent
Rockafellar--Uryasev variational form,
\begin{equation}\label{eq:ru}
\CVaR_\tau(X)
=\sup_{\eta\in\mathbb{R}}
\Bigl\{\eta-\tfrac{1}{\tau}\mathbb{E}\bigl[(\eta-X)_{+}\bigr]\Bigr\},
\end{equation}
handles atoms cleanly and is the form typically used for optimization.

\paragraph{From \eqref{eq:cvar} to \eqref{eq:disc-cvar}.}
For $P=\sum_{i=1}^m w_i\delta_{x_i}$ with $x_1\le\cdots\le x_m$, the CDF
$F(x)=\sum_{i:x_i\le x}w_i$ is a right-continuous step function, so the
quantile function is piecewise constant:
\begin{equation}\label{eq:disc-quantile}
F_X^{-1}(u)=x_j
\qquad\text{for } u\in(c_{j-1},c_j],\quad j=1,\dots,m,
\end{equation}
where $c_j=\sum_{i=1}^j w_i$ and $c_0=0$. Letting
$k=\min\{j:c_j\ge\tau\}$ and splitting the integral at $c_{k-1}$,
\begin{align*}
\int_0^\tau F_X^{-1}(u)\,du
&=\sum_{j=1}^{k-1}\int_{c_{j-1}}^{c_j}x_j\,du
 \;+\;\int_{c_{k-1}}^{\tau}x_k\,du \\
&=\sum_{j=1}^{k-1} w_j\,x_j \;+\; (\tau-c_{k-1})\,x_k.
\end{align*}
The first sum collects the atoms fully contained in the lower tail
(each contributing its full mass $w_j$), while the second term takes
the fractional remainder $\tau-c_{k-1}$ of the boundary atom $x_k$ so
that the total mass integrated equals exactly $\tau$. Dividing by
$\tau$ gives~\eqref{eq:disc-cvar}.

\begin{remark}
When $\tau=c_k$ for some index $k$, the boundary term vanishes and
\eqref{eq:disc-cvar} reduces to the conditional mean
$\tfrac{1}{c_k}\sum_{i=1}^k w_i x_i$ over the smallest $k$ atoms, as
expected.
\end{remark}

\subsection{CVaR--Wasserstein Lipschitz Property}

We first record the well-known Lipschitz property of CVaR with respect to the
Wasserstein-1 distance.

\begin{lemma}[CVaR is $1/\tau$-Lipschitz in $W_1$]\label{lem:cvar_lip_app}
Let $P$ and $Q$ be probability distributions supported on $[0,M]$. Then for any $\tau \in (0,1]$:
\[
\bigl|\CVaR_\tau(P) - \CVaR_\tau(Q)\bigr| \leq \frac{1}{\tau}\,W_1(P,Q).
\]
\end{lemma}

\begin{proof}
Recall $\CVaR_\tau(P) = \frac{1}{\tau}\int_0^\tau F_P^{-1}(u)\,du$ where $F_P^{-1}$ is the quantile function.
By the quantile representation of $W_1$:
\[
W_1(P,Q) = \int_0^1 |F_P^{-1}(u) - F_Q^{-1}(u)|\,du \;\ge\; \int_0^\tau |F_P^{-1}(u) - F_Q^{-1}(u)|\,du.
\]
Therefore:
\begin{align*}
\bigl|\CVaR_\tau(P) - \CVaR_\tau(Q)\bigr|
&= \frac{1}{\tau}\biggl|\int_0^\tau \bigl(F_P^{-1}(u) - F_Q^{-1}(u)\bigr)\,du\biggr|\\
&\le \frac{1}{\tau}\int_0^\tau |F_P^{-1}(u) - F_Q^{-1}(u)|\,du\\
&\le \frac{1}{\tau}\,W_1(P,Q).
\end{align*}
\end{proof}

\subsection{Empirical \texorpdfstring{$W_1$}{W1} Concentration}

This subsection proves Theorem~\ref{thm:w1_app} in a way that keeps the random visit counts explicit. The object
controlled below is the deterministic-prefix empirical law of the first $m$ backups recorded at the
root edge. The stopped-time child-value term is left visible, rather than implicitly replacing a
random visit count by a deterministic sample size inside the proof.

\begin{theorem}[Root-backup $W_1$ reduction]\label{thm:w1_app}
Fix a root action $a\in\mathcal A_{s_0}$ and let $\tau_t$ be the global simulation index of the
$t$-th visit to edge $(s_0,a)$. On that visit, write $(r_t,S'_t)$ for the sampled reward and next
state, and define
\[
\overline Q_t = r_t + \gamma \widehat V_{\tau_t}(S'_t),
\qquad
G_t = r_t + \gamma \widetilde V(S'_t).
\]
Let
\[
\widehat{\mathcal D}_m(s_0,a) := \frac{1}{m}\sum_{t=1}^m \delta_{\overline Q_t},
\qquad
\mathcal D(s_0,a) := \mathcal L(G_1).
\]
Then \algpatso satisfies
\[
\mathbb E\bigl[W_1\bigl(\widehat{\mathcal D}_m(s_0,a),\mathcal D(s_0,a)\bigr)\bigr]
\le
\frac{\gamma}{m}\sum_{t=1}^m \Delta_t(a)
+
C_{\mathrm{FG}}\frac{R_{\max}}{1-\gamma}\,m^{-1/2},
\]
where
\[
\Delta_t(a):=\mathbb E\left[\left|\widehat V_{\tau_t}(S'_t)-\widetilde V(S'_t)\right|\right].
\]
In particular, if the stopped-time child-value errors satisfy
$\Delta_t(a)\le C_V t^{-1/2}$ for all $t\ge 1$, then
\[
\mathbb E\bigl[W_1\bigl(\widehat{\mathcal D}_m(s_0,a),\mathcal D(s_0,a)\bigr)\bigr]
\le
\left(2\gamma C_V + C_{\mathrm{FG}}\frac{R_{\max}}{1-\gamma}\right)m^{-1/2}.
\]
For \algcatso, add $O(R_{\max}/(N(1{-}\gamma)))$ discretization error.
\end{theorem}

\begin{proof}
Fix $a\in\mathcal A_{s_0}$ and an integer $m\ge 1$. Define the oracle empirical law
\[
\widehat{\mathcal D}^{\mathrm{oracle}}_m(s_0,a)
:= \frac{1}{m}\sum_{t=1}^m \delta_{G_t}.
\]
Because each visit to the root edge draws a fresh reward/next-state sample from
$(\mathcal R,P)(\cdot\mid s_0,a)$, the oracle samples $\{G_t\}_{t=1}^m$ are i.i.d. with law
$\mathcal D(s_0,a)$.

By the triangle inequality for $W_1$,
\begin{equation}\label{eq:w1_triangle_fixedprefix}
W_1\bigl(\widehat{\mathcal D}_m,\mathcal D\bigr)
\le
\underbrace{W_1\bigl(\widehat{\mathcal D}_m,\widehat{\mathcal D}^{\mathrm{oracle}}_m\bigr)}_{\text{(I)}}
+
\underbrace{W_1\bigl(\widehat{\mathcal D}^{\mathrm{oracle}}_m,\mathcal D\bigr)}_{\text{(II)}}.
\end{equation}

\paragraph{I. coupling term.}
Consider the natural coupling
\[
\pi_m := \frac{1}{m}\sum_{t=1}^m \delta_{(\overline Q_t,G_t)}.
\]
Its marginals are $\widehat{\mathcal D}_m$ and $\widehat{\mathcal D}^{\mathrm{oracle}}_m$, so the
coupling characterization of $W_1$ gives
\begin{align}
W_1\bigl(\widehat{\mathcal D}_m,\widehat{\mathcal D}^{\mathrm{oracle}}_m\bigr)
&\le \int |x-y|\,d\pi_m(x,y) \nonumber\\
&= \frac{1}{m}\sum_{t=1}^m |\overline Q_t-G_t| \nonumber\\
&= \frac{\gamma}{m}\sum_{t=1}^m
\left|\widehat V_{\tau_t}(S'_t)-\widetilde V(S'_t)\right|.\label{eq:w1_coupling_exact}
\end{align}
Taking expectations in \eqref{eq:w1_coupling_exact} yields
\[
\mathbb E[\text{(I)}]
\le
\frac{\gamma}{m}\sum_{t=1}^m
\mathbb E\left[\left|\widehat V_{\tau_t}(S'_t)-\widetilde V(S'_t)\right|\right]
=
\frac{\gamma}{m}\sum_{t=1}^m \Delta_t(a).
\]

\paragraph{II. i.i.d. empirical-measure term.}
The oracle law $\mathcal D(s_0,a)$ is supported on $[0,R_{\max}/(1-\gamma)]$. By the
one-dimensional bounded-support case of \citet{fournier2015rate}, there exists a universal constant
$C_{\mathrm{FG}}>0$ such that
\[
\mathbb E\bigl[W_1\bigl(\widehat{\mathcal D}^{\mathrm{oracle}}_m,\mathcal D\bigr)\bigr]
\le
C_{\mathrm{FG}}\frac{R_{\max}}{1-\gamma}\,m^{-1/2}.
\]

\paragraph{III. combine.}
Substituting the bounds from Steps~1--2 into \eqref{eq:w1_triangle_fixedprefix} proves the first
claim:
\[
\mathbb E\bigl[W_1\bigl(\widehat{\mathcal D}_m(s_0,a),\mathcal D(s_0,a)\bigr)\bigr]
\le
\frac{\gamma}{m}\sum_{t=1}^m \Delta_t(a)
+
C_{\mathrm{FG}}\frac{R_{\max}}{1-\gamma}\,m^{-1/2}.
\]
If moreover $\Delta_t(a)\le C_V t^{-1/2}$ for all $t\ge 1$, then
\[
\frac{\gamma}{m}\sum_{t=1}^m \Delta_t(a)
\le
\frac{\gamma C_V}{m}\sum_{t=1}^m t^{-1/2}
\le 2\gamma C_V m^{-1/2},
\]
which gives the displayed $O(m^{-1/2})$ bound. For \algcatso, nearest-atom projection contributes
an additional $O(R_{\max}/(N(1-\gamma)))$ term in $W_1$.
\end{proof}

\begin{remark}[Why Theorem~5.3 is stated with a fixed prefix]\label{rem:fixedprefix_randomcounts}
The counts $T_a(n)$ and $T_{s,a}(n)$ are random variables. Theorem~5.3 therefore controls the
empirical law $\widehat{\mathcal D}_m(s_0,a)$ for a deterministic prefix length $m$. Passing from a
fixed $m$ to the adaptive choice $m=T_{s_0,a}(n)$ requires a separate stopped-time argument; that
extra step is not hidden inside the proof above.
\end{remark}

\section{Local CVaR Bandit Analysis}
\label{app:local_bandits}
\subsection{Non-Stationary CVaR Bandit Setup}

Throughout Sections~\ref{sec:catso_cvar_tails}--\ref{sec:patso_cvar_bandit} we fix a CVaR level
$\tau\in(0,1]$ and polynomial-bonus exponents
\[
\mathrm{bonus}(t,n)
:= C\frac{t^{1/4}}{n^{1/2}}.
\]

For each arm $a\in[K]$, let $F_a$ denote its return law and define
\[
c_a^\tau := \CVaR_\tau(F_a),
\qquad
c_\star^\tau := \max_{a\in[K]} c_a^\tau,
\qquad
a_\tau^\star \in \arg\max_{a\in[K]} c_a^\tau.
\]
We assume the maximizer is unique whenever a suboptimal-pull statement is made.

For a probability law $F$ on $[0,R]$ and a threshold $c\in\mathbb R$, define the upper-tail risk
information number
\[
\mathcal K_{\inf}^{\tau,+}(F,c)
:=
\inf\Bigl\{
\mathrm{KL}(F\|G):
G\in\mathcal P([0,R]),\
\CVaR_\tau(G)\ge c
\Bigr\}.
\]
On a fixed grid $\mathcal Z_N=\{z_0<\cdots<z_N\}\subset[0,R]$, define the grid-restricted version
\[
\mathcal K_{\inf}^{(N),\tau,+}(F,c)
:=
\inf\Bigl\{
\mathrm{KL}(F\|G):
G\in\mathcal P(\mathcal Z_N),\
\CVaR_\tau(G)\ge c
\Bigr\}.
\]

When the empirical law of arm $a$ after $n$ pulls is $\widehat F_{a,n}$ and the maintained posterior
is sampled as a random law $\widetilde P_{a,n}$, the posterior-sampled CVaR score is
\[
c_{a,n}^\tau := \CVaR_\tau(\widetilde P_{a,n}).
\]
The risk-sensitive Thompson-with-bonus index is then
\[
U_t(a) := c_{a,T_a(t)}^\tau + C\frac{t^{1/4}}{(T_a(t)\vee 1)^{1/2}},
\]
where $T_a(t)$ is the number of pulls of arm $a$ before round $t$.

For the detailed setup and assumptions, refer to section~\ref{sec:cvar_bandit}.

\subsection{\algcatso posterior CVaR tails on a fixed grid}
\label{sec:catso_cvar_tails}

\algcatso is the only place where we currently prove a posterior large-deviation statement for a
CVaR-based Thompson score. The proof uses the fixed-simplex Dirichlet geometry of
\citet{riou2020bandit}.

\begin{lemma}[Dirichlet convex-set tail; \citet{riou2020bandit}]
\label{lem:dirichlet-convex-tail}
Let $\Delta_N := \{q\in[0,1]^{N+1} : \sum_{i=0}^N q_i = 1\}$. For $\alpha\in\mathbb N^{N+1}$ with $\alpha_i\ge 1$, set $A := \sum_i \alpha_i$ and $p := \alpha/A \in \Delta_N$, and let $L\sim\mathrm{Dir}(\alpha)$. If $C\subseteq\Delta_N$ is closed and convex and $p\notin C$, then
\[
    \PP(L\in C) \le C_N^{\rm RH}\,A^{N/2}\exp\left\{-A\inf_{q\in C}\mathrm{KL}(p\|q)\right\}, \quad C_N^{\rm RH} := \tfrac{e^{1/12}}{\Gamma(N+1)}\bigl(\tfrac{1}{\sqrt{2\pi}}\bigr)^N,
\]
where $\mathrm{KL}(p\|q) := \sum_i p_i\log(p_i/q_i)$, with $\mathrm{KL}(p\|q) = +\infty$ if $p_i>0$ and $q_i=0$ for some $i$.
\end{lemma}

\begin{proposition}[\algcatso posterior CVaR upper tail on a fixed grid]
\label{prop:catso-cvar-upper-tail-fixed-grid}
Fix a risk level $\tau\in(0,1]$ and an ordered grid $Z_N = \{z_0<z_1<\cdots<z_N\}\subset[0,R]$. Let $\alpha = \alpha_a(m)\in\mathbb N^{N+1}$ with $\alpha_{a,i}(m)\ge 1$, set $A = A_{a,m} := \sum_i \alpha_{a,i}(m)$ and $p = P_{\alpha_a(m)} := \alpha/A\in\Delta_N$, and draw $L_a(m)\sim\mathrm{Dir}(\alpha_a(m))$. Define the Thompson-sampled categorical law and its lower-tail CVaR
\[
    \widehat D^{\rm TS}_{a,m} := \sum_i L_{a,i}(m)\delta_{z_i}, \qquad \widehat c^\tau_{a,m} := \mathrm{CVaR}_\tau\bigl(\widehat D^{\rm TS}_{a,m}\bigr),
\]
and for $q\in\Delta_N$ write $\mathrm{CVaR}_\tau(q;Z_N) := \mathrm{CVaR}_\tau(\sum_i q_i\delta_{z_i}) = \tfrac{1}{\tau}\int_0^\tau F_q^{-1}(u)\,du$, with the grid-restricted upper-tail information number $K_{\inf}^{(N),\tau,+}(p,c) := \inf\{\mathrm{KL}(p\|q) : q\in\Delta_N,\, \mathrm{CVaR}_\tau(q;Z_N)\ge c\}$. Then, for every $c\in\mathbb R$,
\[
    \PP\left(\widehat c^\tau_{a,m}\ge c\middle|\alpha_a(m)\right) \le \begin{cases} (N+1)\,C_N^{\rm RH}\,A_{a,m}^{N/2}\exp\left\{-A_{a,m}\,K_{\inf}^{(N),\tau,+}\bigl(P_{\alpha_a(m)},c\bigr)\right\}, & c > \mathrm{CVaR}_\tau(P_{\alpha_a(m)};Z_N), \\[0.4em] 1, & c \le \mathrm{CVaR}_\tau(P_{\alpha_a(m)};Z_N). \end{cases}
\]
Equivalently, the \citet{riou2020bandit} exponential bound is valid in the genuine upper-tail regime $K_{\inf}^{(N),\tau,+}(P_{\alpha_a(m)},c) > 0$.
\end{proposition}

\begin{proof}
Condition on the vector \(\alpha=\alpha_a(m)\). Throughout the proof,
write
\[
A=A_{a,m},
\qquad
p=P_{\alpha_a(m)}=\frac{\alpha}{A}.
\]

If
\[
c\le \mathrm{CVaR}_\tau(p;Z_N),
\]
then the claimed bound is simply the trivial inequality
\[
\mathbb P\left(
\widehat c^\tau_{a,m}\ge c
\middle|\,
\alpha_a(m)
\right)\le 1.
\]
Thus it remains to prove the exponential bound in the case
\[
c>\mathrm{CVaR}_\tau(p;Z_N).
\]

For \(q=(q_0,\ldots,q_N)\in\Delta_N\), define cumulative sums
\[
S_{-1}(q):=0,
\qquad
S_k(q):=\sum_{i=0}^k q_i,
\qquad
0\le k\le N.
\]
For each \(k\in\{0,\ldots,N\}\), define
\[
C_k^+(c)
:=
\left\{
q\in\Delta_N:
S_{k-1}(q)\le \tau \le S_k(q),
\quad
\tau z_k+\sum_{i=0}^{k-1}q_i(z_i-z_k)\ge \tau c
\right\},
\]
where an empty sum is interpreted as zero.

We first identify the CVaR superlevel set as a finite union of these
sets. Fix \(q\in\Delta_N\). If
\[
S_{k-1}(q)\le \tau\le S_k(q),
\]
then the lower \(\tau\)-tail of the discrete law
\(\sum_{i=0}^N q_i\delta_{z_i}\) consists of all atoms
\(z_0,\ldots,z_{k-1}\) and a fractional amount
\(\tau-S_{k-1}(q)\) of the atom \(z_k\). Hence
\[
\begin{aligned}
\tau\mathrm{CVaR}_\tau(q;Z_N)
&=
\int_0^\tau F_q^{-1}(u)\,du
\\
&=
\sum_{i=0}^{k-1} q_i z_i
+
\bigl(\tau-S_{k-1}(q)\bigr)z_k
\\
&=
\tau z_k
+
\sum_{i=0}^{k-1}q_i(z_i-z_k).
\end{aligned}
\]
Therefore, on the region \(S_{k-1}(q)\le\tau\le S_k(q)\),
the condition
\[
\mathrm{CVaR}_\tau(q;Z_N)\ge c
\]
is exactly the affine inequality
\[
\tau z_k+\sum_{i=0}^{k-1}q_i(z_i-z_k)\ge \tau c.
\]
Since for every \(q\in\Delta_N\) there exists at least one
\(k\in\{0,\ldots,N\}\) such that
\(S_{k-1}(q)\le\tau\le S_k(q)\), we obtain
\[
\left\{
q\in\Delta_N:
\mathrm{CVaR}_\tau(q;Z_N)\ge c
\right\}
=
\bigcup_{k=0}^N C_k^+(c).
\]

Each set \(C_k^+(c)\) is a closed convex polytope: it is the intersection
of the simplex with finitely many affine half-space constraints,
namely
\[
S_{k-1}(q)\le \tau,
\qquad
S_k(q)\ge \tau,
\qquad
\tau z_k+\sum_{i=0}^{k-1}q_i(z_i-z_k)\ge \tau c.
\]

Because we are in the case
\[
c>\mathrm{CVaR}_\tau(p;Z_N),
\]
we have
\[
p\notin
\left\{
q\in\Delta_N:
\mathrm{CVaR}_\tau(q;Z_N)\ge c
\right\}.
\]
Consequently,
\[
p\notin C_k^+(c)
\qquad
\text{for every } k=0,\ldots,N.
\]
Thus Lemma~\ref{lem:dirichlet-convex-tail} applies to each
\(C_k^+(c)\). With the convention that
\(\inf_{\varnothing}\mathrm{KL}(p\|\cdot)=+\infty\), we get
\[
\mathbb P\left(
L_a(m)\in C_k^+(c)
\middle|\,
\alpha_a(m)
\right)
\le
C_N^{\rm RH} A^{N/2}
\exp\left\{
-A\inf_{q\in C_k^+(c)}\mathrm{KL}(p\|q)
\right\}.
\]
Using the union representation of the superlevel event and a union bound,
\[
\begin{aligned}
\mathbb P\left(
\widehat c^\tau_{a,m}\ge c
\middle|\,
\alpha_a(m)
\right)
&=
\mathbb P\left(
L_a(m)\in
\bigcup_{k=0}^N C_k^+(c)
\middle|\,
\alpha_a(m)
\right)
\\
&\le
\sum_{k=0}^N
\mathbb P\left(
L_a(m)\in C_k^+(c)
\middle|\,
\alpha_a(m)
\right)
\\
&\le
C_N^{\rm RH}A^{N/2}
\sum_{k=0}^N
\exp\left\{
-A\inf_{q\in C_k^+(c)}\mathrm{KL}(p\|q)
\right\}
\\
&\le
(N+1)C_N^{\rm RH}A^{N/2}
\exp\left\{
-A\min_{0\le k\le N}
\inf_{q\in C_k^+(c)}\mathrm{KL}(p\|q)
\right\}.
\end{aligned}
\]
Finally, since
\[
\left\{
q\in\Delta_N:
\mathrm{CVaR}_\tau(q;Z_N)\ge c
\right\}
=
\bigcup_{k=0}^N C_k^+(c),
\]
we have
\[
\min_{0\le k\le N}
\inf_{q\in C_k^+(c)}\mathrm{KL}(p\|q)
=
\inf\left\{
\mathrm{KL}(p\|q):
q\in\Delta_N,\ 
\mathrm{CVaR}_\tau(q;Z_N)\ge c
\right\}
=
K_{\inf}^{(N),\tau,+}(p,c).
\]
Substituting \(p=P_{\alpha_a(m)}\) and \(A=A_{a,m}\) proves the claimed
upper-tail bound.
\end{proof}

\begin{remark}[Why the statement is not valid with the same exponential
prefactor for all \(c\)]
If
\[
c\le \mathrm{CVaR}_\tau(P_{\alpha_a(m)};Z_N),
\]
then \(P_{\alpha_a(m)}\) itself belongs to the CVaR superlevel set, and
therefore
\[
K_{\inf}^{(N),\tau,+}
\bigl(P_{\alpha_a(m)},c\bigr)=0.
\]
In this non-tail regime, the \citet{riou2020bandit} convex-set tail lemma is not
applicable because its end-tail condition \(p\notin C\) fails. Moreover,
the prefactor
\[
(N+1)C_N^{\rm RH}A_{a,m}^{N/2}
\]
can be smaller than \(1\) for small \(A_{a,m}\), so the original
``for every \(c\)'' exponential display is not a valid replacement for
the trivial probability bound \(1\).
\end{remark}

\begin{remark}[Fractional \algcatso projection weights]
The proof above is fully justified by the integer-parameter form of the
\citet{riou2020bandit} lemma. If \algcatso uses linear categorical projection and updates
\[
\alpha_i\leftarrow \alpha_i+\omega_i,
\qquad
\omega_i\in[0,1],
\qquad
\sum_i\omega_i=1,
\]
then the posterior parameters need not remain integer-valued. In that
case one should either use hard binning, so that \(\alpha_i\in\mathbb N\),
or explicitly assume/prove the following real-parameter analogue: for
\(\alpha\in[1,\infty)^{N+1}\), \(A=\sum_i\alpha_i\), \(p=\alpha/A\), and
every closed convex \(C\subseteq\Delta_N\) with \(p\notin C\),
\[
\mathbb P(L\in C)
\le
C_N^{\rm Dir} A^{N/2}
\exp\left\{
-A\inf_{q\in C}\mathrm{KL}(p\|q)
\right\},
\]
for a dimension-only constant \(C_N^{\rm Dir}\). Under this real-parameter
lemma, the proof of Proposition~\ref{prop:catso-cvar-upper-tail-fixed-grid}
is unchanged, with \(C_N^{\rm RH}\) replaced by \(C_N^{\rm Dir}\).
\end{remark}

\subsection{\algcatso CVaR suboptimal-arm pulls}
\label{sec:bandits-convergence}

This section provides CVaR suboptimal-arm pull bounds. The result is stated in \emph{reduction form}: the post-convergence part is proved, and the
optimal-arm lower-tail contribution is isolated as an explicit remainder.

\begin{lemma}[\algcatso posterior-center concentration on a fixed grid]
\label{lem:catso_posterior_center_concentration}
Fix an arm $a$ and grid $Z_N = \{z_0<z_1<\cdots<z_N\}$. On the $r$-th local pull of arm $a$, \algcatso adds a random weight vector $W_{a,r} = (W_{a,r,0},\ldots,W_{a,r,N})\in\Delta_N$ to the Dirichlet parameter, with $\EE[W_{a,r}\mid\mathcal H_{a,r-1}] = p_a$ a.s.\ for all $r\ge 1$ (some $p_a\in\Delta_N$) and $0\le W_{a,r,i}\le 1$. With initial prior $\alpha_{a,i}(0) = 1$ for all $i$, define $\alpha_{a,i}(m) = 1 + \sum_{r=1}^m W_{a,r,i}$ and $P_{\alpha_a(m)} := \alpha_a(m)/(m+N+1)$. Then for every $\varepsilon>0$,
\[
    \sum_{m=1}^{\infty}\PP\bigl(\|P_{\alpha_a(m)}-p_a\|_\infty > \varepsilon\bigr) < \infty.
\]
\end{lemma}

\begin{proof}
Write
\[
\overline W_{a,m}:=\frac1m\sum_{r=1}^m W_{a,r}.
\]
Then
\[
P_{\alpha_a(m)}
=
\frac{\mathbf 1+m\overline W_{a,m}}{m+N+1},
\]
where \(\mathbf 1=(1,\ldots,1)\). Hence
\[
P_{\alpha_a(m)}-p_a
=
\frac{m}{m+N+1}(\overline W_{a,m}-p_a)
+
\frac{\mathbf 1-(N+1)p_a}{m+N+1}.
\]
Let
\[
B_a:=\|\mathbf 1-(N+1)p_a\|_\infty.
\]
Choose
\[
m_\varepsilon
:=
\left\lceil \frac{2B_a}{\varepsilon}\right\rceil\vee 1.
\]
For all \(m\ge m_\varepsilon\),
\[
\frac{B_a}{m+N+1}\le \frac{B_a}{m}\le \frac{\varepsilon}{2}.
\]
Therefore, for \(m\ge m_\varepsilon\),
\[
\left\{
\|P_{\alpha_a(m)}-p_a\|_\infty>\varepsilon
\right\}
\subseteq
\left\{
\|\overline W_{a,m}-p_a\|_\infty>\frac{\varepsilon}{2}
\right\}.
\]
For each coordinate \(i\), the sequence
\[
W_{a,r,i}-p_{a,i}
\]
is a bounded martingale-difference sequence with range length at most \(1\).
By Hoeffding--Azuma,
\[
\mathbb P\left(
\left|
\frac1m\sum_{r=1}^m(W_{a,r,i}-p_{a,i})
\right|
>
\frac{\varepsilon}{2}
\right)
\le
2\exp\left(-\frac{m\varepsilon^2}{2}\right).
\]
Taking a union bound over \(i=0,\ldots,N\), for all
\(m\ge m_\varepsilon\),
\[
\mathbb P\left(
\|P_{\alpha_a(m)}-p_a\|_\infty>\varepsilon
\right)
\le
2(N+1)\exp\left(-\frac{m\varepsilon^2}{2}\right).
\]
Thus
\[
\sum_{m=1}^{\infty}
\mathbb P\left(
\|P_{\alpha_a(m)}-p_a\|_\infty>\varepsilon
\right)
\le
m_\varepsilon
+
2(N+1)
\sum_{m=m_\varepsilon}^{\infty}
\exp\left(-\frac{m\varepsilon^2}{2}\right)
<\infty.
\]
\end{proof}

Define the optimal-arm lower-tail term
\[
\mathcal L_{n,\mathrm{\algcatso}}^\tau
:=
\mathbb E\left[
\sum_{t=1}^n
\mathbf 1\left\{
\widetilde c_{a_\tau^\star}^\tau(t)
<
c_\star^\tau
-
C\frac{t^{1/4}}
{(N_{a_\tau^\star}(t)\vee1)^{1/2}}
\right\}
\right].
\]

\begin{lemma}[Optimal-arm \algcatso lower-tail term]
\label{lem:catso_optimal_lower_tail_term}
Assume the fixed-grid \algcatso setting of Lemma~\ref{lm_suboptimal_arm_pull_cats_appendix} with each arm pulled once initially, and write $a^\tau_\star := a_\tau^\star$, $p_\star := p_{a^\tau_\star}$, $c_\star^\tau := \mathrm{CVaR}_\tau(p_\star;Z_N)$. Assume the fixed-grid posterior-center concentration: there exist $B_N, b_N>0$ such that for every $m\ge 1$ and $s>0$,
\[
    \PP\left(\|P_{\alpha_{a^\tau_\star}}^{(m)}-p_\star\|_1 > s\right) \le B_N\exp(-b_N m s^2).
\]
Assume also the fixed-grid Dirichlet lower-tail bound: with $K_{\inf}^{(N),\tau,-}(p,c) := \inf\{\mathrm{KL}(p\|q) : q\in\Delta_N,\, \mathrm{CVaR}_\tau(q;Z_N)\le c\}$, for $L\sim\mathrm{Dir}(\alpha)$, $A = \sum_i\alpha_i$, $p = \alpha/A$, and any $c<\mathrm{CVaR}_\tau(p;Z_N)$,
\[
    \PP\left(\mathrm{CVaR}_\tau\Bigl(\textstyle\sum_i L_i\delta_{z_i}\Bigr) \le c \Big|\, \alpha\right) \le C_N^{\rm low}\,A^{N/2}\exp\left\{-A\,K_{\inf}^{(N),\tau,-}(p,c)\right\}.
\]
Then $\mathcal L_{n,\mathrm{\algcatso}}^\tau = O(1)$, and in particular $\mathcal L_{n,\mathrm{\algcatso}}^\tau = O(n^{1/2})$.
\end{lemma}

\begin{proof}
Write
\[
m_t:=N_{a^\tau_\star}(t)\vee1,
\qquad
\Delta_t
:=
C\frac{t^{1/4}}{m_t^{1/2}}.
\]
Let
\[
E_t
:=
\left\{
\widetilde c_{a^\tau_\star}^\tau(t)
<
c_\star^\tau-\Delta_t
\right\}.
\]
It is enough to show that
\[
\sum_{t=1}^\infty \mathbb P(E_t)<\infty.
\]

For \(m\ge1\), define
\[
\Delta_{t,m}:=
C\frac{t^{1/4}}{m^{1/2}}.
\]
On the event \(\{m_t=m\}\), we have \(\Delta_t=\Delta_{t,m}\), and
\[
m\Delta_{t,m}^2=C^2t^{1/2}.
\]

We first record a quadratic lower bound for the lower-tail information number.
The CVaR functional is \(1/\tau\)-Lipschitz in \(W_1\), and on \([0,R]\),
\[
W_1(p,q)\le \frac{R}{2}\|p-q\|_1.
\]
Hence
\[
\left|
\mathrm{CVaR}_\tau(p;Z_N)
-
\mathrm{CVaR}_\tau(q;Z_N)
\right|
\le
\frac{R}{2\tau}\|p-q\|_1.
\]
Let
\[
G_{t,m}
:=
\left\{
\|P_{\alpha_{a^\tau_\star}}^{(m)}-p_\star\|_1
\le
\frac{\tau}{R}\Delta_{t,m}
\right\}.
\]
On \(G_{t,m}\),
\[
\mathrm{CVaR}_\tau(P_{\alpha_{a^\tau_\star}}^{(m)};Z_N)
\ge
c_\star^\tau-\frac{\Delta_{t,m}}{2}.
\]
Therefore, for every \(q\in\Delta_N\) satisfying
\[
\mathrm{CVaR}_\tau(q;Z_N)
\le
c_\star^\tau-\Delta_{t,m},
\]
we have
\[
\mathrm{CVaR}_\tau(P_{\alpha_{a^\tau_\star}}^{(m)};Z_N)
-
\mathrm{CVaR}_\tau(q;Z_N)
\ge
\frac{\Delta_{t,m}}{2}.
\]
Using the same Lipschitz bound again,
\[
\|P_{\alpha_{a^\tau_\star}}^{(m)}-q\|_1
\ge
\frac{\tau}{R}\Delta_{t,m}.
\]
By Pinsker's inequality,
\[
\mathrm{KL}(P_{\alpha_{a^\tau_\star}}^{(m)}\|q)
\ge
\frac12
\|P_{\alpha_{a^\tau_\star}}^{(m)}-q\|_1^2
\ge
\frac{\tau^2}{2R^2}\Delta_{t,m}^2.
\]
Thus, on \(G_{t,m}\),
\[
K_{\inf}^{(N),\tau,-}
\left(
P_{\alpha_{a^\tau_\star}}^{(m)},
c_\star^\tau-\Delta_{t,m}
\right)
\ge
\lambda_\tau\Delta_{t,m}^2,
\qquad
\lambda_\tau:=\frac{\tau^2}{2R^2}.
\]

Now decompose
\[
\mathbb P(E_t)
\le
\sum_{m=1}^t
\mathbb P(E_t,\ m_t=m,\ G_{t,m})
+
\sum_{m=1}^t
\mathbb P(m_t=m,\ G_{t,m}^c).
\]

For the bad-center part, by the assumed posterior-center concentration,
\[
\mathbb P(G_{t,m}^c)
\le
B_N
\exp\left(
-b_Nm\frac{\tau^2}{R^2}\Delta_{t,m}^2
\right)
=
B_N
\exp\left(
-b_N\frac{\tau^2C^2}{R^2}t^{1/2}
\right).
\]
Therefore
\[
\sum_{m=1}^t
\mathbb P(m_t=m,\ G_{t,m}^c)
\le
B_N t
\exp\left(
-b_N\frac{\tau^2C^2}{R^2}t^{1/2}
\right).
\]

For the good-center part, condition on the pre-round history. On
\(\{m_t=m\}\cap G_{t,m}\), the fixed-grid Dirichlet lower-tail bound gives
\[
\mathbb P(E_t\mid \mathcal F_{t-1})
\le
C_N^{\rm low}
(m+N+1)^{N/2}
\exp\left(
-(m+N+1)\lambda_\tau\Delta_{t,m}^2
\right).
\]
Since \(m\le t\) and \(m\Delta_{t,m}^2=C^2t^{1/2}\),
\[
\mathbb P(E_t\mid \mathcal F_{t-1})
\le
C'_N
t^{N/2}
\exp\left(
-\lambda_\tau C^2t^{1/2}
\right),
\]
for a constant \(C'_N\) depending only on \(N\). Hence
\[
\sum_{m=1}^t
\mathbb P(E_t,\ m_t=m,\ G_{t,m})
\le
C'_N
t^{N/2}
\exp\left(
-\lambda_\tau C^2t^{1/2}
\right).
\]

Combining the two bounds,
\[
\mathbb P(E_t)
\le
B_N t
\exp(-c_1 t^{1/2})
+
C'_N t^{N/2}
\exp(-c_2 t^{1/2})
\]
for positive constants \(c_1,c_2\). Since
\[
\sum_{t=1}^{\infty} t^r e^{-c\sqrt t}<\infty
\qquad
\text{for every fixed } r<\infty,\ c>0,
\]
we get
\[
\sup_{n\ge1}
\mathcal L_{n,\mathrm{\algcatso}}^\tau
\le
\sum_{t=1}^{\infty}\mathbb P(E_t)
<\infty.
\]
This proves the claim.
\end{proof}

\begin{lemma}[\algcatso CVaR suboptimal-arm pulls]
\label{lm_suboptimal_arm_pull_cats_appendix}
Adopt the notation of Section~\ref{sec:cvar_bandit}, with each arm having fixed grid law $F_a = \sum_i p_a^i\delta_{z_i}$, $p_a\in\Delta_N$, and unique CVaR-optimal arm $a_\tau^\star$. Suppose \algcatso selects
\[
    I(t) \in \arg\max_{a\in[K]}\left\{\widehat c_a^\tau(t) + C\tfrac{t^{1/4}}{(T_a(t)\vee 1)^{1/2}}\right\},
\]
where $T_a(t)$ is the number of pulls of arm $a$ before round $t$. Assume Lemma~\ref{lem:catso_posterior_center_concentration}, that $p\mapsto\mathcal K_{\inf}^{(N),\tau,+}(p,u)$ is $\|\cdot\|_\infty$-continuous on $\Delta_N$ for every fixed $u<c_\star^\tau$, and that $c\mapsto\mathcal K_{\inf}^{(N),\tau,+}(F_a,c)$ is left-continuous at $c = c_\star^\tau$. Then for every suboptimal $a\neq a_\tau^\star$ and every $\epsilon_0>0$,
\[
    \bE[T_a(n)] \le \frac{(1+\epsilon_0)\log n}{\mathcal K_{\inf}^{(N),\tau,+}(F_a,c_\star^\tau)} + \mathcal O(n^{1/2}).
\]
\end{lemma}

\begin{proof}
Fix a CVaR-suboptimal arm \(a\neq a_\tau^\star\). Write
\[
K_a(c)
:=
\mathcal K_{\inf}^{(N),\tau,+}(F_a,c),
\qquad
K_a^\star
:=
K_a(c_\star^\tau).
\]
Since \(a\) is CVaR-suboptimal and the CVaR-optimal arm is unique,
\[
c_a^\tau<c_\star^\tau.
\]
Thus, for \(\epsilon_1>0\) small enough,
\[
u:=c_\star^\tau-\epsilon_1
\]
satisfies
\[
c_a^\tau<u<c_\star^\tau.
\]
For such \(\epsilon_1\), the fixed-grid CVaR superlevel set
\[
\left\{
q\in\Delta_N:
\mathrm{CVaR}_\tau(q;Z_N)\ge u
\right\}
\]
is closed and does not contain \(p_a\). Hence
\[
K_a(u)>0.
\]
Choose \(\epsilon_3>0\) such that
\[
\kappa
:=
\kappa_a^\tau(\epsilon_1,\epsilon_3)
:=
K_a(c_\star^\tau-\epsilon_1)-\epsilon_3
>0.
\]

By the assumed continuity of
\[
p\longmapsto
\mathcal K_{\inf}^{(N),\tau,+}(p,c_\star^\tau-\epsilon_1),
\]
there exists \(\epsilon_2>0\) such that, whenever
\[
\|p-p_a\|_\infty\le \epsilon_2,
\]
we have
\[
\mathcal K_{\inf}^{(N),\tau,+}(p,c_\star^\tau-\epsilon_1)
\ge
K_a(c_\star^\tau-\epsilon_1)-\epsilon_3
=
\kappa.
\]
In particular, on this event the information number is strictly positive.
Therefore
\[
c_\star^\tau-\epsilon_1
>
\mathrm{CVaR}_\tau(p;Z_N),
\]
because otherwise \(p\) itself would be feasible in the definition of
\(\mathcal K_{\inf}^{(N),\tau,+}(p,c_\star^\tau-\epsilon_1)\), forcing the
information number to be zero. Hence the corrected upper-tail
Proposition~\ref{prop:catso-cvar-upper-tail-fixed-grid} is applicable on this
good-center event.

For clarity, write
\[
P_{\alpha_a}^{(m)}
\]
for the posterior center of arm \(a\) after exactly \(m\) local updates, so
that, on the event \(\{T_a(t)=m\}\),
\[
P_{\alpha_a(t)}=P_{\alpha_a}^{(m)},
\qquad
A_{a,t}=m+N+1.
\]
The extra \(N+1\) comes from the unit Dirichlet prior on the \(N+1\) grid
atoms.

Define the bonus
\[
\beta_t(a)
:=
C\frac{t^{1/4}}{(T_a(t)\vee 1)^{1/2}}.
\]
Set
\[
m_\beta
:=
\left\lceil
\left(\frac{C}{\epsilon_1}\right)^{1/2}
n^{1/2}
\right\rceil,
\]
\[
m_{\mathrm{KL}}
:=
\left\lceil
\frac{
\log n+(N/2+1)\log\log(en)
}{
\kappa
}
\right\rceil,
\]
and
\[
m_0:=\max\{1,m_\beta,m_{\mathrm{KL}}\}.
\]

We first decompose the number of pulls of arm \(a\). The event
\(\{I(t)=a,\ T_a(t)<m_0\}\) can occur at most \(m_0\) times. Therefore
\[
\mathbb E[T_a(n)]
=
\mathbb E\left[\sum_{t=1}^n \mathbf 1\{I(t)=a\}\right]
\le
m_0+A_1+A_2,
\]
where
\[
A_1
:=
\mathbb E\left[
\sum_{t=1}^n
\mathbf 1\left\{
I(t)=a,\ 
T_a(t)\ge m_0,\ 
\widehat c_a^\tau(t)\ge c_\star^\tau-\beta_t(a),\
\|P_{\alpha_a(t)}-p_a\|_\infty\le \epsilon_2
\right\}
\right],
\]
and
\[
A_2
:=
\mathbb E\left[
\sum_{t=1}^n
\mathbf 1\left\{
I(t)=a,\ 
T_a(t)\ge m_0,\ 
\left[
\widehat c_a^\tau(t)<c_\star^\tau-\beta_t(a)
\right]
\cup
\left[
\|P_{\alpha_a(t)}-p_a\|_\infty>\epsilon_2
\right]
\right\}
\right].
\]

\paragraph{Bound on \(A_1\).}
Let \(\mathcal F_{t-1}\) denote the history immediately before the
Thompson draw at round \(t\). On the event \(T_a(t)\ge m_0\), we have
\[
\beta_t(a)
=
C\frac{t^{1/4}}{(T_a(t)\vee 1)^{1/2}}
\le
C\frac{n^{1/4}}{m_0^{1/2}}
\le
C\frac{n^{1/4}}{m_\beta^{1/2}}
\le
\epsilon_1.
\]
Hence
\[
\widehat c_a^\tau(t)\ge c_\star^\tau-\beta_t(a)
\quad\Longrightarrow\quad
\widehat c_a^\tau(t)\ge c_\star^\tau-\epsilon_1.
\]
Dropping the event \(\{I(t)=a\}\), which can only increase the probability,
we obtain
\[
A_1
\le
\sum_{t=1}^n
\mathbb E\left[
\mathbf 1\left\{
T_a(t)\ge m_0,\
\|P_{\alpha_a(t)}-p_a\|_\infty\le \epsilon_2
\right\}
\mathbb P\left(
\widehat c_a^\tau(t)\ge c_\star^\tau-\epsilon_1
\middle|\,
\mathcal F_{t-1}
\right)
\right].
\]

Now fix a history with
\[
T_a(t)=m\ge m_0
\quad\text{and}\quad
\|P_{\alpha_a(t)}-p_a\|_\infty\le \epsilon_2.
\]
By the choice of \(\epsilon_2\),
\[
\mathcal K_{\inf}^{(N),\tau,+}
(P_{\alpha_a(t)},c_\star^\tau-\epsilon_1)
\ge \kappa>0.
\]
Thus \(c_\star^\tau-\epsilon_1\) is a genuine upper-tail threshold for the
posterior center, and the corrected Proposition~\ref{prop:catso-cvar-upper-tail-fixed-grid}
gives
\[
\mathbb P\left(
\widehat c_a^\tau(t)\ge c_\star^\tau-\epsilon_1
\middle|\,
\mathcal F_{t-1}
\right)
\le
(N+1)C_N^{\mathrm{RH}}
(m+N+1)^{N/2}
\exp\left\{-(m+N+1)\kappa\right\}.
\]
If one uses the real-parameter Dirichlet tail lemma instead of the
integer-count \citet{riou2020bandit} lemma, replace \(C_N^{\mathrm{RH}}\) by the
corresponding dimension-only constant.

Define
\[
f_\kappa(m)
:=
(m+N+1)^{N/2}
\exp\left\{-(m+N+1)\kappa\right\}.
\]
The function \(f_\kappa(m)\) is decreasing for all sufficiently large \(m\).
Since \(m_0\to\infty\) with \(n\), for all sufficiently large \(n\),
\[
\sup_{m\ge m_0} f_\kappa(m)
=
f_\kappa(m_0).
\]
The finitely many smaller \(n\) can be absorbed into the final
\(O(n^{1/2})\) term. Therefore
\[
A_1
\le
(N+1)C_N^{\mathrm{RH}}\,
n\,
(m_0+N+1)^{N/2}
\exp\left\{-(m_0+N+1)\kappa\right\}.
\]
Since \(m_0\ge m_{\mathrm{KL}}\) and \(f_\kappa\) is eventually decreasing,
\[
A_1
\le
(N+1)C_N^{\mathrm{RH}}\,
n\,
(m_{\mathrm{KL}}+N+1)^{N/2}
\exp\left\{-m_{\mathrm{KL}}\kappa\right\}
\cdot O(1).
\]
By definition of \(m_{\mathrm{KL}}\),
\[
\exp\{-m_{\mathrm{KL}}\kappa\}
\le
\frac{1}{n[\log(en)]^{N/2+1}}.
\]
Also
\[
m_{\mathrm{KL}}+N+1
=
O(\log n).
\]
Hence
\[
A_1
\le
O(1)\,
\frac{(\log n)^{N/2}}{[\log(en)]^{N/2+1}}
=
O(1).
\]

\paragraph{Bound on \(A_2\).}
Write
\[
A_2\le R_1+R_2,
\]
where
\[
R_1
:=
\mathbb E\left[
\sum_{t=1}^n
\mathbf 1\left\{
I(t)=a,\
\widehat c_a^\tau(t)<c_\star^\tau-\beta_t(a)
\right\}
\right],
\]
and
\[
R_2
:=
\mathbb E\left[
\sum_{t=1}^n
\mathbf 1\left\{
I(t)=a,\
\|P_{\alpha_a(t)}-p_a\|_\infty>\epsilon_2
\right\}
\right].
\]

For \(R_2\), the event
\[
\{I(t)=a,\ T_a(t)=m\}
\]
can occur at most once for each \(m\in\{0,\ldots,n-1\}\), namely at the
round of the \((m+1)\)-st pull of arm \(a\). Hence
\[
R_2
\le
\sum_{m=0}^{n-1}
\mathbb P\left(
\|P_{\alpha_a}^{(m)}-p_a\|_\infty>\epsilon_2
\right)
=
O(1),
\]
by Lemma~\ref{lem:catso_posterior_center_concentration}. The \(m=0\) term is just
a finite constant.

For \(R_1\), suppose that
\[
I(t)=a
\quad\text{and}\quad
\widehat c_a^\tau(t)<c_\star^\tau-\beta_t(a).
\]
Then
\[
U_t(a)
=
\widehat c_a^\tau(t)+\beta_t(a)
<
c_\star^\tau.
\]
Since \(I(t)=a\) maximizes the index,
\[
U_t(a_\tau^\star)\le U_t(a)<c_\star^\tau.
\]
Therefore
\[
\widehat c_{a_\tau^\star}^\tau(t)
<
c_\star^\tau-\beta_t(a_\tau^\star).
\]
Consequently,
\[
R_1
\le
\mathbb E\left[
\sum_{t=1}^n
\mathbf 1\left\{
\widehat c_{a_\tau^\star}^\tau(t)
<
c_\star^\tau-\beta_t(a_\tau^\star)
\right\}
\right]
=
\mathcal L_{n,\mathrm{\algcatso}}^\tau \leq \mathcal O(n^{1/2}).
\]

Combining the bounds on \(A_1\), \(R_1\), and \(R_2\), we obtain
\[
\mathbb E[T_a(n)]
\le
m_0+\mathcal O(n^{1/2}) + O(1).
\]
Since
\[
m_0
\le
1+m_\beta+m_{\mathrm{KL}},
\]
we have
\[
m_0
\le
\frac{
\log n+(N/2+1)\log\log(en)
}{
\kappa
}
+
O(n^{1/2}).
\]
Hence
\[
\mathbb E[T_a(n)]
\le
\frac{
\log n+(N/2+1)\log\log(en)
}{
\kappa_a^\tau(\epsilon_1,\epsilon_3)
}
+
\mathcal O(n^{1/2})
+
O(n^{1/2}).
\]

It remains to choose \(\epsilon_1\) and \(\epsilon_3\). By the assumed
left-continuity of
\[
c\longmapsto
\mathcal K_{\inf}^{(N),\tau,+}(F_a,c)
\]
at \(c=c_\star^\tau\), for every \(\epsilon_0>0\) we can choose
\(\epsilon_1>0\) small enough, and then \(\epsilon_3>0\) small enough, so
that
\[
\frac{1}{
\kappa_a^\tau(\epsilon_1,\epsilon_3)
}
\le
\frac{1+\epsilon_0/2}{
\mathcal K_{\inf}^{(N),\tau,+}(F_a,c_\star^\tau)
}.
\]
Since
\[
\log\log(en)=o(\log n),
\]
for all sufficiently large \(n\),
\[
\frac{
\log n+(N/2+1)\log\log(en)
}{
\kappa_a^\tau(\epsilon_1,\epsilon_3)
}
\le
\frac{
(1+\epsilon_0)\log n
}{
\mathcal K_{\inf}^{(N),\tau,+}(F_a,c_\star^\tau)
}.
\]
Absorbing the finitely many small-\(n\) cases into the
\(O(n^{1/2})\) term gives
\[
\mathbb E[T_a(n)]
\le
\frac{
(1+\epsilon_0)\log n
}{
\mathcal K_{\inf}^{(N),\tau,+}(F_a,c_\star^\tau)
}
+
O(n^{1/2}),
\]
which proves the lemma.
\end{proof}

\subsection{\algpatso Posterior CVaR Tails}

\begin{assumption}[\algpatso posterior CVaR upper tail, history-conditional]
\label{assump:patso_cvar_upper_tail}
Fix \(\tau\in(0,1]\). For each arm \(a\), local sample size \(m\ge1\), and
threshold \(c\in(c_a^\tau,R]\), assume that for every \(\eta>0\) there exists
\(C_{\tau,\eta,c}>0\) such that, whenever the pre-round history has
\(T_a(t)=m\),
\[
\mathbb P\left(
\widehat c_a^\tau(t)\ge c
\middle|\,
\mathcal F_{t-1}
\right)
\le
C_{\tau,\eta,c}
\exp\left\{
-m\left(
\mathcal K_{\inf}^{\tau,+}(\widehat F_{a,m},c)-\eta
\right)
\right\}
\qquad\text{a.s.}
\]
\end{assumption}

\begin{proposition}[Uncapped \algpatso posterior CVaR upper tail]
\label{prop:patso_cvar_upper_tail_uncapped}
Fix $\tau\in(0,1]$ and consider uncapped \algpatso, so that after $m\ge 1$ pulls of arm $a$ the particle support equals the support of the empirical law of the first $m$ returns. Write $\widehat F_{a,m} = \sum_{j=1}^{M_m} p_j\delta_{s_j}$ with $s_1<\cdots<s_{M_m}$, $p_j = n_j/m$, $n_j\in\mathbb N$, $\sum_j n_j = m$, draw $L_a(m)\sim\mathrm{Dir}(n_1,\ldots,n_{M_m})$, and define
\[
    \widehat D^{\mathrm{TS}}_{a,m} := \sum_j L_{a,j}(m)\delta_{s_j}, \qquad \widehat c^\tau_{a,m} := \mathrm{CVaR}_\tau(\widehat D^{\mathrm{TS}}_{a,m}).
\]
For a probability law $F$ on $[0,R]$, set $\mathcal K_{\inf}^{\tau,+}(F,c) := \inf\{\mathrm{KL}(F\|G) : G\in\mathcal P([0,R]),\, \mathrm{CVaR}_\tau(G)\ge c\}$. Then for every $c\in\mathbb R$ and $\eta>0$ there exists $C_\eta<\infty$, independent of $m$, such that
\[
    \PP\left(\widehat c^\tau_{a,m}\ge c \middle|\, \widehat F_{a,m}\right) \le C_\eta\exp\left\{-m\bigl(\mathcal K_{\inf}^{\tau,+}(\widehat F_{a,m},c)-\eta\bigr)\right\} \quad\text{a.s.;}
\]
one may take $C_\eta = \exp\{\tfrac{1}{12} + \tfrac{1}{2\pi\eta}\}$. In particular, Assumption~\ref{assump:patso_cvar_upper_tail} holds for uncapped exact-particle \algpatso.
\end{proposition}

\begin{proof}
Fix \(m\ge 1\) and condition on the empirical law
\[
\widehat F_{a,m}
=
\sum_{j=1}^{M_m}p_j\delta_{s_j},
\qquad
s_1<\cdots<s_{M_m},
\qquad
p_j=\frac{n_j}{m}.
\]
Under uncapped \algpatso, the particle support and the Dirichlet parameters are
therefore fixed by \(\widehat F_{a,m}\):
\[
S_a(m)=\{s_1,\ldots,s_{M_m}\},
\qquad
\alpha_a(m)=(n_1,\ldots,n_{M_m}).
\]
For readability, write \(M:=M_m\), \(p=(p_1,\ldots,p_M)\), and
\[
\Delta_{M-1}
:=
\left\{
q=(q_1,\ldots,q_M)\in[0,1]^M:
\sum_{j=1}^M q_j=1
\right\}.
\]

For \(q\in\Delta_{M-1}\), define cumulative sums
\[
S_0(q):=0,
\qquad
S_k(q):=\sum_{j=1}^k q_j,
\qquad
1\le k\le M.
\]
For each \(k\in\{1,\ldots,M\}\), define
\[
C_k^+(c)
:=
\left\{
q\in\Delta_{M-1}:
S_{k-1}(q)\le \tau\le S_k(q),
\quad
\tau s_k+\sum_{j=1}^{k-1}q_j(s_j-s_k)\ge \tau c
\right\}.
\]
As in the fixed-grid \algcatso proof, on the region
\[
S_{k-1}(q)\le \tau\le S_k(q),
\]
the lower-tail CVaR of the discrete law \(\sum_{j=1}^M q_j\delta_{s_j}\) is
given by
\[
\begin{aligned}
\tau\,
\mathrm{CVaR}_\tau\left(\sum_{j=1}^M q_j\delta_{s_j}\right)
&=
\sum_{j=1}^{k-1}q_j s_j
+
\bigl(\tau-S_{k-1}(q)\bigr)s_k =
\tau s_k+\sum_{j=1}^{k-1}q_j(s_j-s_k).
\end{aligned}
\]
Therefore,
\[
\left\{
q\in\Delta_{M-1}:
\mathrm{CVaR}_\tau\left(\sum_{j=1}^M q_j\delta_{s_j}\right)
\ge c
\right\}
=
\bigcup_{k=1}^M C_k^+(c).
\]
Each \(C_k^+(c)\) is a closed convex polytope, since it is the intersection
of the simplex with finitely many affine half-space constraints.

Define the support-restricted information number
\[
K_{S}^{\tau,+}(p,c)
:=
\inf\left\{
\mathrm{KL}(p\|q):
q\in\Delta_{M-1},
\ 
\mathrm{CVaR}_\tau\left(\sum_{j=1}^M q_j\delta_{s_j}\right)
\ge c
\right\}.
\]
Here
\[
\mathrm{KL}(p\|q)
:=
\sum_{j=1}^M p_j\log\frac{p_j}{q_j}.
\]

If
\[
c\le
\mathrm{CVaR}_\tau(\widehat F_{a,m}),
\]
then \(p\) itself is feasible in the definition of
\(\mathcal K_{\inf}^{\tau,+}(\widehat F_{a,m},c)\), so
\[
\mathcal K_{\inf}^{\tau,+}(\widehat F_{a,m},c)=0.
\]
Thus
\[
\mathbb P\left(
\widehat c^\tau_{a,m}\ge c
\middle|
\widehat F_{a,m}
\right)
\le 1
\le
C_\eta e^{m\eta}
=
C_\eta
\exp\left\{
-m\left(
\mathcal K_{\inf}^{\tau,+}(\widehat F_{a,m},c)-\eta
\right)
\right\},
\]
provided \(C_\eta\ge 1\). Hence the desired inequality is trivial in this
non-tail case.

It remains to consider the genuine upper-tail case
\[
c>
\mathrm{CVaR}_\tau(\widehat F_{a,m}).
\]
Then \(p\notin C_k^+(c)\) for every \(k\). Applying the \citet{riou2020bandit}
Dirichlet convex-set tail bound to each closed convex set \(C_k^+(c)\), with
total Dirichlet mass \(m=\sum_{j=1}^M n_j\), gives
\[
\mathbb P\left(
L_a(m)\in C_k^+(c)
\middle|
\widehat F_{a,m}
\right)
\le
C_{M-1}^{\mathrm{RH}} m^{(M-1)/2}
\exp\left\{
-m\inf_{q\in C_k^+(c)}\mathrm{KL}(p\|q)
\right\},
\]
where
\[
C_{M-1}^{\mathrm{RH}}
=
\frac{e^{1/12}}{\Gamma(M)}
\left(\frac{1}{\sqrt{2\pi}}\right)^{M-1}.
\]
Using the union representation of the CVaR superlevel event and a union
bound,
\[
\begin{aligned}
\mathbb P\left(
\widehat c^\tau_{a,m}\ge c
\middle|
\widehat F_{a,m}
\right)
\le
M C_{M-1}^{\mathrm{RH}} m^{(M-1)/2}\times
\exp\left\{
-m K_S^{\tau,+}(p,c)
\right\}.
\end{aligned}
\]

Now compare the support-restricted and unrestricted information numbers.
Every probability law of the form
\[
\sum_{j=1}^M q_j\delta_{s_j},
\qquad q\in\Delta_{M-1},
\]
is also an element of \(\mathcal P([0,R])\). Therefore the unrestricted
feasible set is larger, and hence
\[
\mathcal K_{\inf}^{\tau,+}(\widehat F_{a,m},c)
\le
K_S^{\tau,+}(p,c).
\]
Consequently,
\[
\exp\{-mK_S^{\tau,+}(p,c)\}
\le
\exp\left\{
-m\mathcal K_{\inf}^{\tau,+}(\widehat F_{a,m},c)
\right\}.
\]

It remains only to control the dimension-dependent prefactor uniformly in
\(m\). Set
\[
d:=M-1,
\qquad
x:=\sqrt{\frac{m}{2\pi}}.
\]
Since \(M=d+1\),
\[
\begin{aligned}
M C_{M-1}^{\mathrm{RH}} m^{(M-1)/2}
&=
(d+1)\frac{e^{1/12}}{\Gamma(d+1)}
\left(\frac{1}{\sqrt{2\pi}}\right)^d
m^{d/2}
\\
&=
e^{1/12}\frac{d+1}{d!}x^d
\\
&\le
e^{1/12}
\sum_{r=0}^{\infty}\frac{r+1}{r!}x^r
\\
&=
e^{1/12}(1+x)e^x
\\
&\le
e^{1/12}e^{2x}.
\end{aligned}
\]
For every \(\eta>0\),
\[
2x
=
2\sqrt{\frac{m}{2\pi}}
\le
\eta m+\frac{1}{2\pi\eta}.
\]
Hence
\[
M C_{M-1}^{\mathrm{RH}} m^{(M-1)/2}
\le
\exp\left\{
\frac{1}{12}+\frac{1}{2\pi\eta}
\right\}
e^{\eta m}.
\]
Define
\[
C_\eta
:=
\exp\left\{
\frac{1}{12}+\frac{1}{2\pi\eta}
\right\}.
\]
Combining the previous displays gives
\[
\mathbb P\left(
\widehat c^\tau_{a,m}\ge c
\middle|\,
\widehat F_{a,m}
\right)
\le
C_\eta
\exp\left\{
-m\left(
\mathcal K_{\inf}^{\tau,+}(\widehat F_{a,m},c)-\eta
\right)
\right\}.
\]
This proves the claim.
\end{proof}

\begin{lemma}[\algpatso empirical-law concentration]
\label{lem:patso_empirical_law_concentration}
Fix an arm \(a\). Suppose the local observations from arm \(a\),
$
X_{a,1},X_{a,2},\ldots,
$
are i.i.d. with law \(F_a\). Define the empirical law
\[
\widehat F_{a,m}
:=
\frac1m\sum_{r=1}^m \delta_{X_{a,r}}.
\]
Let \(D_L\) denote the Lévy distance. Then, for every
\(\epsilon>0\),
\[
\sum_{m=1}^{\infty}
\mathbb P\left(
D_L(\widehat F_{a,m},F_a)>\epsilon
\right)
<\infty.
\]
\end{lemma}

\begin{proof}
Let
\[
\|\widehat F_{a,m}-F_a\|_\infty
:=
\sup_{x\in\mathbb R}
\left|
\widehat F_{a,m}(x)-F_a(x)
\right|
\]
be the Kolmogorov distance between the two distribution functions.

The Lévy distance is bounded by the Kolmogorov distance:
\[
D_L(\widehat F_{a,m},F_a)
\le
\|\widehat F_{a,m}-F_a\|_\infty.
\]
Indeed, if
\[
\|\widehat F_{a,m}-F_a\|_\infty\le \epsilon,
\]
then for every \(x\),
\[
F_a(x-\epsilon)-\epsilon
\le
F_a(x)-\epsilon
\le
\widehat F_{a,m}(x)
\le
F_a(x)+\epsilon
\le
F_a(x+\epsilon)+\epsilon,
\]
which is exactly the defining Lévy-distance condition with radius
\(\epsilon\).

By the Dvoretzky--Kiefer--Wolfowitz inequality,
\[
\mathbb P\left(
\|\widehat F_{a,m}-F_a\|_\infty>\epsilon
\right)
\le
2\exp(-2m\epsilon^2).
\]
Therefore
\[
\mathbb P\left(
D_L(\widehat F_{a,m},F_a)>\epsilon
\right)
\le
2\exp(-2m\epsilon^2).
\]
Summing over \(m\),
\[
\sum_{m=1}^{\infty}
\mathbb P\left(
D_L(\widehat F_{a,m},F_a)>\epsilon
\right)
\le
2\sum_{m=1}^{\infty}e^{-2m\epsilon^2}
<\infty.
\]
\end{proof}

\begin{lemma}[Optimal-arm lower-tail term for uncapped \algpatso]
\label{lem:patso_optimal_lower_tail_term}
Assume the stationary local-bandit model for the optimal arm $a_\tau^\star$, with rewards supported on $[0,R]$, and write $a^\tau_\star := a_\tau^\star$, $F_\star := F_{a^\tau_\star}$, $c_\star^\tau := \mathrm{CVaR}_\tau(F_\star)$. For uncapped \algpatso, after $m\ge 1$ pulls of $a^\tau_\star$ the empirical and Thompson-sampled particle laws are
\[
    \widehat F_{\star,m} = \tfrac{1}{m}\sum_{r=1}^m \delta_{X_{\star,r}} = \sum_{j=1}^{M_m} \tfrac{n_j}{m}\delta_{s_j}, \qquad \widetilde F_{\star,m} = \sum_{j=1}^{M_m} L_j\delta_{s_j}, \quad L\sim\mathrm{Dir}(n_1,\ldots,n_{M_m}).
\]
Assume each arm is initialized once, and let $\widetilde c_{a^\tau_\star}^\tau(t) := \mathrm{CVaR}_\tau(\widetilde F_{\star,N_{a^\tau_\star}(t)})$. Then
\[
    \mathcal L_{n,\mathrm{\algpatso}}^\tau := \EE\left[\sum_{t=1}^n \mathbf 1\left\{\widetilde c_{a^\tau_\star}^\tau(t) < c_\star^\tau - C\tfrac{t^{1/4}}{(N_{a^\tau_\star}(t)\vee 1)^{1/2}}\right\}\right] = O(1),
\]
and in particular $\mathcal L_{n,\mathrm{\algpatso}}^\tau = O(n^{1/2})$.
\end{lemma}

\begin{proof}
Let
\[
m_t:=N_{a^\tau_\star}(t)\vee1,
\qquad
\Delta_t
:=
C\frac{t^{1/4}}{m_t^{1/2}}.
\]
Define the event
\[
E_t
:=
\left\{
\widetilde c_{a^\tau_\star}^\tau(t)
<
c_\star^\tau-\Delta_t
\right\}.
\]
It is enough to show that
\[
\sum_{t=1}^{\infty}\mathbb P(E_t)<\infty.
\]

We use two standard concentration facts.

First, by the Dvoretzky--Kiefer--Wolfowitz inequality, for every
\(m\ge1\) and \(u>0\),
\[
\mathbb P\left(
\|\widehat F_{\star,m}-F_\star\|_\infty>u
\right)
\le
2e^{-2mu^2}.
\]
Second, conditionally on \(\widehat F_{\star,m}\), the Dirichlet-weighted
CDF
\[
\widetilde F_{\star,m}(x)
\]
satisfies the Bayesian-bootstrap DKW bound
\[
\mathbb P\left(
\|\widetilde F_{\star,m}-\widehat F_{\star,m}\|_\infty>u
\middle|
\widehat F_{\star,m}
\right)
\le
2m e^{-c_B m u^2}
\]
for a universal constant \(c_B>0\). Indeed, at each empirical support point
\(s_k\),
\[
\widetilde F_{\star,m}(s_k)
=
\sum_{j\le k}L_j
\]
is beta-distributed with mean
\(\widehat F_{\star,m}(s_k)\), by aggregation of Dirichlet coordinates.
Applying a beta Chernoff bound and union-bounding over at most \(m\)
support points gives the display.

Because the laws are supported on \([0,R]\),
\[
W_1(P,Q)\le R\|F_P-F_Q\|_\infty.
\]
Also, CVaR is \(1/\tau\)-Lipschitz in \(W_1\). Hence
\[
\left|
\mathrm{CVaR}_\tau(P)-\mathrm{CVaR}_\tau(Q)
\right|
\le
\frac{R}{\tau}\|F_P-F_Q\|_\infty.
\]

Now fix \(t\) and condition on the event \(m_t=m\). Set
\[
\Delta_{t,m}:=
C\frac{t^{1/4}}{m^{1/2}}.
\]
Then
\[
m\Delta_{t,m}^2
=
C^2t^{1/2}.
\]
If
\[
\mathrm{CVaR}_\tau(\widetilde F_{\star,m})
<
c_\star^\tau-\Delta_{t,m},
\]
then at least one of the following two events occurs:
\[
G_{t,m}
:=
\left\{
\mathrm{CVaR}_\tau(\widehat F_{\star,m})
<
c_\star^\tau-\frac{\Delta_{t,m}}2
\right\},
\]
or
\[
H_{t,m}
:=
\left\{
\mathrm{CVaR}_\tau(\widetilde F_{\star,m})
<
\mathrm{CVaR}_\tau(\widehat F_{\star,m})
-\frac{\Delta_{t,m}}2
\right\}.
\]

For \(G_{t,m}\), the CVaR--Kolmogorov Lipschitz bound and DKW give
\[
\begin{aligned}
\mathbb P(G_{t,m})
&\le
\mathbb P\left(
\|\widehat F_{\star,m}-F_\star\|_\infty
>
\frac{\tau\Delta_{t,m}}{2R}
\right)
\\
&\le
2\exp\left(
-2m\frac{\tau^2\Delta_{t,m}^2}{4R^2}
\right)
\\
&=
2\exp\left(
-\frac{\tau^2C^2}{2R^2}t^{1/2}
\right).
\end{aligned}
\]

For \(H_{t,m}\), conditionally on \(\widehat F_{\star,m}\),
\[
\begin{aligned}
\mathbb P(H_{t,m}\mid \widehat F_{\star,m})
&\le
\mathbb P\left(
\|\widetilde F_{\star,m}-\widehat F_{\star,m}\|_\infty
>
\frac{\tau\Delta_{t,m}}{2R}
\middle|
\widehat F_{\star,m}
\right)
\\
&\le
2m\exp\left(
-c_Bm\frac{\tau^2\Delta_{t,m}^2}{4R^2}
\right)
\\
&\leq
2t\exp\left(
-\frac{c_B\tau^2C^2}{4R^2}t^{1/2}
\right),
\end{aligned}
\]
because \(m\le t\).

Therefore, using the union bound over possible values \(m\in\{1,\ldots,t\}\),
\[
\mathbb P(E_t)
\leq
C_1t\exp(-C_2\sqrt t)
\]
for constants \(C_1,C_2>0\) depending only on
\((C,\tau,R)\). Since
\[
\sum_{t=1}^{\infty}t e^{-C_2\sqrt t}<\infty,
\]
we get
\[
\sup_{n\ge1}
\mathcal L_{n,\mathrm{\algpatso}}^\tau
\leq
\sum_{t=1}^{\infty}\mathbb P(E_t)
<\infty.
\]
This proves the claim.
\end{proof}

\subsection{\algpatso CVaR suboptimal-arm pulls}
\label{sec:patso_cvar_bandit}

\begin{lemma}[\algpatso CVaR suboptimal-arm pulls]
\label{lm_suboptimal_arm_pull_pats_appendix}
Adopt the notation of Section~\ref{sec:cvar_bandit} with $a_\tau^\star$ the unique CVaR-optimal arm, and write $T_a(t) := T_a(t-1)$, $\widehat c_a^\tau(t) := \widetilde c_{a,T_a(t)}^\tau$. Assume \algpatso selects
\[
    I(t) \in \arg\max_{a\in[K]}\left\{\widehat c_a^\tau(t) + C\tfrac{t^{1/4}}{(T_a(t)\vee 1)^{1/2}}\right\},
\]
and that $F\mapsto\mathcal K_{\inf}^{\tau,+}(F,u)$ is $D_L$-continuous for every fixed $u<c_\star^\tau$ and $c\mapsto\mathcal K_{\inf}^{\tau,+}(F_a,c)$ is left-continuous at $c = c_\star^\tau$. Then for every suboptimal $a\neq a_\tau^\star$ and every $\epsilon_0>0$,
\[
    \bE[T_a(n)] \le \frac{(1+\epsilon_0)\log n}{\mathcal K_{\inf}^{\tau,+}(F_a,c_\star^\tau)} + \mathcal O(n^{1/2}).
\]
\end{lemma}

\begin{proof}
Fix a CVaR-suboptimal arm \(a\neq a_\tau^\star\), and write
\[
\beta_t(a)
:=
C\frac{t^{1/4}}{(T_a(t)\vee 1)^{1/2}}.
\]
Assume each arm is initialized once, or equivalently absorb the finitely many
rounds with \(T_a(t)=0\) into the final \(O(1)\) term. Hence below
\(\widehat F_{a,m}\) is only used for \(m\ge1\).

Since \(a\) is CVaR-suboptimal and the CVaR-optimal arm is unique,
\[
c_a^\tau<c_\star^\tau.
\]
Choose
\[
0<\epsilon_1<c_\star^\tau-c_a^\tau
\]
so that
\[
u:=c_\star^\tau-\epsilon_1
\]
satisfies
\[
c_a^\tau<u<c_\star^\tau.
\]
Let \(\eta,\epsilon_3>0\), and define
\[
\kappa_a^\tau(\epsilon_1,\eta,\epsilon_3)
:=
\mathcal K_{\inf}^{\tau,+}(F_a,c_\star^\tau-\epsilon_1)
-\eta-\epsilon_3.
\]
For \(\epsilon_1,\eta,\epsilon_3\) sufficiently small,
\[
\kappa_a^\tau(\epsilon_1,\eta,\epsilon_3)>0.
\]
Indeed, by left-continuity of
\(c\mapsto \mathcal K_{\inf}^{\tau,+}(F_a,c)\) at \(c=c_\star^\tau\), and
because \(c_\star^\tau>c_a^\tau\), the information number at
\(c_\star^\tau-\epsilon_1\) is positive for \(\epsilon_1\) small enough.

Set
\[
C_*:=C_{\tau,\eta,c_\star^\tau-\epsilon_1},
\]
\[
m_\beta
:=
\left\lceil
\left(\frac{C}{\epsilon_1}\right)^{1/2}
n^{1/2}
\right\rceil,
\]
\[
m_{\mathrm{KL}}
:=
\left\lceil
\frac{\log(C_*n)}
{\kappa_a^\tau(\epsilon_1,\eta,\epsilon_3)}
\right\rceil,
\]
and
\[
m_0:=\max\{1,m_\beta,m_{\mathrm{KL}}\}.
\]
For small \(n\), one may replace \(\log(C_*n)\) by \(\log(eC_*n)\); this only
changes constants and avoids irrelevant negative-log edge cases.

We decompose
\[
\mathbb E[T_a(n)]
=
\mathbb E\left[\sum_{t=1}^n \mathbf 1\{I(t)=a\}\right]
\le
m_0+B_1+B_2,
\]
where
\[
B_1
:=
\mathbb E\left[
\sum_{t=1}^n
\mathbf 1\left\{
I(t)=a,\,
T_a(t)\ge m_0,\,
\widehat c_a^\tau(t)\ge c_\star^\tau-\beta_t(a),\,
D_L(\widehat F_{a,T_a(t)},F_a)\le\epsilon_2
\right\}
\right],
\]
and
\[
B_2
:=
\mathbb E\left[
\sum_{t=1}^n
\mathbf 1\left\{
I(t)=a,\,
T_a(t)\ge m_0,\,
\left[
\widehat c_a^\tau(t)<c_\star^\tau-\beta_t(a)
\right]
\cup
\left[
D_L(\widehat F_{a,T_a(t)},F_a)>\epsilon_2
\right]
\right\}
\right].
\]

\paragraph{Bound on \(B_1\).}
Dropping the event \(\{I(t)=a\}\) and conditioning on the local sample size,
\[
B_1
\le
\sum_{t=1}^n
\sum_{m=m_0}^{n-1}
\mathbb P\left(
\widehat c_a^\tau(t)\ge c_\star^\tau-\beta_t(a),\,
D_L(\widehat F_{a,m},F_a)\le\epsilon_2,\,
T_a(t)=m
\right).
\]
If \(m\ge m_0\ge m_\beta\) and \(t\le n\), then
\[
\beta_t(a)
=
C\frac{t^{1/4}}{(m\vee 1)^{1/2}}
\leq
C\frac{n^{1/4}}{m_0^{1/2}}
\le
\epsilon_1.
\]
Therefore
\[
\widehat c_a^\tau(t)\ge c_\star^\tau-\beta_t(a)
\quad\Longrightarrow\quad
\widehat c_a^\tau(t)\ge c_\star^\tau-\epsilon_1.
\]
Hence
\[
B_1
\le
\sum_{t=1}^n
\sum_{m=m_0}^{n-1}
\mathbb P\left(
\widehat c_a^\tau(t)\ge c_\star^\tau-\epsilon_1,\,
D_L(\widehat F_{a,m},F_a)\le\epsilon_2,\,
T_a(t)=m
\right).
\]

By the assumed continuity of
\[
F\mapsto \mathcal K_{\inf}^{\tau,+}(F,c_\star^\tau-\epsilon_1)
\]
at \(F_a\), for every \(\epsilon_3>0\) there exists \(\epsilon_2>0\) such that
\[
D_L(\widehat F_{a,m},F_a)\le\epsilon_2
\quad\Longrightarrow\quad
\mathcal K_{\inf}^{\tau,+}(\widehat F_{a,m},c_\star^\tau-\epsilon_1)
\ge
\mathcal K_{\inf}^{\tau,+}(F_a,c_\star^\tau-\epsilon_1)-\epsilon_3.
\]
Using the history-conditional form of
Assumption~\ref{assump:patso_cvar_upper_tail}, on the event
\[
\{T_a(t)=m,\ D_L(\widehat F_{a,m},F_a)\le\epsilon_2\},
\]
we have
\[
\mathbb P\left(
\widehat c_a^\tau(t)\ge c_\star^\tau-\epsilon_1
\middle|
\mathcal F_{t-1}
\right)
\le
C_*
\exp\left\{
-m\kappa_a^\tau(\epsilon_1,\eta,\epsilon_3)
\right\}.
\]
Thus
\[
\begin{aligned}
B_1
\le
C_*
\sum_{t=1}^n
\sum_{m=m_0}^{n-1}
\exp\left\{
-m\kappa_a^\tau(\epsilon_1,\eta,\epsilon_3)
\right\}\times
\mathbb P\left(
D_L(\widehat F_{a,m},F_a)\le\epsilon_2,\,
T_a(t)=m
\right).
\end{aligned}
\]
Since \(m\ge m_0\),
\[
\exp\{-m\kappa_a^\tau\}
\leq
\exp\{-m_0\kappa_a^\tau\}.
\]
Also, for each fixed \(t\),
\[
\sum_{m=m_0}^{n-1}
\mathbb P\left(
D_L(\widehat F_{a,m},F_a)\le\epsilon_2,\,
T_a(t)=m
\right)
\leq 1.
\]
Therefore
\[
B_1
\leq
C_* n
\exp\left\{
-m_0\kappa_a^\tau(\epsilon_1,\eta,\epsilon_3)
\right\}
\le 1
\]
by the definition of \(m_{\mathrm{KL}}\).

\paragraph{Bound on \(B_2\).}
Write
\[
B_2\le R_1+R_2,
\]
where
\[
R_1
:=
\mathbb E\left[
\sum_{t=1}^n
\mathbf 1\left\{
I(t)=a,\,
\widehat c_a^\tau(t)<c_\star^\tau-\beta_t(a)
\right\}
\right] = O (n^{1/2}),
\]
and
\[
R_2
:=
\mathbb E\left[
\sum_{t=1}^n
\mathbf 1\left\{
I(t)=a,\,
T_a(t)\ge m_0,\,
D_L(\widehat F_{a,T_a(t)},F_a)>\epsilon_2
\right\}
\right].
\]

For \(R_2\), the event
\[
\{I(t)=a,\ T_a(t)=m\}
\]
can occur at most once for each \(m\). Hence
\[
R_2
\le
\sum_{m=m_0}^{n-1}
\mathbb P\left(
D_L(\widehat F_{a,m},F_a)>\epsilon_2
\right)
=
O(1)
\]
by Lemma~\ref{lem:patso_empirical_law_concentration}. This is why it is cleaner to
sum from \(m=m_0\), rather than from \(m=0\).

Combining the bounds on \(B_1\), \(R_1\), and \(R_2\),
\[
\mathbb E[T_a(n)]
\le
m_0+ O(n^{1/2}) +O(1).
\]
By the definition of \(m_0\),
\[
m_0
\le
\frac{
\log(C_*n)
}{
\kappa_a^\tau(\epsilon_1,\eta,\epsilon_3)
}
+
O(n^{1/2}).
\]
Finally, by left-continuity of
\[
c\mapsto \mathcal K_{\inf}^{\tau,+}(F_a,c)
\]
at \(c=c_\star^\tau\), for every \(\epsilon_0>0\) we can choose
\(\epsilon_1,\eta,\epsilon_3>0\) sufficiently small so that
\[
\frac{1}{
\kappa_a^\tau(\epsilon_1,\eta,\epsilon_3)
}
\le
\frac{1+\epsilon_0/2}{
\mathcal K_{\inf}^{\tau,+}(F_a,c_\star^\tau)
}.
\]
Since \(C_*\) is fixed after \(\epsilon_1,\eta\) are fixed,
\[
\log(C_*n)=\log n+O(1).
\]
Therefore, for all sufficiently large \(n\),
\[
\frac{
\log(C_*n)
}{
\kappa_a^\tau(\epsilon_1,\eta,\epsilon_3)
}
\leq
\frac{
(1+\epsilon_0)\log n
}{
\mathcal K_{\inf}^{\tau,+}(F_a,c_\star^\tau)
}.
\]
Absorbing finitely many small-\(n\) cases into the \(O(n^{1/2})\) term
gives
\[
\mathbb E[T_a(n)]
\leq
\frac{
(1+\epsilon_0)\log n
}{
\mathcal K_{\inf}^{\tau,+}(F_a,c_\star^\tau)
}
+
O(n^{1/2}),
\]
as claimed.
\end{proof}

\begin{remark}[Uncapped \algpatso]
The statement above is for the uncapped particle representation. A cap-and-merge implementation
adds a separate approximation bias, which should be tracked explicitly rather than folded into the
risk-KL term.
\end{remark}

\begin{lemma}\label{lm_cats_upper_bound_power_mean_and_optimal}
Let \algcatso be applied to a non-stationary bandit problem whose payoff sequence satisfies Assumption~\ref{assump:patso_cvar_upper_tail}. 
Define the estimator 
\[
  \widehat{c}^\tau_n
  =
    \sum_{a=1}^K 
    \frac{T_a(n)}{n}
    \widehat{c}^\tau_{a,T_a(n)}, 
\]
and let 
\(\delta_{\star,n} = c^{\tau}_{\star} - c^{\tau}_{\star,n}.\)
For any $\epsilon_0>0,$ the following upper bound holds:
\[
  \bE[\left|\widehat{c}_n - c^{\tau}_{\star,n}\right|]
  \leq
  \bigl|\delta_{\star,n}\bigr|
  +
  \frac{R}{n}
  \sum_{\substack{a\neq a_\tau^\star}}^K
  \Bigl\{
    \frac{(1+\epsilon_0)\log n}{\mathcal K_{\inf}^{(N),\tau,+}(F_a,c_\star^\tau)}
    +
    \mathcal{O}(\sqrt{ n })
  \Bigr\}.
\]
\end{lemma}

\begin{proof}
We observe that 
\begin{flalign}
    \left|\widehat{c}^\tau_n -  c^{\tau}_{\star}\right| \leq \left|\widehat{c}^\tau_n -  c^{\tau}_{\star,n}\right| + \left|c^{\tau}_{\star} - c^{\tau}_{\star,n}\right| = \left|\widehat{c}^\tau_n -  c^{\tau}_{\star,n}\right| + \left| \delta_{\star,n}\right|
\end{flalign}
We have  
\begin{flalign}
    \widehat{c}^\tau_n - c^{\tau}_{\star,n} = \frac{\sum_{a\neq a^\tau_\star}^{K} T_a(n) ( \widehat{c}^\tau_{a,T_a(n)} - \mu^\tau_{*,n})}{n}
\end{flalign}

Therefore
\begin{flalign}
     \bE[\left|\widehat{c}^\tau_n - c^{\tau}_{\star,n}\right|] 
    \leq R \sum_{a\neq a_\tau^\star}^K \frac{\bE[T_a(n)]}{n}
\end{flalign}

According to Lemma~\ref{lm_suboptimal_arm_pull_cats_appendix}, we have 
\begin{flalign}
     \bE[\left|\widehat{c}^\tau_n - c^{\tau}_{\star,n}\right|]  &\leq R \sum_{a\neq a_\tau^\star}^K \frac{\bE[T_a(n)]}{n} \leq \frac{R}{n}  \sum_{a\neq a_\tau^\star}^K \left\{ \frac{(1+\epsilon_0)\log n}{\mathcal K_{\inf}^{(N),\tau,+}(F_a,c_\star^\tau)} +  \mathcal{O}(\sqrt{ n })\right\}, 
\end{flalign}
which concludes the proof.
\end{proof}

\begin{lemma}\label{lm_pats_upper_bound_power_mean_and_optimal}
Let \algpatso (Particle Thompson Sampling with Optimistic Bonus) be applied to a non-stationary bandit problem whose payoff process satisfies Assumption~\ref{assump:cvar_main}. 
Define the estimator
\[
  \widehat{c}^\tau_n 
  =
    \sum_{a=1}^{K}
    \frac{T_a(n)}{n} \widehat{c}^\tau_{a,T_a(n)},
\]
and let
\(\delta_{\star,n} = c^{\tau}_{\star} - c^{\tau}_{\star,n}.\)
For any \(p \ge 1\) and \(\epsilon_0>0,\) the following holds:
\[
    \bE\bigl[\Bigl|\widehat{c}^\tau_n
    -
    c^{\tau}_{\star}
    \Bigr| \bigr]
  \le
  \bigl|\delta_{\star,n}\bigr|
  +
  \frac{R}{n}
  \sum_{\substack{ a\neq a_\tau^\star}}^{K}
  \bigl\{
    \frac{(1+\epsilon_0)\log n}{\mathcal K_{\inf}^{\tau,+}(F_a,c_\star^\tau)}
    +
  \mathcal{O}(\sqrt{ n })
  \bigr\}.
\]
\end{lemma}

\begin{proof}
Similar to Lemma~\ref{lm_cats_upper_bound_power_mean_and_optimal}, we can derive
\begin{flalign}
     \bE[\left|\widehat{c}^\tau_n - c^{\tau}_{\star,n}\right|]  &\leq |\delta_{\star,n}| + R \sum_{a\neq a_\tau^\star}^K \frac{\bE[T_a(n)]}{n}.
\end{flalign}
and the bound is obtained as a direct corollary of Lemma~\ref{lm_suboptimal_arm_pull_pats_appendix}.
\end{proof}

\paragraph{Proof Of Theorem~\ref{thm:cvar_bandit_main} for \algcatso case}

\begin{proof}
We first prove that $\underset{n \rightarrow \infty}{\lim} \bE[\widehat{c}^\tau_n] = c^{\tau}_{\star}$.
According to the result of Lemma~\ref{lm_cats_upper_bound_power_mean_and_optimal}, we have
\begin{flalign}
    \bE [\left|\widehat{c}^\tau_n -  c^{\tau}_{\star}\right|] &\leq |\delta_{\star,n}| + R \sum_{a=1,a\neq a_\tau^\star}^K \frac{\bE[T_a(n)]}{n}\\ 
    &\leq |\delta_{\star,n}| + \frac{R}{n}  \sum_{a=1,a\neq a_\tau^\star}^K \left\{ \frac{(1+\epsilon_0)\log n}{\mathcal K_{\inf}^{(N),\tau,+}(F_a,c_\star^\tau)} + O(\sqrt{ n })\right\} 
\end{flalign}
with $\delta_{\star,n} = c^{\tau}_{\star} - c^{\tau}_{\star,n}$, and because $\underset{n \rightarrow \infty}{\lim} c^\tau_{*,n}=c^{\tau}_{\star} $, we can conclude that 
\[
\underset{n \rightarrow \infty}{\lim} \bE[\widehat{c}^\tau_n] = c^{\tau}_{\star}.
\]
Second, we prove that
\[
\forall n\geq 1, \forall 0< \varepsilon \leq n^{-1/2}(\log n)^{-1/2}, \exists c > 0 \text{ that } \bP\left(|\widehat{c}^\tau_n - c^{\tau}_{\star}| > \varepsilon\right) \leq c n^{-1}\varepsilon^{-2}.
\]
We observe that 
\begin{flalign}
    &\left|\widehat{c}^\tau_n -  c^{\tau}_{\star}\right| \leq \left|\widehat{c}^\tau_n -  c^{\tau}_{\star,n}\right| + \left|c^{\tau}_{\star} - c^{\tau}_{\star,n}\right| = \left|\widehat{c}^\tau_n -  c^{\tau}_{\star,n}\right| + \left| \delta_{\star,n}\right|\\
    \Longrightarrow &\bP (\left|\widehat{c}^\tau_n -  c^{\tau}_{\star}\right| \geq \epsilon) \leq \bP (\left|\widehat{c}^\tau_n -  c^{\tau}_{\star,n}\right| \geq \epsilon/2) + \bP(\left| \delta_{\star,n}\right| \geq \epsilon/2).
\end{flalign}
Because $ \underset{n \rightarrow n}{\lim} |\delta_{\star,n}| = 0$, therefore, $\exists N_0 > 0$ such that $\forall n \geq N_0$, we have $|\delta_{\star,n}| <\epsilon/2 $ that means
\[
\forall n > N_0, \bP(\left| \delta_{\star,n}\right| \geq \epsilon/2) = 0.
\]
Next, according to Lemma~\ref{lm_cats_upper_bound_power_mean_and_optimal}, 
\begin{flalign}
    \bE[\left|\widehat{c}^\tau_n -  c^{\tau}_{\star,n}\right|] \leq \frac{R}{n}  \sum_{a=1,a\neq a_\tau^\star}^K \left\{ \frac{(1+\epsilon_0)\log n}{\mathcal K_{\inf}^{(N),\tau,+}(F_a,c_\star^\tau)} + O(\sqrt{ n })\right\} = O(n^{-1/2}),
\end{flalign}
that leads to
\begin{flalign}
    \bP (\left|\widehat{c}^\tau_n -  c^{\tau}_{\star,n}\right| \geq \epsilon/2) \leq \frac{\bE[\left|\widehat{c}^\tau_n -  c^{\tau}_{\star,n}\right|]}{\epsilon/2} = \frac{ O(n^{-1/2})}{\epsilon/2}.
\end{flalign}
Therefore, $\exists c > 0, \forall 0 < \epsilon \leq n^{-1/2}$ such that
\begin{flalign}
    \bP (\left|\widehat{c}^\tau_n -  c^{\tau}_{\star,n}\right| \geq \epsilon/2) \leq c n^{-1/2}  (\epsilon) (\epsilon^{-2}) \leq c n^{-1/2} \bigg( n^{-1/2}  \bigg) \epsilon^{-2} = c n^{-1}\epsilon^{-2},
\end{flalign}
which means
\[
\forall n\geq N_0, \forall 0< \varepsilon < n^{-1/2}, \exists c > 0 \text{ that } \bP\left(|\widehat{c}^\tau_n - c^{\tau}_{\star}| > \varepsilon\right) \leq c n^{-1}\varepsilon^{-2}.
\]\\
With $0<\epsilon \leq n^{-1/2} \Rightarrow \epsilon^{-2} \leq n^{-1},$ so that $n^{-1}\epsilon^{-2} \leq n^{-2}$.\\ 
With $1 \leq n < N_0 \Rightarrow n^{-1}\varepsilon^{-2} \leq N_0^{-2}$. Therefore
\[
\forall C > N_0^{-2} \Rightarrow \bP\left(|\widehat{c}^\tau_n - c^{\tau}_{\star}| > \varepsilon\right) \leq 1 < Cn^{-1}\varepsilon^{-2},
\]
which means 
\[
\forall n\geq 1, \forall 0< \varepsilon \leq n^{-1/2}, \exists C > 0 \text{ that } \bP\left(|\widehat{c}^\tau_n - c^{\tau}_{\star}| > \varepsilon\right) \leq C n^{-1}\varepsilon^{-2}.
\]
Next, we prove the following claim

\begingroup
\setlength{\fboxsep}{6pt}%
\noindent\fbox{%
\parbox{0.98\linewidth}{%
\textbf{Claim.}
\[
  \bP\bigl(\, |\widehat c^\tau_n-c^{\tau}_{\star}| > \varepsilon \bigr)
  \leq C\,n^{-1}\varepsilon^{-2},
  \qquad \forall\varepsilon > n^{-1/2},
\]
for a universal constant $C>0$ depending only on problem parameters ($K,R$, etc.).}}%
\endgroup

By Lemma~\ref{lm_cats_upper_bound_power_mean_and_optimal} we have
\begin{equation*}
  \bigl|\widehat c^\tau_n-c^{\tau}_{\star}\bigr|
  \leq
  R\sum_{a\neq a^\tau_\star}\frac{T_a(n)}{n} := A_n.
\end{equation*}

Let $S_n=\sum_{a\neq a^\tau_\star}T_a(n)$. Lemma~\ref{lm_suboptimal_arm_pull_cats_appendix} implies that $\exists c_1 > 0$ such that $\bE[S_n]\le c_1(K-1)\sqrt{n}$.
By Azuma--Hoeffding,
\[
  \bP\bigl(S_n-\bE[S_n]\ge t\bigr)\le \exp\Bigl(-\frac{t^2}{2n}\Bigr).
\]
Choose $t=\frac{\varepsilon n}{2R}-c_1(K-1)\sqrt{n}$ (valid since $\varepsilon>n^{-1/2}$), then
\[
  =\bP\Bigl(S_n > \frac{\varepsilon n}{2R}\Bigr)
  \le \exp(-c_3 n\varepsilon^2)
  \le \tfrac{1}{2}\,n^{-1}\varepsilon^{-2}.
\]

Therefore,
\[
  \bP\bigl(|\widehat c^\tau_n-c^{\tau}_{\star}|>\varepsilon\bigr)
  \leq \bP(A_n>\varepsilon)
  \leq C\,n^{-1}\varepsilon^{-2},
\]
which completes the proof.
\end{proof}

\paragraph{Proof Of Theorem~\ref{thm:cvar_bandit_main} for \algpatso case}

\begin{proof}
The proof follows the same steps as Theorem~\ref{thm:cvar_bandit_main} for \algcatso case. 
We first prove that $\underset{n \rightarrow \infty}{\lim} \bE[\widehat{c}^\tau_n] = c^{\tau}_{\star}$.
According to the result of Lemma~\ref{lm_pats_upper_bound_power_mean_and_optimal}, we have
\begin{flalign}
    \bE [\left|\widehat{c}^\tau_n -  c^{\tau}_{\star}\right|] &\leq |\delta_{\star,n}| + R \sum_{a=1,a\neq a_\tau^\star}^K \frac{\bE[T_a(n)]}{n}\\
    &\leq |\delta_{\star,n}| + \frac{R}{n}  \sum_{a=1,a\neq a_\tau^\star}^K \left\{ \frac{(1+\epsilon_0)\log n}{\mathcal K_{\inf}^{\tau,+}(F_a,c_\star^\tau)} + o(\sqrt{ n }) + O(1)\right\} 
\end{flalign}
with $\delta_{\star,n} = c^{\tau}_{\star} - c^{\tau}_{\star,n}$, and because $\underset{n \rightarrow \infty}{\lim} c^\tau_{*,n}=c^{\tau}_{\star} $, we can conclude that 
\[
\underset{n \rightarrow \infty}{\lim} \bE[\widehat{c}^\tau_n] = c^{\tau}_{\star}.
\]
Second, we prove that
\[
\forall n\geq 1, \forall 0< \varepsilon < n^{-1/2}, \exists c > 0 \text{ that } \bP\left(|\widehat{c}^\tau_n - c^{\tau}_{\star}| > \varepsilon\right) \leq c n^{-1}\varepsilon^{-2}.
\]
We observe that 
\begin{flalign}
    &\left|\widehat{c}^\tau_n -  c^{\tau}_{\star}\right| \leq \left|\widehat{c}^\tau_n -  c^{\tau}_{\star,n}\right| + \left|c^{\tau}_{\star} - c^{\tau}_{\star,n}\right| = \left|\widehat{c}^\tau_n -  c^{\tau}_{\star,n}\right| + \left| \delta_{\star,n}\right|\\
    \Longrightarrow &\bP (\left|\widehat{c}^\tau_n -  c^{\tau}_{\star}\right| \geq \epsilon) \leq \bP (\left|\widehat{c}^\tau_n -  c^{\tau}_{\star,n}\right| \geq \epsilon/2) + \bP(\left| \delta_{\star,n}\right| \geq \epsilon/2).
\end{flalign}
Because $ \underset{n \rightarrow n}{\lim} |\delta_{\star,n}| = 0$, therefore, $\exists N_0 > 0$ such that $\forall n \geq N_0$, we have $|\delta_{\star,n}| <\epsilon/2 $ that means
\[
\forall n > N_0, \bP(\left| \delta_{\star,n}\right| \geq \epsilon/2) = 0.
\]
Next, according to Lemma~\ref{lm_pats_upper_bound_power_mean_and_optimal}, 
\begin{flalign}
    \bE[\left|\widehat{c}^\tau_n -  c^{\tau}_{\star,n}\right|] \leq \frac{R}{n}  \sum_{a=1,a\neq a_\tau^\star}^K \left\{ \frac{(1+\epsilon_0)\log n}{\mathcal K_{\inf}^{\tau,+}(F_a,c_\star^\tau)} + o(\sqrt{ n }) + O(1)\right\} = O(n^{-1/2}),
\end{flalign}
that leads to
\begin{flalign}
    \bP (\left|\widehat{c}^\tau_n -  c^{\tau}_{\star,n}\right| \geq \epsilon/2) \leq \frac{\bE[\left|\widehat{c}^\tau_n -  c^{\tau}_{\star,n}\right|]}{\epsilon/2} = \frac{ O(n^{-1/2})}{\epsilon/2}.
\end{flalign}
Therefore, $\exists c > 0, \forall 0 < \epsilon < n^{-1/2}$ such that
\begin{flalign}
    \bP (\left|\widehat{c}^\tau_n -  c^{\tau}_{\star,n}\right| \geq \epsilon/2) \leq c n^{-1/2} (\epsilon) (\epsilon^{-2}) \leq c n^{-1/2} \bigg( n^{-1/2} \bigg) \epsilon^{-2} = c n^{-1}\epsilon^{-2},
\end{flalign}
which means
\[
\forall n\geq N_0, \forall 0< \varepsilon < n^{-1/2}, \exists c > 0 \text{ that } \bP\left(|\widehat{c}^\tau_n - c^{\tau}_{\star}| > \varepsilon\right) \leq c n^{-1}\varepsilon^{-2}.
\]\\
With $0<\epsilon < n^{-1/2} \Rightarrow \epsilon^{-2} > n^{-1},$ so that $n^{-1}\epsilon^{-2} > n^{-2}$.\\ 
With $1 \leq n < N_0 \Rightarrow n^{-1}\varepsilon^{-2} > N_0^{-2}$. Therefore
\[
\forall C > N_0^{-2} \Rightarrow \bP\left(|\widehat{c}^\tau_n - c^{\tau}_{\star}| > \varepsilon\right) \leq 1 < Cn^{-1}\varepsilon^{-2},
\]
which means 
\[
\forall n\geq 1, \forall 0< \varepsilon < n^{-1/2}, \exists C > 0 \text{ that } \bP\left(|\widehat{c}^\tau_n - c^{\tau}_{\star}| > \varepsilon\right) \leq C n^{-1}\varepsilon^{-2}.
\]
Similarly to \algcatso, we prove the following claim

\begingroup
\setlength{\fboxsep}{6pt}%
\noindent\fbox{%
\parbox{0.98\linewidth}{%
\textbf{Claim.}
\[
  \bP\bigl(\, |\widehat c^\tau_n-c^{\tau}_{\star}| > \varepsilon \bigr)
  \leq C\,n^{-1}\varepsilon^{-2},
  \qquad \forall\varepsilon > n^{-1/2},
\]
for a universal constant $C>0$ depending only on problem parameters ($K,R$, etc.).}}%
\endgroup

\paragraph{Setup \& decomposition.}
By Lemma~\ref{lm_pats_upper_bound_power_mean_and_optimal} we have
\begin{equation*}
  \bigl|\widehat c^\tau_n-c^{\tau}_{\star}\bigr|
  \leq
  R\sum^K_{a\neq a^\tau_\star}\frac{T_a(n)}{n} := A_n.
\end{equation*}

Let $S_n=\sum^K_{a\neq a^\tau_\star}T_a(n)$. Lemma~\ref{lm_suboptimal_arm_pull_pats_appendix} implies $\exists c_1 >0$ such that $\bE[S_n]\le c_1(K-1)\sqrt{n}$.
By Azuma--Hoeffding,
\[
  \bP\bigl(S_n-\bE[S_n]\ge t\bigr)\le \exp\Bigl(-\frac{t^2}{2n}\Bigr).
\]
Choose $t=\frac{\varepsilon n}{2R}-c_1(K-1)\sqrt{n}$ (valid since $\varepsilon>n^{-1/2}$), then
\[
  =\bP\Bigl(S_n > \frac{\varepsilon n}{2R}\Bigr)
  \le \exp(-c_3 n\varepsilon^2)
  \le \tfrac{1}{2}\,n^{-1}\varepsilon^{-2}.
\]

Therefore,
\[
  \bP\bigl(|\widehat c^\tau_n-c^{\tau}_{\star}|>\varepsilon\bigr)
  \leq \bP(A_n>\varepsilon)
  \leq C\,n^{-1}\varepsilon^{-2},
\]
which completes the proof.
\end{proof}

\section{MCTS Convergence Proofs}
\label{app:mcts_convergence}

\subsection{One-Step CVaR Lifting for Q-Edges}
\label{app:one_step_convergence}
\paragraph{Proof of Lemma~\ref{lem:cvar_q_conc}}

\begin{proof}
Let
\[
  R_h:=\frac{R_{\max}}{1-\gamma}.
\]
All plug-in and oracle backed-up samples lie in $[0,R_h]$.

Define the oracle empirical law
\[
  \widehat{\mathcal D}_{n}^{\mathrm{or}}(s_h,a)
  :=
  \frac1n\sum_{i=1}^n\delta_{G_i}.
\]
The samples $G_i=X_i+\gamma V^\tau(S_i)$ are i.i.d. with law
$\mathcal D^\tau(s_h,a)$.

By the $W_1$-Lipschitz property of CVaR,
\[
  \left|
    \widehat Q_{n}^{\tau}(s_h,a)
    -
    Q^\tau(s_h,a)
  \right|
  \le
  \frac1\tau
  W_1
  \left(
    \widehat{\mathcal D}_{n}(s_h,a),
    \mathcal D^\tau(s_h,a)
  \right).
\]
Using the triangle inequality,
\[
\begin{aligned}
  W_1
  \left(
    \widehat{\mathcal D}_{n},
    \mathcal D^\tau
  \right)
  &\le
  W_1
  \left(
    \widehat{\mathcal D}_{n},
    \widehat{\mathcal D}_{n}^{\mathrm{or}}
  \right)
  +
  W_1
  \left(
    \widehat{\mathcal D}_{n}^{\mathrm{or}},
    \mathcal D^\tau
  \right).
\end{aligned}
\]

We first control the oracle empirical-measure term. Since all laws are
supported on $[0,R_h]$,
\[
  W_1(P,Q)
  \le
  R_h\|F_P-F_Q\|_\infty.
\]
Therefore, by the Dvoretzky--Kiefer--Wolfowitz inequality,
\[
\begin{aligned}
&\mathbb P
\left(
  W_1
  \left(
    \widehat{\mathcal D}_{n}^{\mathrm{or}},
    \mathcal D^\tau
  \right)
  >
  \frac{\tau\varepsilon}{2}
\right)
\\
&\qquad\le
\mathbb P
\left(
  \left\|
    F_{\widehat{\mathcal D}_{n}^{\mathrm{or}}}
    -
    F_{\mathcal D^\tau}
  \right\|_\infty
  >
  \frac{\tau\varepsilon}{2R_h}
\right)
\\
&\qquad\le
2\exp
\left(
  -\frac{n\tau^2\varepsilon^2}{2R_h^2}
\right).
\end{aligned}
\]
Since $2e^{-x}\le 2/(ex)$ for $x>0$, this term is bounded by
\[
  \frac{4R_h^2}{e\tau^2}
  \frac1{n\varepsilon^2}.
\]

Next control the plug-in term. Couple each plug-in sample
$\bar q_i^{(n)}$ with the corresponding oracle sample $G_i$. Then
\[
\begin{aligned}
  W_1
  \left(
    \widehat{\mathcal D}_{n},
    \widehat{\mathcal D}_{n}^{\mathrm{or}}
  \right)
  &\le
  \frac1n
  \sum_{i=1}^n
  \left|
    \bar q_i^{(n)}-G_i
  \right|
  \\
  &=
  \frac{\gamma}{n}
  \sum_{i=1}^n
  \left|
    \widehat V_{N_{S_i}^n}^\tau(S_i)
    -
    V^\tau(S_i)
  \right|
  \\
  &=
  \gamma
  \sum_{m=1}^M
  \frac{N_m^n}{n}
  \left|
    \widehat V_{N_m^n}^\tau(s_{h+1}^{(m)})
    -
    V^\tau(s_{h+1}^{(m)})
  \right|
  \\
  &\le
  \gamma
  \max_{1\le m\le M}
  \left|
    \widehat V_{N_m^n}^\tau(s_{h+1}^{(m)})
    -
    V^\tau(s_{h+1}^{(m)})
  \right|.
\end{aligned}
\]
If $\gamma=0$, this term is identically zero. Assume below that
$\gamma>0$.

Let
\[
  \mathcal E_n
  :=
  \left\{
    N_m^n\ge \frac{np_m}{2}
    \text{ for all }m=1,\ldots,M
  \right\}.
\]
By a multiplicative Chernoff bound,
\[
  \mathbb P(\mathcal E_n^c)
  \le
  \sum_{m=1}^M
  \mathbb P
  \left(
    N_m^n<\frac{np_m}{2}
  \right)
  \le
  M\exp
  \left(
    -\frac{np_{\min}}{8}
  \right).
\]
For $0<\varepsilon\le R_h$, this is bounded by
\[
  \frac{8MR_h^2}{e p_{\min}}
  \frac1{n\varepsilon^2}.
\]
For $\varepsilon>R_h$, the target probability is zero because both CVaR
values lie in $[0,R_h]$.

On $\mathcal E_n$, every child count satisfies
\[
  N_m^n\ge \frac{np_m}{2}.
\]
Thus, using the stopped-count concentration assumption and a union bound,
\[
\begin{aligned}
&\mathbb P
\left(
  W_1
  \left(
    \widehat{\mathcal D}_{n},
    \widehat{\mathcal D}_{n}^{\mathrm{or}}
  \right)
  >
  \frac{\tau\varepsilon}{2},
  \mathcal E_n
\right)
\\
&\qquad\le
\sum_{m=1}^M
\mathbb P
\left(
  \left|
    \widehat V_{N_m^n}^\tau(s_{h+1}^{(m)})
    -
    V^\tau(s_{h+1}^{(m)})
  \right|
  >
  \frac{\tau\varepsilon}{2\gamma},
  \mathcal E_n
\right)
\\
&\qquad\le
\sum_{m=1}^M
\frac{
  2c_m
}{
  np_m
}
\left(
  \frac{2\gamma}{\tau\varepsilon}
\right)^2
\\
&\qquad=
\frac{8\gamma^2}{\tau^2}
\left(
  \sum_{m=1}^M \frac{c_m}{p_m}
\right)
\frac1{n\varepsilon^2}.
\end{aligned}
\]

Combining the oracle empirical-measure term, the plug-in term, and the count
event gives a constant $c_Q>0$ such that
\[
  \mathbb P
  \left(
    \left|
      \widehat Q_{n}^{\tau}(s_h,a)
      -
      Q^\tau(s_h,a)
    \right|
    >
    \varepsilon
  \right)
  \le
  \frac{c_Q}{n\varepsilon^2}
\]
for every $n\ge1$ and every $\varepsilon>0$.

It remains to prove convergence of expectations. Since both
$\widehat Q_{n}^{\tau}(s_h,a)$ and $Q^\tau(s_h,a)$ lie in $[0,R_h]$,
the displayed tail bound implies convergence in probability, and boundedness
then implies convergence in $L^1$. Hence
\[
  \lim_{n\to\infty}
  \mathbb E
  \left[
    \left|
      \widehat Q_{n}^{\tau}(s_h,a)
      -
      Q^\tau(s_h,a)
    \right|
  \right]
  =
  0.
\]
Therefore,
\[
  \lim_{n\to\infty}
  \mathbb E
  \left[
    \widehat Q_{n}^{\tau}(s_h,a)
  \right]
  =
  Q^\tau(s_h,a).
\]
Together with the concentration inequality, this is exactly
\[
  \underset{n\to\infty}{\mathrm{plim}}\,
  \widehat Q_{n}^{\tau}(s_h,a)
  =
  Q^\tau(s_h,a),
\]
this concludes the proof.
\end{proof}

\subsection{Convergence of CATSO in MCTS}

\begin{theorem}[Finite-depth CVaR convergence of \algcatso]
\label{thm_cats_appendix}
Let \algcatso run on a finite tree of depth $H$ with bounded rewards and
finite transition support at every Q-edge.

Let
\[
  \widehat Q_{n}^\tau(s_h,a)
  :=
  \CVaR_\tau
  \left(
    \widehat{\mathcal D}_{n}(s_h,a)
  \right)
\]
be the empirical CVaR Q-edge estimator after $n$ local visits to edge
$(s_h,a)$. Then, for every node $s_h$
and every action $a\in\mathcal A(s_h)$,
\[
  \underset{n\to\infty}{\mathrm{plim}}\,
  \widehat Q_{n}^\tau(s_h,a)
  =
  \widetilde Q^\tau(s_h,a),
\]
and
\[
  \underset{n\to\infty}{\mathrm{plim}}\,
  \widehat V_{n}^\tau(s_h)
  =
  \widetilde V^\tau(s_h).
\]
\end{theorem}

\begin{proof}
We prove the result by backward induction on the remaining depth
$d:=H-h$.

\paragraph{Base case.}
At terminal depth $H$, the rollout or leaf estimator is assumed to satisfy
\[
  \underset{n\to\infty}{\mathrm{plim}}\,
  \widehat V_{n}^\tau(s_h)
  =
  \widetilde V^\tau(s_h).
\]
Thus the claim holds at depth $H$.

\paragraph{Induction hypothesis.}
Assume that for every child state $s_{h+1}$ at depth $h+1$,
\[
  \underset{n\to\infty}{\mathrm{plim}}\,
  \widehat V_{n}^\tau(s_{h+1})
  =
  \widetilde V^\tau(s_{h+1}).
\]
We prove the two claims for a node $s_h$ at depth $h$.

\paragraph{I: Q-edge CVaR convergence.}
Fix an action $a\in\mathcal A(s_h)$. On the $i$-th local visit to edge
$(s_h,a)$, write the sampled reward and next state as
\[
  X_i\sim \mathcal R(s_h,a,S_i),
  \qquad
  S_i\sim P(\cdot\mid s_h,a).
\]
The plug-in backed-up sample stored at the Q-edge is
\[
  \bar q_i^{(n)}
  =
  X_i
  +
  \gamma
  \widehat V_{N_{S_i}^n}^\tau(S_i),
\]
where $N_{S_i}^n$ is the number of times child state $S_i$ appears among the
first $n$ local visits to edge $(s_h,a)$.

The corresponding oracle backed-up sample is
\[
  G_i
  =
  X_i
  +
  \gamma
  \widetilde V^\tau(S_i).
\]
Therefore the population Q-edge law is
\[
  \mathcal D^\tau(s_h,a)
  :=
  \mathcal L(G_i),
\]
and the population CVaR Q-value is
\[
  \widetilde Q^\tau(s_h,a)
  =
  \CVaR_\tau
  \left(
    \mathcal D^\tau(s_h,a)
  \right).
\]

By the induction hypothesis, every child estimator
$\widehat V_{n}^\tau(s_{h+1})$ converges with concentration to
$\widetilde V^\tau(s_{h+1})$. Therefore
Lemma~\ref{lem:cvar_q_conc} applies and gives
\[
  \underset{n\to\infty}{\mathrm{plim}}\,
  \widehat Q_{n}^\tau(s_h,a)
  =
  \widetilde Q^\tau(s_h,a).
\]
Since $a$ was arbitrary, this proves Q-edge convergence for all actions at
$s_h$.

\paragraph{II: V-node convergence from the local bandit theorem.}
At node $s_h$, view each action $a\in\mathcal A(s_h)$ as an arm whose
empirical CVaR estimator is
\[
  \widehat c_{a,n}^\tau
  :=
  \widehat Q_{n}^\tau(s_h,a),
\]
and whose limiting CVaR value is
\[
  c_a^\tau
  :=
  \widetilde Q^\tau(s_h,a).
\]
Step I verifies the armwise assumption
\[
  \underset{n\to\infty}{\mathrm{plim}}\,
  \widehat c_{a,n}^\tau
  =
  c_a^\tau
\]
for every action $a$. Hence Assumption~\ref{assump:cvar_main} holds for the
local bandit at node $s_h$.

By Theorem~\ref{thm:cvar_bandit_main}, the local CATSO V-backup
\[
  \widehat V_{n}^\tau(s_h)
  =
  \sum_{a\in\mathcal A(s_h)}
  \frac{T_{s_h,a}(n)}{T_{s_h}(n)}
  \widehat Q_{T_{s_h,a}(n)}^\tau(s_h,a)
\]
satisfies
\[
  \underset{n\to\infty}{\mathrm{plim}}\,
  \widehat V_{n}^\tau(s_h)
  =
  \max_{a\in\mathcal A(s_h)}
  \widetilde Q^\tau(s_h,a).
\]
By definition of the target recursion,
\[
  \max_{a\in\mathcal A(s_h)}
  \widetilde Q^\tau(s_h,a)
  =
  \widetilde V^\tau(s_h).
\]
Therefore
\[
  \underset{n\to\infty}{\mathrm{plim}}\,
  \widehat V_{n}^\tau(s_h)
  =
  \widetilde V^\tau(s_h).
\]

\paragraph{Conclusion.}
The induction step proves both Q-edge and V-node convergence at depth $h$
whenever V-node convergence holds at depth $h+1$. Since the terminal case
holds, the result follows for every node in the finite-depth tree.
\end{proof}

\paragraph{Proof of Theorem~\ref{thm:mcts_convergence} for \algcatso case}

\begin{proof}
We prove the stronger statement that, for every depth $h$ and every node
$s_h$ at that depth, there exist constants $C_{h,V},C_{h,Q}<\infty$ such that
\[
  \mathbb E
  \left[
    \left|
      \widehat V_{n}^\tau(s_h)
      -
      \widetilde V^\tau(s_h)
    \right|
  \right]
  \le
  C_{h,V}n^{-1/2}
\]
and, for every action $a\in\mathcal A(s_h)$,
\[
  \mathbb E
  \left[
    \left|
      \widehat Q_{n}^\tau(s_h,a)
      -
      \widetilde Q^\tau(s_h,a)
    \right|
  \right]
  \le
  C_{h,Q}n^{-1/2}.
\]
The desired theorem follows by taking $h=0$ and $s_h=s_0$.

We use backward induction on the remaining depth $H-h$.

\paragraph{Base case.}
At depth $H$, the claimed V-rate holds by the terminal assumption:
\[
  \mathbb E
  \left[
    \left|
      \widehat V_{n}^\tau(s_h)
      -
      \widetilde V^\tau(s_h)
    \right|
  \right]
  \le
  C_H n^{-1/2}.
\]
Thus the induction starts with $C_{H,V}:=C_H$.

\paragraph{Induction hypothesis.}
Assume that, for all nodes $s_{h+1}$ at depth $h+1$,
\[
  \mathbb E
  \left[
    \left|
      \widehat V_{n}^\tau(s_{h+1})
      -
      \widetilde V^\tau(s_{h+1})
    \right|
  \right]
  \le
  C_{h+1,V} n^{-1/2}.
\]
We prove the corresponding Q-rate and V-rate at depth $h$.

\paragraph{I: Q-edge CVaR rate.}
Fix a node $s_h$ and an action $a\in\mathcal A(s_h)$. Let
\[
  \{s_{h+1}^{(1)},\ldots,s_{h+1}^{(M)}\}
\]
be the finite support of $P(\cdot\mid s_h,a)$, and write
\[
  p_m:=P(s_{h+1}^{(m)}\mid s_h,a).
\]
On the $i$-th local visit to edge $(s_h,a)$, let
\[
  X_i\sim \mathcal R(s_h,a,S_i),
  \qquad
  S_i\sim P(\cdot\mid s_h,a).
\]
Define the plug-in backed-up sample
\[
  \bar q_i^{(n)}
  :=
  X_i
  +
  \gamma
  \widehat V_{N_{S_i}^n}^\tau(S_i),
\]
where $N_{S_i}^n$ is the number of appearances of child state $S_i$ among
the first $n$ local visits to edge $(s_h,a)$. Define also the oracle sample
\[
  G_i
  :=
  X_i
  +
  \gamma
  \widetilde V^\tau(S_i).
\]
Let
\[
  \widehat{\mathcal D}_{n}(s_h,a)
  :=
  \frac1n\sum_{i=1}^n \delta_{\bar q_i^{(n)}},
  \qquad
  \mathcal D_h^\tau(s_h,a)
  :=
  \mathcal L(G_1).
\]
Then
\[
  \widehat Q_{n}^\tau(s_h,a)
  =
  \CVaR_\tau
  \left(
    \widehat{\mathcal D}_{n}(s_h,a)
  \right),
  \qquad
  \widetilde Q^\tau(s_h,a)
  =
  \CVaR_\tau
  \left(
    \mathcal D_h^\tau(s_h,a)
  \right).
\]

By the $1/\tau$-Lipschitz property of CVaR in $W_1$,
\[
  \left|
    \widehat Q_{n}^\tau(s_h,a)
    -
    \widetilde Q^\tau(s_h,a)
  \right|
  \le
  \frac1\tau
  W_1
  \left(
    \widehat{\mathcal D}_{n}(s_h,a),
    \mathcal D_h^\tau(s_h,a)
  \right).
\]
Introduce the oracle empirical law
\[
  \widehat{\mathcal D}_{h,n}^{\mathrm{or}}(s_h,a)
  :=
  \frac1n\sum_{i=1}^n\delta_{G_i}.
\]
The triangle inequality gives
\[
\begin{aligned}
  W_1
  \left(
    \widehat{\mathcal D}_{h,n},
    \mathcal D_h^\tau
  \right)
  &\le
  W_1
  \left(
    \widehat{\mathcal D}_{h,n},
    \widehat{\mathcal D}_{h,n}^{\mathrm{or}}
  \right)
  +
  W_1
  \left(
    \widehat{\mathcal D}_{n}^{\mathrm{or}},
    \mathcal D^\tau
  \right).
\end{aligned}
\]

For the first term, couple each plug-in sample with its oracle counterpart.
Then
\[
\begin{aligned}
  W_1
  \left(
    \widehat{\mathcal D}_{n},
    \widehat{\mathcal D}_{n}^{\mathrm{or}}
  \right)
  &\le
  \frac1n\sum_{i=1}^n
  \left|
    \bar q_i^{(n)}-G_i
  \right|
  \\
  &=
  \frac{\gamma}{n}
  \sum_{i=1}^n
  \left|
    \widehat V_{N_{S_i}^n}^\tau(S_i)
    -
    \widetilde V^\tau(S_i)
  \right|
  \\
  &=
  \gamma
  \sum_{m=1}^M
  \frac{N_m^n}{n}
  \left|
    \widehat V_{N_m^n}^\tau(s_{h+1}^{(m)})
    -
    \widetilde V^\tau(s_{h+1}^{(m)})
  \right|.
\end{aligned}
\]
Taking expectations and using the induction hypothesis at the stopped count
$N_m^n$,
\[
\begin{aligned}
&\mathbb E
\left[
  W_1
  \left(
    \widehat{\mathcal D}_{n},
    \widehat{\mathcal D}_{n}^{\mathrm{or}}
  \right)
\right]
\\
&\qquad\le
\gamma
\sum_{m=1}^M
\mathbb E
\left[
  \frac{N_m^n}{n}
  C_{h+1,V}(N_m^n\vee1)^{-1/2}
\right]
\\
&\qquad\le
\frac{\gamma C_{h+1,V}}{n}
\sum_{m=1}^M
\mathbb E\left[\sqrt{N_m^n}\right]
\\
&\qquad\le
\frac{\gamma C_{h+1,V}}{n}
\sum_{m=1}^M
\sqrt{\mathbb E[N_m^n]}
\\
&\qquad=
\gamma C_{h+1,V}
\left(
  \sum_{m=1}^M \sqrt{p_m}
\right)n^{-1/2}.
\end{aligned}
\]
The inequality
\[
  \mathbb E[\sqrt{N_m^n}]
  \le
  \sqrt{\mathbb E[N_m^n]}
\]
uses Jensen's inequality.

For the second term, the oracle samples $G_1,\ldots,G_n$ are i.i.d. and
bounded. Let
\[
  R_h
  :=
  R_{\max}
  \sum_{\ell=0}^{H-h-1}\gamma^\ell
  =
  R_{\max}
  \frac{1-\gamma^{H-h}}{1-\gamma}.
\]
Then $\mathcal D^\tau(s_h,a)$ is supported on $[0,R_h]$. By the
one-dimensional empirical $W_1$ bound,
\[
  \mathbb E
  \left[
    W_1
    \left(
      \widehat{\mathcal D}_{n}^{\mathrm{or}},
      \mathcal D^\tau
    \right)
  \right]
  \le
  C_{\mathrm{FG}}R_h n^{-1/2}.
\]
Combining the two bounds yields
\[
\begin{aligned}
  \mathbb E
  \left[
    \left|
      \widehat Q_{n}^\tau(s_h,a)
      -
      \widetilde Q^\tau(s_h,a)
    \right|
  \right]
  &\le
  \frac1\tau
  \left[
    \gamma C_{h+1,V}
    \sum_{m=1}^M\sqrt{p_m}
    +
    C_{\mathrm{FG}}R_h
  \right]
  n^{-1/2}.
\end{aligned}
\]
Thus the Q-rate holds with
\[
  C_{h,Q}
  :=
  \frac1\tau
  \left[
    \gamma C_{h+1,V}
    \max_{s_h,a}
    \sum_{s'\in\mathrm{supp}(P(\cdot\mid s_h,a))}
    \sqrt{P(s'\mid s_h,a)}
    +
    C_{\mathrm{FG}}R_h
  \right].
\]

\paragraph{II: V-node CVaR-average rate.}
Fix a node $s_h$. Write
\[
  q_a^\tau
  :=
  \widetilde Q^\tau(s_h,a),
  \qquad
  q_\star^\tau
  :=
  \max_{a\in\mathcal A(s_h)}q_a^\tau
  =
  \widetilde V^\tau(s_h),
\]
and let
\[
  a_\tau^\star(s_h)
  \in
  \arg\max_{a\in\mathcal A(s_h)}q_a^\tau
\]
be the unique CVaR-optimal action. For readability, write
\[
  T_a(n):=T_{s_h,a}(n),
  \qquad
  w_a(n):=\frac{T_a(n)}{n}.
\]
The algorithm uses the CVaR-average V-backup
\[
  \widehat V_{n}^\tau(s_h)
  =
  \sum_{a\in\mathcal A(s_h)}
  w_a(n)
  \widehat Q_{h,T_a(n)}^\tau(s_h,a).
\]
Subtracting $q_\star^\tau$ and adding/subtracting $q_a^\tau$ gives
\[
\begin{aligned}
  \widehat V_{n}^\tau(s_h)-q_\star^\tau
  &=
  \sum_a
  w_a(n)
  \left(
    \widehat Q_{h,T_a(n)}^\tau(s_h,a)
    -
    q_a^\tau
  \right) + \sum_a w_a(n)
  \left(
    q_a^\tau-q_\star^\tau
  \right).
\end{aligned}
\]
Taking absolute values and using that all CVaR values lie in $[0,R_h]$,
\[
\begin{aligned}
  \left|
    \widehat V_{n}^\tau(s_h)-q_\star^\tau
  \right|
  &\le
  \sum_a
  w_a(n)
  \left|
    \widehat Q_{h,T_a(n)}^\tau(s_h,a)-q_a^\tau
  \right|
  +
  R_h
  \sum_{a\neq a_\tau^\star(s_h)}w_a(n).
\end{aligned}
\]
For suboptimal actions, the Q-error is also bounded by $R_h$. Hence
\[
\begin{aligned}
  \left|
    \widehat V_{n}^\tau(s_h)-q_\star^\tau
  \right|
  &\le
  \left|
    \widehat Q_{h,T_{a_\tau^\star}(n)}^\tau
    (s_h,a_\tau^\star)
    -
    q_\star^\tau
  \right|
  +
  2R_h
  \sum_{a\neq a_\tau^\star(s_h)}w_a(n).
\end{aligned}
\]
Taking expectations gives
\begin{equation}
\label{eq:v_rate_decomp_catso}
\begin{aligned}
  \mathbb E
  \left[
    \left|
      \widehat V_{n}^\tau(s_h)
      -
      \widetilde V^\tau(s_h)
    \right|
  \right]
  &\le
  \mathbb E
  \left[
    \left|
      \widehat Q_{h,T_{a_\tau^\star}(n)}^\tau
      (s_h,a_\tau^\star)
      -
      \widetilde Q^\tau(s_h,a_\tau^\star)
    \right|
  \right] + \frac{2R_h}{n} \sum_{a\neq a_\tau^\star(s_h)}
  \mathbb E[T_a(n)].
\end{aligned}
\end{equation}

We first bound the suboptimal-weight term. By the local \algcatso
suboptimal-arm pull bound,
\[
  \mathbb E[T_a(n)]
  \le
  L_a\log n+D_a n^{1/2}
\]
for constants $L_a,D_a<\infty$. Therefore,
\[
  \frac{1}{n}
  \sum_{a\neq a_\tau^\star(s_h)}
  \mathbb E[T_a(n)]
  =
  O(n^{-1/2}),
\]
since $\log n/n=O(n^{-1/2})$.

It remains to control the optimal-action Q-error at the random count
$T_{a_\tau^\star}(n)$. Let
\[
  S_n
  :=
  \sum_{a\neq a_\tau^\star(s_h)}T_a(n).
\]
Then
\[
  T_{a_\tau^\star}(n)=n-S_n.
\]
By Markov's inequality and the suboptimal-pull bound,
\[
  \mathbb P\left(S_n>\frac n2\right)
  \le
  \frac{2\mathbb E[S_n]}{n}
  =
  O(n^{-1/2}).
\]
On the event $\{S_n\le n/2\}$, we have
\[
  T_{a_\tau^\star}(n)\ge \frac n2.
\]
Using the stopped-count form of the Q-rate from I,
\[
\begin{aligned}
&\mathbb E
\left[
  \left|
    \widehat Q_{T_{a_\tau^\star}(n)}^\tau
    (s_h,a_\tau^\star)
    -
    \widetilde Q^\tau(s_h,a_\tau^\star)
  \right|
  \mathbf 1\left\{S_n\le \frac n2\right\}
\right]
\\
&\qquad\le
C_{h,Q}
\left(\frac n2\right)^{-1/2}.
\end{aligned}
\]
On the complement $\{S_n>n/2\}$, the absolute error is bounded by $R_h$.
Thus
\[
\begin{aligned}
&\mathbb E
\left[
  \left|
    \widehat Q_{T_{a_\tau^\star}(n)}^\tau
    (s_h,a_\tau^\star)
    -
    \widetilde Q^\tau(s_h,a_\tau^\star)
  \right|
\right]
\\
&\qquad\le
\sqrt{2}\,C_{h,Q}n^{-1/2}
+
R_h\,\mathbb P\left(S_n>\frac n2\right)
\\
&\qquad=
O(n^{-1/2}).
\end{aligned}
\]
Substituting this and the suboptimal-weight bound into
\eqref{eq:v_rate_decomp_catso} gives
\[
  \mathbb E
  \left[
    \left|
      \widehat V_{n}^\tau(s_h)
      -
      \widetilde V^\tau(s_h)
    \right|
  \right]
  \le
  C_{h,V}n^{-1/2}
\]
for some finite constant $C_{h,V}$.

\paragraph{In sum.}
The induction proves the V-rate at depth $h$ whenever it holds at depth
$h+1$. Since the terminal case holds, the claim holds for every node in the
finite-depth tree. In particular, at the root $s_0$,
\[
  \mathbb E
  \left[
    \left|
      \widehat V_n^\tau(s_0)
      -
      \widetilde V^\tau(s_0)
    \right|
  \right]
  \le
  C'n^{-1/2}
\]
for a finite constant $C'=C_{0,V}$.
\end{proof}

\subsection{Convergence of PATSO in MCTS}

Similarly, we can derive the following Theorem

\begin{theorem}[Finite-depth CVaR convergence of \algpatso]
\label{thm_pats_appendix}
Let \algpatso run on a finite tree of depth $H$ with bounded rewards and
finite transition support at every Q-edge.

Let
\[
  \widehat Q_{n}^\tau(s_h,a)
  :=
  \CVaR_\tau
  \left(
    \widehat{\mathcal D}_{n}(s_h,a)
  \right)
\]
be the empirical CVaR Q-edge estimator after $n$ local visits to edge
$(s_h,a)$. Then, for every node $s_h$
and every action $a\in\mathcal A(s_h)$,
\[
  \underset{n\to\infty}{\mathrm{plim}}\,
  \widehat Q_{n}^\tau(s_h,a)
  =
  \widetilde Q^\tau(s_h,a),
\]
and
\[
  \underset{n\to\infty}{\mathrm{plim}}\,
  \widehat V_{n}^\tau(s_h)
  =
  \widetilde V^\tau(s_h).
\]
\end{theorem}

\begin{proof}
    The proof follows the same steps as Theorem~\ref{thm_cats_appendix} by applying the results of Lemma~\ref{lem:cvar_q_conc} and Theorem~\ref{thm:cvar_bandit_main}.
\end{proof}

\paragraph{Proof of Theorem~\ref{thm:mcts_convergence} for \algpatso case}

\begin{proof}
We prove the stronger statement that, for every depth $h$ and every node
$s_h$ at that depth, there exists a constant $C_{h,V}<\infty$ such that
\[
  \mathbb E
  \left[
    \left|
      \widehat V_{n}^\tau(s_h)
      -
      \widetilde V^\tau(s_h)
    \right|
  \right]
  \le
  C_{h,V}n^{-1/2}.
\]
The desired result follows by taking $h=0$ and $s_h=s_0$.

We proceed by backward induction on the remaining depth $H-h$.

\paragraph{Base case.}
At terminal depth $H$, the claim is exactly the assumed rollout/leaf rate:
\[
  \mathbb E
  \left[
    \left|
      \widehat V_{n}^\tau(s_H)
      -
      \widetilde V^\tau(s_H)
    \right|
  \right]
  \le
  C_H n^{-1/2}.
\]
Thus the induction starts with $C_{H,V}:=C_H$.

\paragraph{Induction hypothesis.}
Assume that for every child node $s_{h+1}$ at depth $h+1$,
\[
  \mathbb E
  \left[
    \left|
      \widehat V_{n}^\tau(s_{h+1})
      -
      \widetilde V^\tau(s_{h+1})
    \right|
  \right]
  \le
  C_{h+1,V}n^{-1/2}.
\]
We prove the corresponding statement at an arbitrary node $s_h$ at depth $h$.

\paragraph{I: Q-edge CVaR convergence.}
Fix an action $a\in\mathcal A(s_h)$. Let
\[
  \{s_{h+1}^{(1)},\ldots,s_{h+1}^{(M)}\}
\]
be the finite support of $P(\cdot\mid s_h,a)$, and write
\[
  p_m:=P(s_{h+1}^{(m)}\mid s_h,a).
\]
On the $i$-th local visit to edge $(s_h,a)$, let
\[
  X_i\sim \mathcal R(s_h,a,S_i),
  \qquad
  S_i\sim P(\cdot\mid s_h,a).
\]
The backed-up sample stored by \algpatso is
\[
  \overline q_i
  =
  X_i
  +
  \gamma
  \widehat V_{\ell_i}^\tau(S_i),
\]
where $\ell_i$ is the local visit count of the child state $S_i$ at the
moment this backup is formed. The corresponding oracle backed-up sample is
\[
  G_i
  =
  X_i+\gamma \widetilde V^\tau(S_i).
\]
Define the empirical \algpatso Q-edge law after $n$ local visits by
\[
  \widehat{\mathcal D}_{n}(s_h,a)
  :=
  \frac1n
  \sum_{i=1}^n
  \delta_{\overline q_i},
\]
and define the oracle law by
\[
  \mathcal D^\tau(s_h,a)
  :=
  \mathcal L(G_1).
\]
Since uncapped \algpatso stores one particle per distinct observed backed-up
return, its maintained particle law is exactly
$\widehat{\mathcal D}_{n}(s_h,a)$.

The empirical and population Q-edge CVaR values are
\[
  \widehat Q_{n}^\tau(s_h,a)
  :=
  \CVaR_\tau
  \left(
    \widehat{\mathcal D}_{n}(s_h,a)
  \right),
\]
and
\[
  \widetilde Q^\tau(s_h,a)
  :=
  \CVaR_\tau
  \left(
    \mathcal D^\tau(s_h,a)
  \right).
\]

We show that the armwise condition needed by the local \algpatso bandit theorem
holds. Introduce the oracle empirical law
\[
  \widehat{\mathcal D}_{n}^{\mathrm{or}}(s_h,a)
  :=
  \frac1n
  \sum_{i=1}^n
  \delta_{G_i}.
\]
By the $1/\tau$-Lipschitz property of CVaR in $W_1$,
\[
  \left|
    \widehat Q_{n}^\tau(s_h,a)
    -
    \widetilde Q^\tau(s_h,a)
  \right|
  \le
  \frac1\tau
  W_1
  \left(
    \widehat{\mathcal D}_{n}(s_h,a),
    \mathcal D^\tau(s_h,a)
  \right).
\]
Using the triangle inequality,
\[
\begin{aligned}
  W_1
  \left(
    \widehat{\mathcal D}_{n},
    \mathcal D^\tau
  \right)
  &\le
  W_1
  \left(
    \widehat{\mathcal D}_{n},
    \widehat{\mathcal D}_{n}^{\mathrm{or}}
  \right)
  +
  W_1
  \left(
    \widehat{\mathcal D}_{n}^{\mathrm{or}},
    \mathcal D^\tau
  \right).
\end{aligned}
\]

For the first term, use the natural index-wise coupling
\[
  \pi_n
  :=
  \frac1n
  \sum_{i=1}^n
  \delta_{(\overline q_i,G_i)}.
\]
This coupling has marginals
$\widehat{\mathcal D}_{n}$ and
$\widehat{\mathcal D}_{n}^{\mathrm{or}}$, so
\[
\begin{aligned}
  W_1
  \left(
    \widehat{\mathcal D}_{n},
    \widehat{\mathcal D}_{n}^{\mathrm{or}}
  \right)
  &\le
  \frac1n
  \sum_{i=1}^n
  |\overline q_i-G_i|
  \\
  &=
  \frac{\gamma}{n}
  \sum_{i=1}^n
  \left|
    \widehat V_{\ell_i}^\tau(S_i)
    -
    \widetilde V^\tau(S_i)
  \right|.
\end{aligned}
\]
Group the sum according to the sampled child state. If child state
$s_{h+1}^{(m)}$ is observed $N_m^n$ times among the first $n$ visits to
edge $(s_h,a)$, then its local child estimator appears with visit indices
$r=1,\ldots,N_m^n$. Hence
\[
\begin{aligned}
  W_1
  \left(
    \widehat{\mathcal D}_{n},
    \widehat{\mathcal D}_{n}^{\mathrm{or}}
  \right)
  &\le
  \frac{\gamma}{n}
  \sum_{m=1}^M
  \sum_{r=1}^{N_m^n}
  \left|
    \widehat V_{r}^\tau(s_{h+1}^{(m)})
    -
    \widetilde V^\tau(s_{h+1}^{(m)})
  \right|.
\end{aligned}
\]
Taking expectations and applying the induction hypothesis,
\[
\begin{aligned}
&\mathbb E
\left[
  W_1
  \left(
    \widehat{\mathcal D}_{n},
    \widehat{\mathcal D}_{n}^{\mathrm{or}}
  \right)
\right]
\\
&\qquad\le
\frac{\gamma C_{h+1,V}}{n}
\sum_{m=1}^M
\mathbb E
\left[
  \sum_{r=1}^{N_m^n} r^{-1/2}
\right]
\\
&\qquad\le
\frac{2\gamma C_{h+1,V}}{n}
\sum_{m=1}^M
\mathbb E
\left[
  \sqrt{N_m^n}
\right]
\\
&\qquad\le
\frac{2\gamma C_{h+1,V}}{n}
\sum_{m=1}^M
\sqrt{\mathbb E[N_m^n]}
\\
&\qquad=
2\gamma C_{h+1,V}
\left(
  \sum_{m=1}^M \sqrt{p_m}
\right)
n^{-1/2}.
\end{aligned}
\]
The second inequality uses
\[
  \sum_{r=1}^{N}r^{-1/2}\le 2\sqrt N,
\]
and the third uses Jensen's inequality.

For the second term, the oracle samples
\[
  G_i=X_i+\gamma\widetilde V^\tau(S_i)
\]
are i.i.d. and bounded. Let
\[
  R_h
  :=
  R_{\max}
  \sum_{\ell=0}^{H-h-1}\gamma^\ell
  =
  R_{\max}
  \frac{1-\gamma^{H-h}}{1-\gamma}.
\]
Then $\mathcal D^\tau(s_h,a)$ is supported on $[0,R_h]$. By the
one-dimensional empirical $W_1$ bound,
\[
  \mathbb E
  \left[
    W_1
    \left(
      \widehat{\mathcal D}_{n}^{\mathrm{or}},
      \mathcal D^\tau
    \right)
  \right]
  \le
  C_{\mathrm{FG}}R_h n^{-1/2}.
\]
Combining the two $W_1$ bounds gives
\[
\begin{aligned}
  \mathbb E
  \left[
    \left|
      \widehat Q_{n}^\tau(s_h,a)
      -
      \widetilde Q^\tau(s_h,a)
    \right|
  \right]
  &\le
  \frac1\tau
  \left[
    2\gamma C_{h+1,V}
    \sum_{m=1}^M\sqrt{p_m}
    +
    C_{\mathrm{FG}}R_h
  \right]
  n^{-1/2}.
\end{aligned}
\]
Thus, for every action $a$, the empirical Q-edge CVaR estimator converges to
$\widetilde Q^\tau(s_h,a)$ at rate $O(n^{-1/2})$ in $L^1$. Since all
quantities are bounded, the same argument also verifies the armwise
convergence-with-concentration premise used by the local \algpatso bandit theorem.

\paragraph{II: Apply the local \algpatso CVaR bandit theorem at $s_h$.}
At node $s_h$, view every action $a\in\mathcal A(s_h)$ as a local bandit arm
with limiting CVaR value
\[
  c_a^\tau
  :=
  \widetilde Q^\tau(s_h,a).
\]
The local optimal CVaR value is
\[
  c_\star^\tau
  :=
  \max_{a\in\mathcal A(s_h)}
  c_a^\tau
  =
  \widetilde V^\tau(s_h).
\]
The empirical arm score after $m$ local pulls is
\[
  \widehat c_{a,m}^\tau
  :=
  \widehat Q_{m}^\tau(s_h,a).
\]
By I, the armwise empirical CVaR scores satisfy the required
asymptotic stationarity/concentration condition for the local bandit at
$s_h$.

The V-node backup used by the algorithm is exactly the CVaR-average local
bandit backup:
\[
  \widehat V_{n}^\tau(s_h)
  =
  \sum_{a\in\mathcal A(s_h)}
  \frac{T_{s_h,a}(n)}{T_{s_h}(n)}
  \widehat Q_{T_{s_h,a}(n)}^\tau(s_h,a).
\]
By the local \algpatso CVaR bandit theorem with the $O(n^{-1/2})$ $L^1$
rate,
\[
  \mathbb E
  \left[
    \left|
      \widehat V_{n}^\tau(s_h)
      -
      c_\star^\tau
    \right|
  \right]
  \le
  C_{h,\mathrm{B}}n^{-1/2},
\]
for some constant $C_{h,\mathrm{B}}<\infty$. Since
\[
  c_\star^\tau
  =
  \widetilde V^\tau(s_h),
\]
we obtain
\[
  \mathbb E
  \left[
    \left|
      \widehat V_{n}^\tau(s_h)
      -
      \widetilde V^\tau(s_h)
    \right|
  \right]
  \le
  C_{h,V}n^{-1/2},
\]
with $C_{h,V}:=C_{h,\mathrm{B}}$.

\paragraph{Conclusion.}
The induction step shows that if all child V-estimators at depth $h+1$ have
the $O(n^{-1/2})$ $L^1$ rate, then the V-estimator at every parent node
$s_h$ also has the $O(n^{-1/2})$ $L^1$ rate. Since the terminal case holds,
the claim follows for every node in the finite-depth tree. In particular,
at the root node $s_0$ there exists a finite constant $C'=C_{0,V}$ such that
\[
  \mathbb E
  \left[
    \left|
      \widehat V_n^\tau(s_0)
      -
      \widetilde V^\tau(s_0)
    \right|
  \right]
  \le
  C'n^{-1/2}.
\]
\end{proof}

\subsection{Risk-Sensitive Simple-Regret Bound}
\label{app:simple_regret}

The convergence theorems in Section~\ref{sec:theory} control the absolute
root value error
\[
  \mathcal E_\tau(n)
  :=
  \mathbb E
  \left[
    \left|
      \widehat V_n^\tau(s_0)
      -
      \widetilde V^\tau(s_0)
    \right|
  \right].
\]
The risk-sensitive simple regret used in the main text is the signed root
value gap
\[
  \mathrm{SR}_\tau(n)
  :=
  \mathbb E
  \left[
    \widetilde V^\tau(s_0)
    -
    \widehat V_n^\tau(s_0)
  \right].
\]
Since $x\le |x|$ for every real number $x$, we have
\[
\begin{aligned}
  \mathrm{SR}_\tau(n)
  &=
  \mathbb E
  \left[
    \widetilde V^\tau(s_0)
    -
    \widehat V_n^\tau(s_0)
  \right]
  \\
  &\le
  \mathbb E
  \left[
    \left|
      \widetilde V^\tau(s_0)
      -
      \widehat V_n^\tau(s_0)
    \right|
  \right]
  =
  \mathcal E_\tau(n).
\end{aligned}
\]
Thus any upper bound on the absolute root value error immediately gives the
same-order upper bound on $\mathrm{SR}_\tau(n)$.

\begin{corollary}[Risk-sensitive simple regret of CATSO]
\label{cor:catso_simple_regret}
Under the assumptions of Theorem~\ref{thm:mcts_convergence}, fixed-grid CATSO
satisfies
\[
  \mathrm{SR}_\tau(n)
  \le
  C_{\mathrm{CATSO}} n^{-1/2}.
\]
For the unprojected target with grid mesh $\Delta_N$,
\[
  \mathrm{SR}_\tau(n)
  \le
  C_{\mathrm{CATSO}} n^{-1/2}
  +
  O\left(\frac{\Delta_N}{\tau}\right).
\]
\end{corollary}

\begin{proof}
Theorem~\ref{thm:mcts_convergence} gives
\[
  \mathcal E_\tau(n)
  \le
  C_{\mathrm{CATSO}}n^{-1/2}
\]
for the fixed-grid target, and adds the discretization term
$O(\Delta_N/\tau)$ for the unprojected target.  Since
$\mathrm{SR}_\tau(n)\le \mathcal E_\tau(n)$, the result follows.
\end{proof}

\begin{corollary}[Risk-sensitive simple regret of uncapped PATSO]
\label{cor:patso_simple_regret}
Under the assumptions of Theorem~\ref{thm:mcts_convergence}, uncapped PATSO
satisfies
\[
  \mathrm{SR}_\tau(n)
  \le
  C_{\mathrm{PATSO}} n^{-1/2}.
\]
\end{corollary}

\begin{proof}
Theorem~\ref{thm:mcts_convergence} gives
\[
  \mathcal E_\tau(n)
  \le
  C_{\mathrm{PATSO}}n^{-1/2}.
\]
The claim follows from
\[
  \mathrm{SR}_\tau(n)\le \mathcal E_\tau(n).
\]
\end{proof}

\begin{corollary}[Risk-sensitive simple regret of capped PATSO]
\label{cor:capped_patso_simple_regret}
Suppose capped PATSO satisfies the compression condition
\[
  W_1
  \left(
    \widehat{\mathcal D}^{\mathrm{cap}}_{m,K}(s,a),
    \widehat{\mathcal D}^{\mathrm{uncap}}_{m}(s,a)
  \right)
  \le
  \varepsilon_K(s,a)
\]
at every visited Q-edge.  Then
\[
  \mathrm{SR}_\tau(n)
  \le
  C_{\mathrm{PATSO}} n^{-1/2}
  +
  O\left(
    \frac{\varepsilon_K}{\tau}
  \right),
\]
where depth-dependent constants are absorbed into the $O(\cdot)$ term.
\end{corollary}

\begin{proof}
By the $1/\tau$-Lipschitz property of CVaR in $W_1$, replacing the uncapped
empirical Q-edge law by the capped law changes each empirical CVaR score by at
most $\varepsilon_K/\tau$.  Propagating this additive bias through the finite
depth tree gives the stated additional term.  Combining with the uncapped
PATSO bound and using
$\mathrm{SR}_\tau(n)\le \mathcal E_\tau(n)$ proves the result.
\end{proof}

\section{Memory Capping and \texorpdfstring{$W_1$}{W1} Compression}
\label{app:memory_capping}
\subsection{Merge-on-Insert (Precise Algorithm)}

We keep a sorted container of pairs $\{(Q_i, \alpha_i)\}_{i=1}^M$ with $M \le K$. Upon receiving a new sample $z$:

\begin{enumerate}[leftmargin=*, itemsep=3pt]
  \item \textbf{Match:} If $\min_i |z - Q_i| \le \varepsilon_{\text{tol}}$, let $i^\star = \arg\min_i |z - Q_i|$ and set $\alpha_{i^\star} \leftarrow \alpha_{i^\star} + 1$.
  
  \item \textbf{Insert:} Else if $M < K$, insert $(z, 1)$ and keep the list sorted.
  
  \item \textbf{Merge:} Else (overflow): find $j$ with smallest gap $|Q_{j+1} - Q_j|$, replace the pair by
  \[
    \big(Q^{\text{new}}, \alpha^{\text{new}}\big) = \left(\frac{\alpha_j Q_j + \alpha_{j+1} Q_{j+1}}{\alpha_j + \alpha_{j+1}},\; \alpha_j + \alpha_{j+1}\right),
  \]
  then insert $(z, 1)$ in sorted order.
\end{enumerate}

\subsection{Bound on \texorpdfstring{$W_1$}{W1} Error Under Capping}

Define $\texttt{range} := \max_i{Q_i} - \min_i{Q_i}$ and let $\widehat{P}$ be the (uncapped) empirical distribution and $\widetilde{P}$ the capped distribution after applying the procedure above. With nearest-neighbor merges:
\[
  W_1(\widehat{P}, \widetilde{P}) \leq \sum_{\text{merges}} (\alpha_j + \alpha_{j+1}) \cdot |Q_{j+1} - Q_j| = O\left(\frac{\texttt{range}}{K}\right).
\]

Since the Bellman operator is $1/(1-\gamma)$-Lipschitz in $W_1$, the root-value error contributed by capping is $O\big(\texttt{range}/(K(1-\gamma))\big)$ and the simple-regret additive term is $O\big(\texttt{range}/(K(1-\gamma)^2)\big)$, which is dominated by the proven $O(n^{-1/2})$ term for moderate $K$.

\section{Experimental Details}
\label{app:experiments}
All experiments were conducted on the I9 14900K processor, 6.00 GHz, with 24 cores and 32 threads in each core, 62 GB RAM. Whenever feasible, we opted for open-source implementations of algorithms and environments. The tuning process takes around 3 hours and the evaluation takes around 2 hours with the best configuration for each environment.

\subsection{Parameter Selection}
We conduct the hyperparameter-tuning experiments with 10 seeds for each environments. Specifically, we search over the number of atoms $N$ in the grid $\{10, 11, \dots, 200\}$, polynomial bonus constant $C$ in the set $\{1,2,4,8\}$ and the CVaR $\tau$ value in $\{0.05, 0.1 , 0.2, 0.25\}$ for CATSO and PATSO. For the baselines, we also tune to find the best config in terms of bias $c$, exploration bonus $\epsilon$ and the power $p$.
\subsection{Environments and Algorithms Hyperparameters}
For each environment, we run CATSO and PATSO with the baselines UCT, PowerUCT, MaxUCT with 50 seeds, each seed uses $10^4$ to $10^5$ simulations depending on the complexity of the environment and the tree depth is set to equal the environment horizon. At each simulation checkpoint, we compute the regret at the root node and perform 200 Monte Carlo evaluation rollouts at each simulation to compute the average of catastrophic count and CVaR return (worst $\tau$-portion of the trajectories in term of return).
Below are the detailed environment configurations. For the hyperparameter used in each environment, refer to table \ref{tab:algorithm_hyperparameters}. One important note for \algcatso is that, since it is difficult to determine a reasonable heuristic grid for the return distribution in the Q-Edge, we choose to build the distribution adaptively. Specifically, we initialize $q_{\min} = 0$ and $q_{\max} = 0.01$ and build the empirical distribution by discretizing the interval $[q_{\min}, q_{\max}]$ into $N$ atoms, then the planner will adaptively expand the distribution as it receives more samples. Details of the adaptive distribution construction is in line 9 to 13 of the pseudo-code in figure \ref{alg:adaptive-simulateQ}. And for \algpatso, as describe in appendix \ref{app:memory_capping}, we use capped version of the particle distribution for memory capability in practical.

\begin{table}[t]
\centering
\caption{Algorithm hyperparameters used for each environment. The value \texttt{auto} means that the code uses \texttt{USE\_AUTO\_BIAS} and computes the effective bias adaptively at runtime.}
\label{tab:algorithm_hyperparameters}
\resizebox{\textwidth}{!}{
\begin{tabular}{lll}
\toprule
Environment & Algorithm & Hyperparameters \\
\midrule

\multirow{5}{*}{Two-level Risky Treasure}
& UCT & \texttt{bias=auto}, \texttt{epsilon=0.5} \\
& MaxUCT & \texttt{bias=auto}, \texttt{epsilon=0.5} \\
& PowerUCT & \texttt{bias=auto}, \texttt{epsilon=0.5}, \texttt{p=2.0} \\
& CATSO & \texttt{n\_atoms=25}, \texttt{optimism=2.0}, \texttt{p=1.0}, \texttt{cvar\_tau=0.25} \\
& PATSO & \texttt{max\_particles=128}, \texttt{optimism=2.0}, \texttt{p=1.0}, \texttt{cvar\_tau=0.25} \\

\midrule

\multirow{5}{*}{Autonomous Vehicle}
& UCT & \texttt{bias=10}, \texttt{epsilon=0.2} \\
& MaxUCT & \texttt{bias=10}, \texttt{epsilon=0.2} \\
& PowerUCT & \texttt{bias=10}, \texttt{epsilon=0.2}, \texttt{p=2.0} \\
& CATSO & \texttt{n\_atoms=100}, \texttt{optimism=4.0}, \texttt{p=1.0}, \texttt{cvar\_tau=0.05} \\
& PATSO & \texttt{max\_particles=64}, \texttt{optimism=2.0}, \texttt{p=1.0}, \texttt{cvar\_tau=0.05} \\

\midrule

\multirow{5}{*}{Risky Shortcut Gridworld}
& UCT & \texttt{bias=auto}, \texttt{epsilon=0.1} \\
& MaxUCT & \texttt{bias=auto}, \texttt{epsilon=0.1} \\
& PowerUCT & \texttt{bias=auto}, \texttt{epsilon=0.1}, \texttt{p=2.0} \\
& CATSO & \texttt{n\_atoms=100}, \texttt{optimism=2.0}, \texttt{p=1.0}, \texttt{cvar\_tau=0.05} \\
& PATSO & \texttt{max\_particles=100}, \texttt{optimism=2.0}, \texttt{p=1.0}, \texttt{cvar\_tau=0.05} \\

\midrule

\multirow{5}{*}{Betting Game}
& UCT & \texttt{bias=2}, \texttt{epsilon=0.05} \\
& MaxUCT & \texttt{bias=2}, \texttt{epsilon=0.05} \\
& PowerUCT & \texttt{bias=2}, \texttt{epsilon=0.05}, \texttt{p=2.0} \\
& CATSO & \texttt{n\_atoms=50}, \texttt{optimism=8.0}, \texttt{p=1.0}, \texttt{cvar\_tau=0.25} \\
& PATSO & \texttt{max\_particles=32}, \texttt{optimism=1.0}, \texttt{p=1.0}, \texttt{cvar\_tau=0.25} \\

\midrule

\multirow{5}{*}{Guarded Maze}
& UCT & \texttt{bias=auto}, \texttt{epsilon=0.1} \\
& MaxUCT & \texttt{bias=auto}, \texttt{epsilon=0.1} \\
& PowerUCT & \texttt{bias=auto}, \texttt{epsilon=0.1}, \texttt{p=2.0} \\
& CATSO & \texttt{n\_atoms=100}, \texttt{optimism=1.0}, \texttt{p=1.0}, \texttt{cvar\_tau=0.25} \\
& PATSO & \texttt{max\_particles=128}, \texttt{optimism=4.0}, \texttt{p=1.0}, \texttt{cvar\_tau=0.25} \\

\bottomrule
\end{tabular}
}
\end{table}

\begin{figure}[t]
\centering

\begin{subfigure}{0.28\textwidth}
    \centering
    \includegraphics[width=\linewidth]{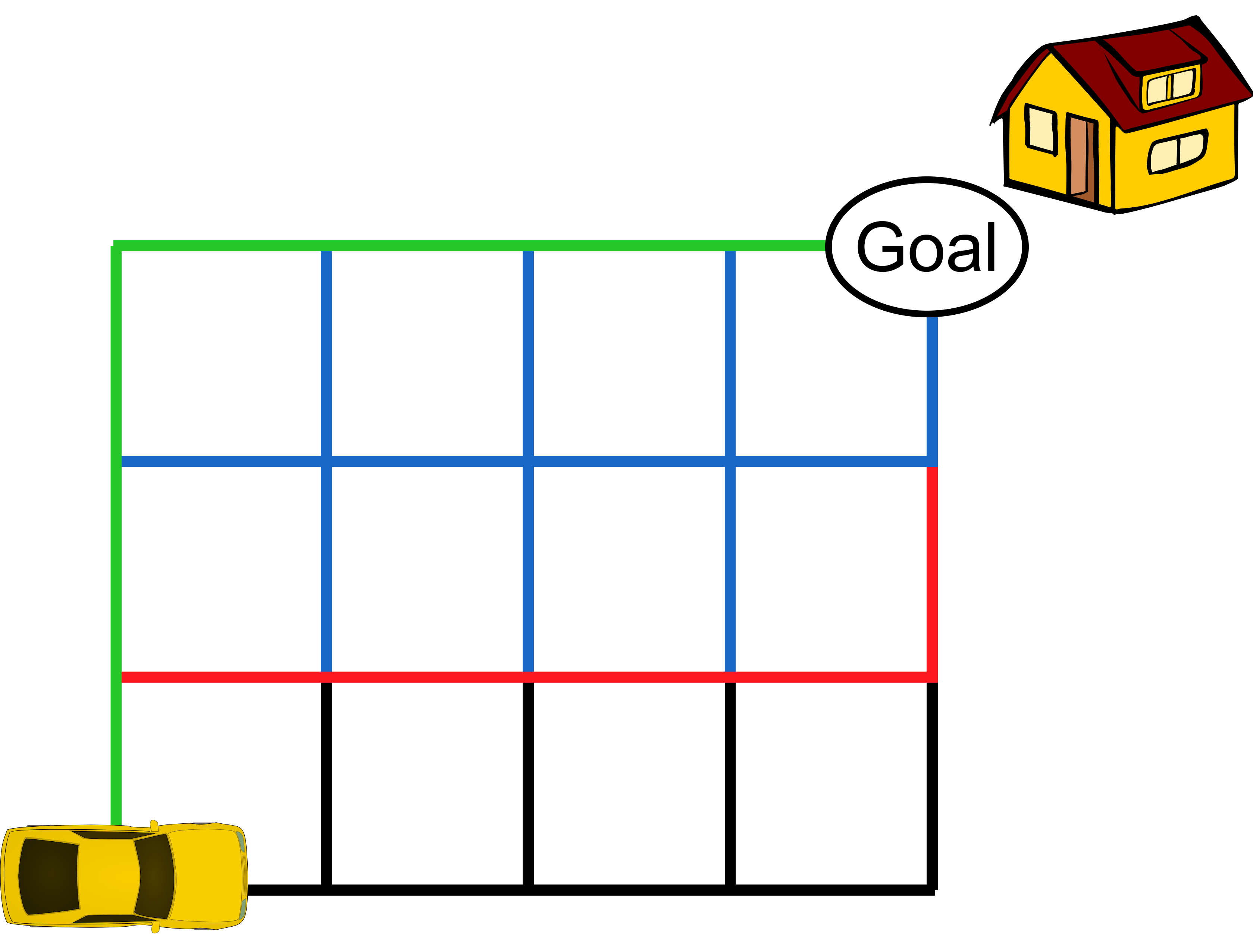}
    \caption{Autonomous vehicle}
    \label{fig:autonomous_vehicle}
\end{subfigure}
\hspace{0.03\textwidth}
\begin{subfigure}{0.28\textwidth}
    \centering
    \includegraphics[width=\linewidth]{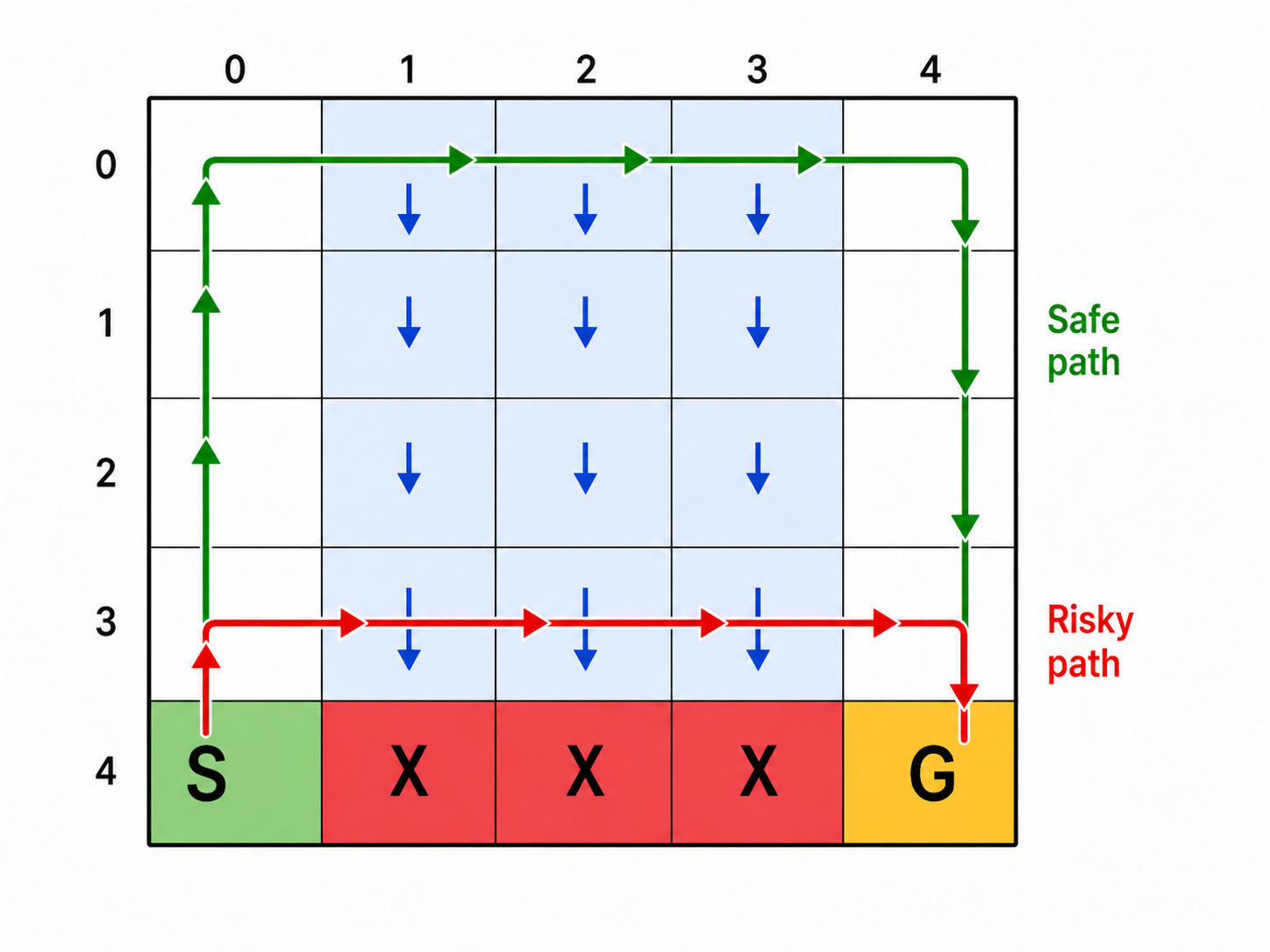}
    \caption{Risky shortcut gridworld}
    \label{fig:risky_shortcut_gridworld}
\end{subfigure}
\hspace{0.03\textwidth}
\begin{subfigure}{0.28\textwidth}
    \centering
    \includegraphics[width=\linewidth]{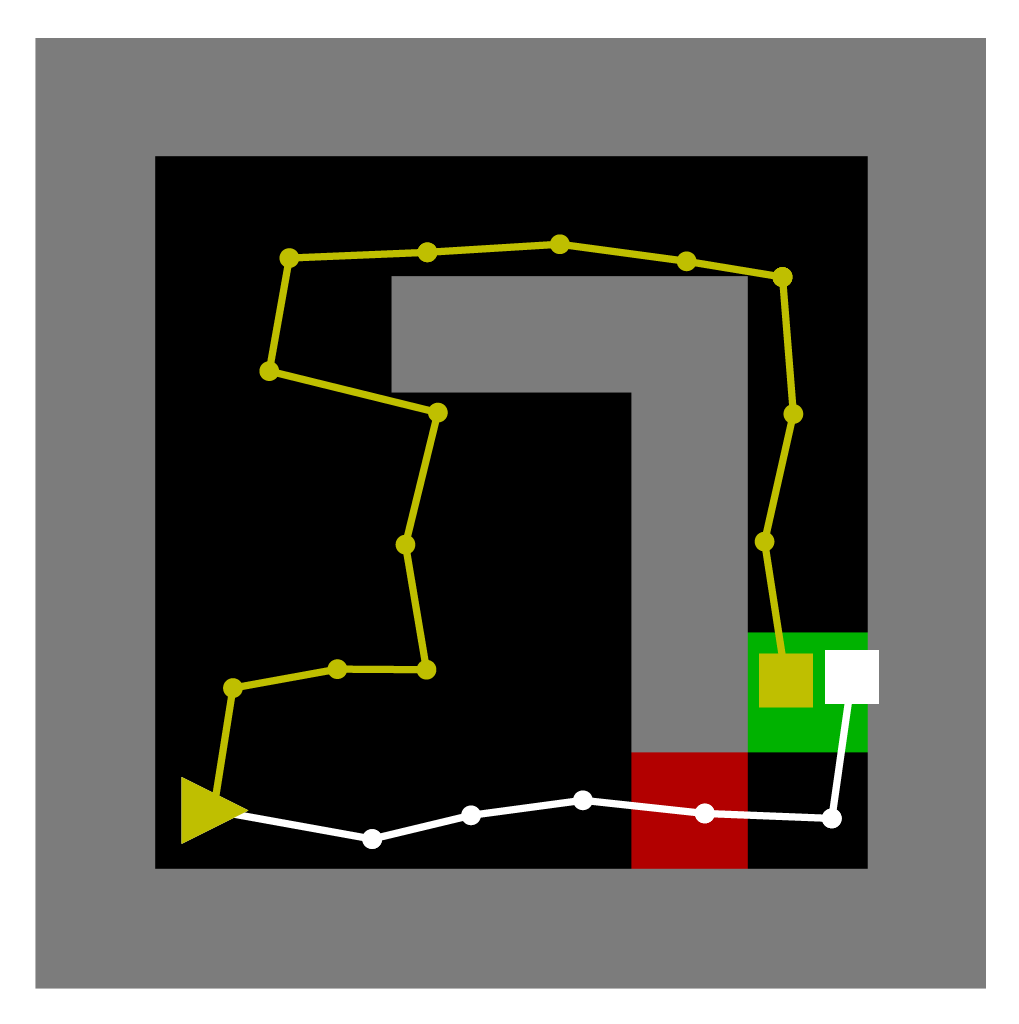}
    \caption{Guarded maze \citet{rigter2021risk}}

\end{subfigure}

\caption{Example risk-sensitive planning environments. Each contains a short risky route and a longer safer route; the risky route can have higher expected return but exposes the agent to catastrophic lower-tail outcomes. (2a) the map has black: highways, red: main roads, blue: streets, green: lanes (2b) S: start cell, G: goal cell, X: cliff cell, wind cell is indicated as the blue cell with an arrow points down. (2c) Red cell: guard cell, green cell: goal cell, gray cell: obstacle.}
\label{fig:environments}
\end{figure}

\paragraph{Autonomous Vehicle \citep{rigter2021risk}}
\begin{itemize}
    \item \textbf{Task:} A route-planning problem where the agent must navigate through a road network to reach a destination. See figure \ref{fig:autonomous_vehicle} for the map.
    \item \textbf{Goal:} Reach the goal while avoiding routes with rare but severe travel-cost outcomes.
    \item \textbf{Action Space:} Move \{Up, Down, Left, Right\}.
    \item \textbf{State Space:} Discrete $(x,y)$ position of the agent.
    \item \textbf{Reward:} Traversing the road pays a stochastic cost, based on the road types. Details are provided table \ref{tab:av_costs}.
    \item \textbf{Transition Dynamics:} The transition across the road is deterministic, while the cost type is sampled with probability $\{0.4, 0.2, 0.4\}$. The horizon is 7.
    \item \textbf{Catastrophic Feature:} Traversing a high-risk road and receiving the largest cost outcome.
\end{itemize}

\paragraph{Two-level Risky Treasure}
\begin{itemize}
    \item \textbf{Task:} A small sequential decision problem in which the agent chooses between a safe option and a risky treasure branch, and later decides whether to continue or abort after observing a hint.
    \item \textbf{Goal:} Maximize return by balancing the high upside of the treasure against the risk of failure.
    \item \textbf{Action Space:} At the root, the actions are \{safe, risky\}; after entering the risky branch, the actions are \{abort, proceed\}.
    \item \textbf{State Space:} Current decision stage, observed hint, and terminal status.
    \item \textbf{Reward:} The safe branch yields a reward of 2, while the risky branch yields either a high positive reward or a negative failure penalty, namely 500 or -10, respectively. The abort reward is -1.
    \item \textbf{Transition Dynamics:} Choosing the risky branch reveals a noisy hint about the hidden treasure state. The agent then chooses whether to abort safely or proceed to a stochastic success/failure outcome with success probability of $0.40$. The horizon is 2.
    \item \textbf{Catastrophic Feature:} Treasure failure after choosing to proceed, i.e., receiving the -10 reward.
\end{itemize}

\begin{table}[t]
\centering
\caption{Cost distribution used in the Autonomous Vehicle environment.}
\label{tab:av_costs}
\begin{tabular}{lccc}
\toprule
Road Type & Small Cost & Medium Cost & Large Cost \\
\midrule
Lane & 4 & 4 & 4 \\
Street & 2 & 3 & 5 \\
Main Road & 3 & 4 & 7 \\
Highway & 1 & 2 & 10 \\
\bottomrule
\end{tabular}
\end{table}

\paragraph{Risky Shortcut Gridworld}
\begin{itemize}
    \item \textbf{Task:} A cliff-style gridworld where the agent may choose between a short but dangerous shortcut and a longer safer path. See figure \ref{fig:risky_shortcut_gridworld} for the map.
    \item \textbf{Goal:} Reach the goal while avoiding the cliff.
    \item \textbf{Action Space:} Four movement actions: \{Up, Right, Down, Left\}.
    \item \textbf{State Space:} Grid position of the agent together with the current time step.
    \item \textbf{Reward:} An ordinary step gives reward -1, reaching the goal gives reward 5, and falling into the cliff gives reward -20.
    \item \textbf{Transition Dynamics:} The movement is deterministic, except when the agent is in a windy cell, there is a $2\%$ chance of being dragged down the hill. The horizon is 32.
    \item \textbf{Catastrophic Feature:} Entering a cliff cell.
\end{itemize}

\paragraph{Betting Game}
\begin{itemize}
    \item \textbf{Task:} A finite-horizon gambling problem where the agent repeatedly chooses how much of its current bankroll to bet.
    \item \textbf{Goal:} Maximize return while avoiding ruin.
    \item \textbf{Action Space:} Nine discrete betting actions, where action $a \in \{0,\dots,8\}$ corresponds to betting a fraction $a/8$ of the current bankroll.
    \item \textbf{State Space:} Current bankroll and current time step, i.e.\ $(b,t)$.
    \item \textbf{Reward:} Each action yields a stochastic gain or loss determined by the chosen bet fraction and the outcome of the bet.
    \item \textbf{Transition Dynamics:} After each action, the bankroll is updated according to the stochastic bet outcome. The episode terminates when the bankroll reaches zero, reaches the maximum bankroll, or the horizon is exhausted. The horizon is 6.
    \item \textbf{Catastrophic Feature:} Bankruptcy, i.e.\ the bankroll reaches zero.
\end{itemize}

\paragraph{Guarded Maze}
\begin{itemize}
    \item \textbf{Task:} A grid-navigation problem where the agent must move from a start cell to a goal cell while deciding whether to pass through a risky guarded region.
    \item \textbf{Goal:} Reach the goal with high return while avoiding dangerous guard outcomes.
    \item \textbf{Action Space:} Four movement actions: \{Up, Down, Left, Right\}.
    \item \textbf{State Space:} Agent position on the grid together with the current time step.
    \item \textbf{Reward:} Each step pays -1 reward, when in the guard cell, the reward is sampled in the set $\{-60, -30, 0, 30, 60\}$ in normal distribution.
    \item \textbf{Transition Dynamics:} Movement is constrained by the maze layout and walls. Entering the guarded region triggers a stochastic outcome, which may produce a large negative reward as described above. The horizon is 20.
    \item \textbf{Catastrophic Feature:} A severe negative guard outcome.
\end{itemize}

\subsection{Full Results}
For completeness, we show the full the detail experimental results of the experiment, which are already depicted in section \ref{sec:experiments}. Recall that the standard deviation of Two-Level Risky Treasure is zero because the algorithms converge in earlier simulation checkpoint due to its short in horizon and simple transition dynamics.

\begin{table}[t]
\centering
\caption{Performance across all environments and algorithms at final simulation budget. Best values per environment and metric are shown in bold.}
\label{tab:final-budget-all-results}
\scriptsize
\resizebox{\textwidth}{!}{%
\begin{tabular}{llccc}
\toprule
Environment & Algorithm & CVaR Regret & CVaR Return & Catastrophic Count\\
\midrule
\multirow{5}{*}{Two-Level Risky Treasure}
& UCT      & $12.000$ & $-10.000 \pm 0.000$ & $119.84$ \\
& MaxUCT   & $12.000$ & $-10.000 \pm 0.000$ & $120.24$ \\
& PowerUCT & $12.000$ & $-10.000 \pm 0.000$ & $120.26$ \\
& CATSO    & $0.020$ & $\mathbf{2.000 \pm 0.000}$ & $\mathbf{0.00}$ \\
& PATSO    & $\mathbf{0.000}$ & $\mathbf{2.000 \pm 0.000}$ & $\mathbf{0.00}$ \\
\midrule
\multirow{5}{*}{Autonomous Vehicle}
& UCT      & $7.574$ & $-31.702 \pm 1.229$ & $2.78$ \\
& MaxUCT   & $7.564$ & $-31.644 \pm 1.278$ & $3.36$ \\
& PowerUCT & $7.566$ & $-31.744 \pm 1.208$ & $3.28$ \\
& CATSO    & $\mathbf{1.428}$ & $-31.368 \pm 1.419$ & $3.36$ \\
& PATSO    & $5.059$ & $\mathbf{-28.240 \pm 0.323}$ & $\mathbf{0.00}$ \\
\midrule
\multirow{5}{*}{Risky Shortcut Gridworld}
& UCT      & $20.000$ & $-14.550 \pm 7.484$ & $8.08$ \\
& MaxUCT   & $20.000$ & $-13.950 \pm 7.698$ & $7.92$ \\
& PowerUCT & $20.000$ & $-13.300 \pm 8.103$ & $7.70$ \\
& CATSO    & $\mathbf{0.528}$ & $\mathbf{0.000 \pm 0.000}$ & $\mathbf{0.00}$ \\
& PATSO    & $2.748$ & $\mathbf{0.000 \pm 0.000}$ & $\mathbf{0.00}$ \\
\midrule
\multirow{5}{*}{Betting Game}
& UCT      & $0.062$ & $-15.996 \pm 0.008$ & $21.12$ \\
& MaxUCT   & $0.062$ & $-15.984 \pm 0.027$ & $23.06$ \\
& PowerUCT & $0.062$ & $-16.000 \pm 0.000$ & $22.70$ \\
& CATSO    & $0.024$ & $\mathbf{-5.613 \pm 5.630}$ & $\mathbf{2.02}$ \\
& PATSO    & $\mathbf{0.008}$ & $-13.714 \pm 2.407$ & $8.46$ \\
\midrule
\multirow{5}{*}{Guarded Maze}
& UCT      & $0.126$ & $-0.132 \pm 0.049$ & $61.70$ \\
& MaxUCT   & $0.126$ & $-0.136 \pm 0.051$ & $60.04$ \\
& PowerUCT & $0.125$ & $-0.135 \pm 0.050$ & $61.74$ \\
& CATSO    & $\mathbf{0.012}$ & $-0.042 \pm 0.001$ & $1.32$ \\
& PATSO    & $0.033$ & $\mathbf{-0.039 \pm 0.000}$ & $\mathbf{0.00}$ \\
\bottomrule
\end{tabular}%
}
\end{table}

\begin{figure}[t]
\centering
\noindent
\begin{algorithm}[H]
\caption{\algcatso Adaptive SimulateQ}
\DontPrintSemicolon
\SetAlgoNoEnd
\Fn{\SimQ{$s_h,a,t$}}{
    $s_{h+1}\sim\mathbb P(\cdot|s_h,a)$,
    $r_t(s_h,a)\sim\mathcal R(s_h,a,s_{h+1})$\\
    \eIf{node $s_{h+1}$ not expanded}{
        $\Rollout(s_{h+1})$
    }{
        $\SimV(s_{h+1},t)$
    }
    $T_{s_h,a}(t)=T_{s_h,a}(t)+1$\\
    $\Qbar_t(s_h,a)=r_t(s_h,a)+\gamma\Vhat^\tau_{T_{s_{h+1}}(t)}(s_{h+1})$\\
    
    \If{$\Qbar_t(s_h,a)\notin[Q_{\min}(s_h,a),Q_{\max}(s_h,a)]$}{
        $z_i^{\mathrm{old}}(s_h,a)\leftarrow z_i(s_h,a),\quad 0\le i\le N$\\
        $\alpha_i^{\mathrm{old}}(s_h,a)\leftarrow \alpha_i(s_h,a),\quad 0\le i\le N$\\

        $Q_{\max}(s_h,a)=\max\{\Qbar_t(s_h,a),Q_{\max}(s_h,a)\}$\\
        $Q_{\min}(s_h,a)=\min\{\Qbar_t(s_h,a),Q_{\min}(s_h,a)\}$\\
        $\triangle z=(Q_{\max}(s_h,a)-Q_{\min}(s_h,a))/N$\\
        $z_i(s_h,a)=Q_{\min}(s_h,a)+i\triangle z,\quad 0\le i\le N$\\

        $\widetilde\alpha_i(s_h,a)\leftarrow 1,\quad 0\le i\le N$\\
        \For{$j=0,\ldots,N$}{
            $i^\star \leftarrow \argmin_i |z_j^{\mathrm{old}}(s_h,a)-z_i(s_h,a)|$\\
            $\widetilde\alpha_{i^\star}(s_h,a)
            \leftarrow
            \widetilde\alpha_{i^\star}(s_h,a)
            + \alpha_j^{\mathrm{old}}(s_h,a)-1$
        }
        $\alpha_i(s_h,a)\leftarrow \widetilde\alpha_i(s_h,a),\quad 0\le i\le N$
    }

    $i \leftarrow \argmin_i|\Qbar_t(s_h,a)-z_i(s_h,a)|$\\
    $\alpha_i(s_h,a) \leftarrow \alpha_i(s_h,a)+1$\\
    Update $p(s_h,a)=[p_0(s_h,a),\ldots,p_N(s_h,a)]$
}
\end{algorithm}
\caption{Pseudocode for the adaptive distribution construction in \texttt{SimulateQ} of \algcatso.}
\label{alg:adaptive-simulateQ}
\end{figure}

\subsection{Sensitivity of CATSO/PATSO Grid Representation}
\label{app:sensitivity}

We study the sensitivity of CATSO and PATSO to the size of their
distributional representations. Specifically, we vary the number of
categorical atoms $N$ for CATSO and the maximum number of particles $U$
for PATSO. We evaluate the resulting computational cost, memory usage,
and planning performance. For comparison, UCT is shown as a scalar
baseline. All experiments use the adaptive-grid implementation of CATSO and the capped version of PATSO used in our practical experiments.

Figure~\ref{fig:sensitivity_compute} shows peak memory and search time
as a function of representation size. As expected, increasing $N$ for
CATSO leads to a higher computational and memory cost, since each
Q-edge maintains more categorical atoms and the CVaR computation must
process a larger support. Similarly, increasing the particle cap $U$
increases the computational cost of PATSO. The effect is particularly
pronounced for uncapped PATSO, whose support can grow with the number of
distinct backed-up returns. In contrast, capped PATSO provides a
bounded representation size and therefore offers a direct trade-off
between computational cost and distributional fidelity.

Figure~\ref{fig:sensitivity_performance} reports the resulting planning
performance. Overall, the performance of CATSO is relatively stable
across the tested values of $N$, although very small grids can provide
insufficient resolution for the underlying return distribution.
PATSO exhibits a similar trend with respect to $U$: increasing the
particle cap provides additional capacity to represent the empirical
distribution, but the performance gains become smaller once a
sufficiently large support is available. These results suggest that
moderate representation sizes can retain most of the risk-sensitive
planning benefit without incurring the full computational cost of a
large representation.

Finally, we measures the compression error of capped PATSO relative to uncapped PATSO. Figure~\ref{fig:sensitivity_compression} shows the compression error decreases as the particle cap $U$ increases, showing that a larger cap allows the capped representation to more closely match the empirical distribution maintained by uncapped PATSO. This agrees with the theoretical characterization in Remark~5.8, where the compression error contributes an additional approximation term to the CVaR estimate.

Overall, the sensitivity results demonstrate the practical trade-off
between representation capacity and computational efficiency. CATSO
trades a larger categorical grid for finer distributional resolution,
while PATSO trades a larger particle cap for lower compression error.
In practice, moderate values of $N$ and $U$ provide a reasonable
balance between planning performance, memory usage, and search time.

\begin{figure}[t]
    \centering
    \includegraphics[width=\linewidth]{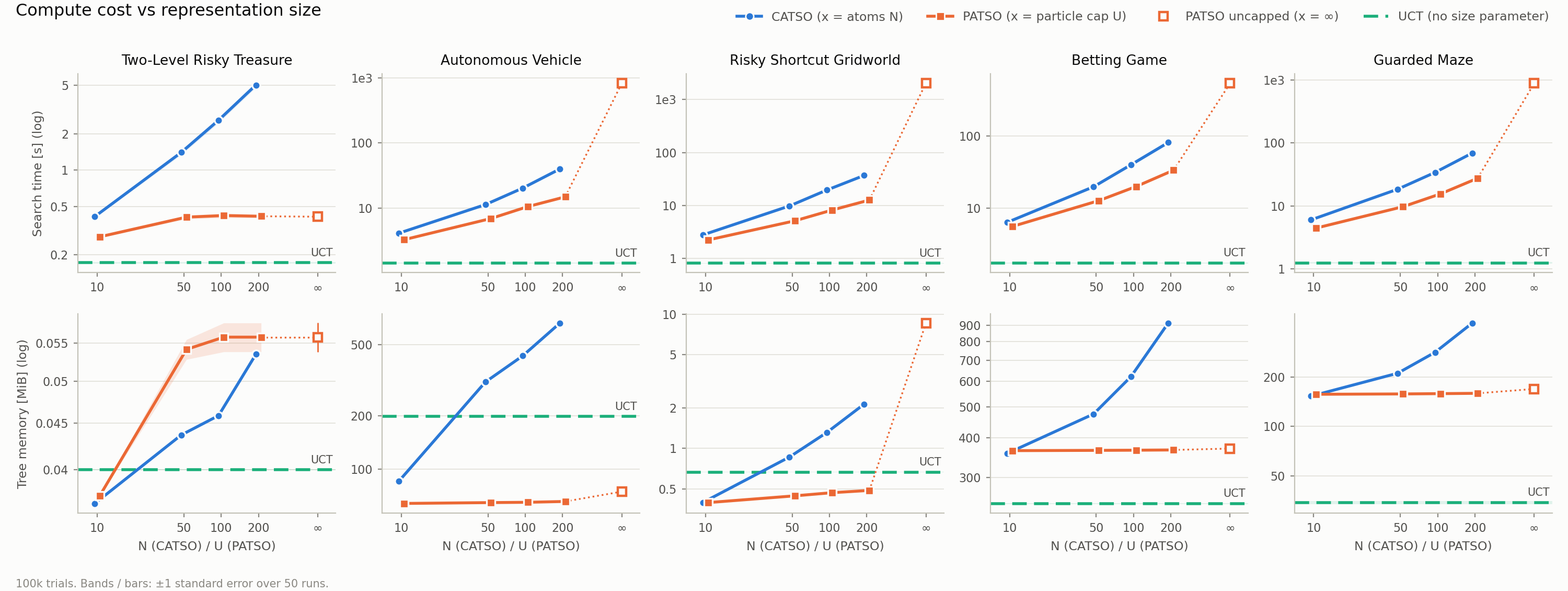}
    \caption{Peak memory and search time of CATSO, PATSO, and UCT as a
    function of the representation size $N$ and $U$.}
    \label{fig:sensitivity_compute}
\end{figure}

\begin{figure}[t]
    \centering
    \includegraphics[width=\linewidth]{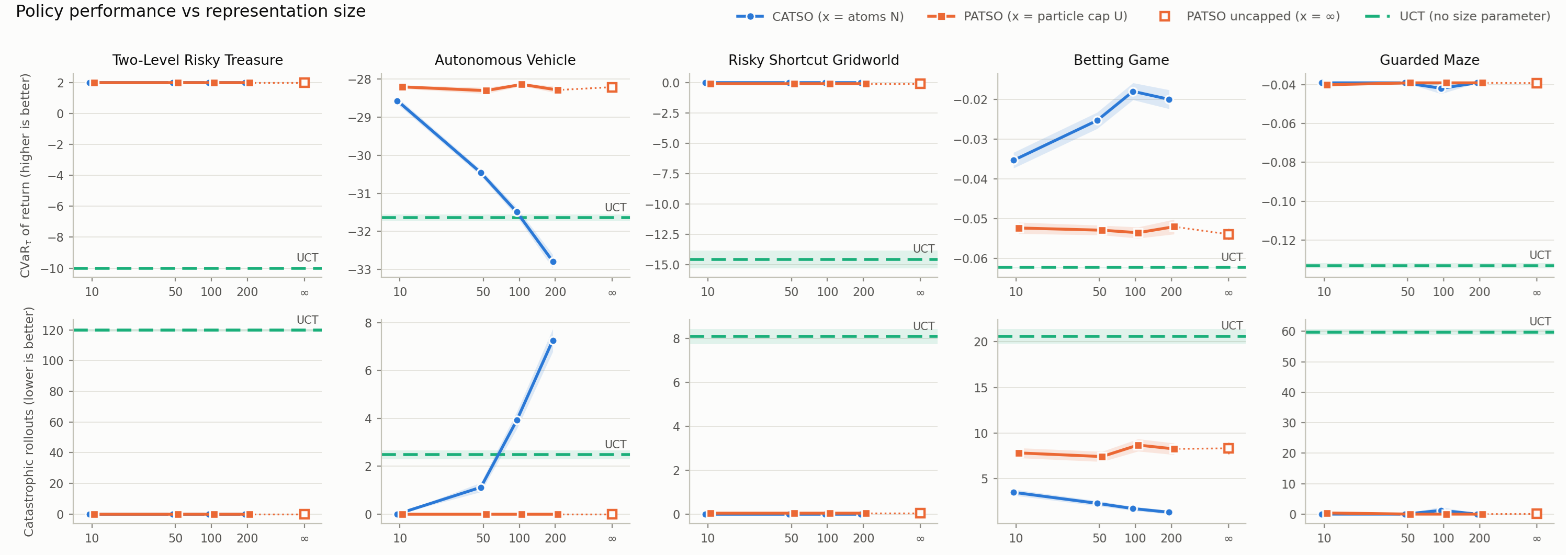}
    \caption{Planning performance of CATSO, PATSO, and UCT as a function
    of the representation size $N$ and $U$.}
    \label{fig:sensitivity_performance}
\end{figure}

\begin{figure}[t]
    \centering
    \includegraphics[width=\linewidth]{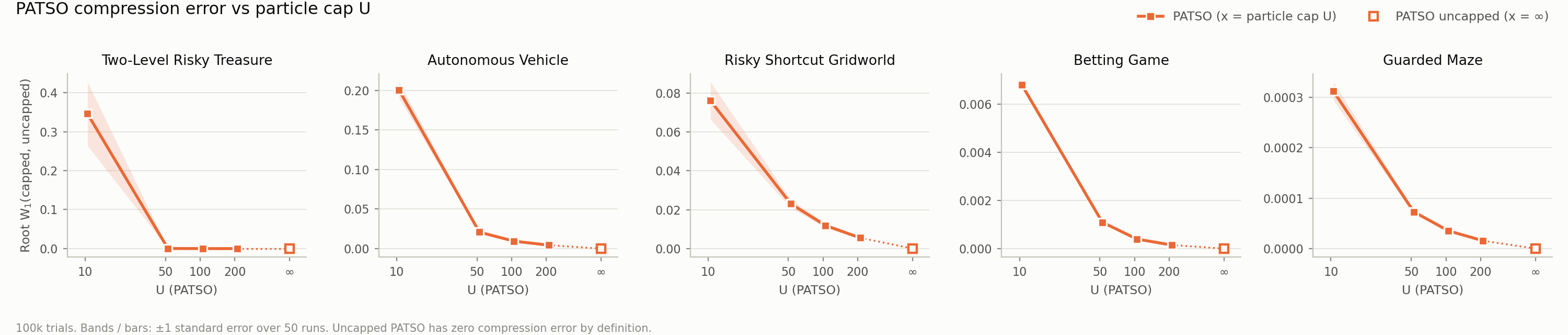}
    \caption{Compression quality of capped PATSO relative to uncapped
    PATSO as a function of the particle cap $U$.}
    \label{fig:sensitivity_compression}
\end{figure}

\section{Limitations}
\label{app:limitations}

\algcatso and \algpatso make MCTS risk-sensitive by storing empirical backup laws at
Q-edges while keeping V-nodes scalar.  This design is lightweight, but it also
has several limitations.

\paragraph{\algcatso discretization bias.}
\algcatso represents each Q-edge law on a fixed categorical grid.  This gives
bounded memory and simple Dirichlet updates, but it introduces a grid
approximation error.  If the grid is too coarse, sharp tail behavior or
multimodal return distributions may be poorly represented.  Adaptive grids can
improve empirical performance, but changing the grid requires reprojection.

\paragraph{\algpatso memory growth and compression bias.}
Uncapped \algpatso represents the empirical Q-edge law exactly, but its support can
grow linearly with the number of distinct backed-up returns.  Capped \algpatso
controls memory by merging particles, but this introduces a separate
$W_1$ compression error.  Since CVaR is sensitive to tail mass, careless
merging can distort the lower tail; the cap must therefore be chosen large
enough that the compression bias is below the statistical error.

\paragraph{Scalability beyond tabular finite-depth search.}
The present analysis is for tabular finite-depth MCTS.  Extending the approach
to large or continuous state spaces, learned dynamics, partial observability,
or neural value approximators requires additional approximation analysis.  In
particular, model error and function-approximation error would interact with
the Q-edge distributional estimates and the CVaR backup.
\section*{NeurIPS Paper Checklist}

\begin{enumerate}

\item {\bf Claims}
    \item[] Question: Do the main claims made in the abstract and introduction accurately reflect the paper's contributions and scope?
    \item[] Answer: \answerYes{} 
    \item[] Justification: All the claims made in the abstract and introduction are shown properly in sections \ref{sec:theory} and \ref{sec:experiments}.
    \item[] Guidelines:
    \begin{itemize}
        \item The answer \answerNA{} means that the abstract and introduction do not include the claims made in the paper.
        \item The abstract and/or introduction should clearly state the claims made, including the contributions made in the paper and important assumptions and limitations. A \answerNo{} or \answerNA{} answer to this question will not be perceived well by the reviewers. 
        \item The claims made should match theoretical and experimental results, and reflect how much the results can be expected to generalize to other settings. 
        \item It is fine to include aspirational goals as motivation as long as it is clear that these goals are not attained by the paper. 
    \end{itemize}

\item {\bf Limitations}
    \item[] Question: Does the paper discuss the limitations of the work performed by the authors?
    \item[] Answer: \answerYes{} 
    \item[] Justification: The limitations are discussed in appendix \ref{app:limitations}.
    \item[] Guidelines:
    \begin{itemize}
        \item The answer \answerNA{} means that the paper has no limitation while the answer \answerNo{} means that the paper has limitations, but those are not discussed in the paper. 
        \item The authors are encouraged to create a separate ``Limitations'' section in their paper.
        \item The paper should point out any strong assumptions and how robust the results are to violations of these assumptions (e.g., independence assumptions, noiseless settings, model well-specification, asymptotic approximations only holding locally). The authors should reflect on how these assumptions might be violated in practice and what the implications would be.
        \item The authors should reflect on the scope of the claims made, e.g., if the approach was only tested on a few datasets or with a few runs. In general, empirical results often depend on implicit assumptions, which should be articulated.
        \item The authors should reflect on the factors that influence the performance of the approach. For example, a facial recognition algorithm may perform poorly when image resolution is low or images are taken in low lighting. Or a speech-to-text system might not be used reliably to provide closed captions for online lectures because it fails to handle technical jargon.
        \item The authors should discuss the computational efficiency of the proposed algorithms and how they scale with dataset size.
        \item If applicable, the authors should discuss possible limitations of their approach to address problems of privacy and fairness.
        \item While the authors might fear that complete honesty about limitations might be used by reviewers as grounds for rejection, a worse outcome might be that reviewers discover limitations that aren't acknowledged in the paper. The authors should use their best judgment and recognize that individual actions in favor of transparency play an important role in developing norms that preserve the integrity of the community. Reviewers will be specifically instructed to not penalize honesty concerning limitations.
    \end{itemize}

\item {\bf Theory assumptions and proofs}
    \item[] Question: For each theoretical result, does the paper provide the full set of assumptions and a complete (and correct) proof?
    \item[] Answer: \answerYes{} 
    \item[] Justification: All assumptions are stated in section \ref{sec:setting}, \ref{sec:theory} and appendix \ref{app:cvar_wasserstein}. While the complete proofs are in section \ref{sec:theory}, appendix \ref{app:local_bandits}, \ref{app:mcts_convergence} and \ref{app:memory_capping}.
    \item[] Guidelines:
    \begin{itemize}
        \item The answer \answerNA{} means that the paper does not include theoretical results. 
        \item All the theorems, formulas, and proofs in the paper should be numbered and cross-referenced.
        \item All assumptions should be clearly stated or referenced in the statement of any theorems.
        \item The proofs can either appear in the main paper or the supplemental material, but if they appear in the supplemental material, the authors are encouraged to provide a short proof sketch to provide intuition. 
        \item Inversely, any informal proof provided in the core of the paper should be complemented by formal proofs provided in appendix or supplemental material.
        \item Theorems and Lemmas that the proof relies upon should be properly referenced. 
    \end{itemize}

    \item {\bf Experimental result reproducibility}
    \item[] Question: Does the paper fully disclose all the information needed to reproduce the main experimental results of the paper to the extent that it affects the main claims and/or conclusions of the paper (regardless of whether the code and data are provided or not)?
    \item[] Answer: \answerYes{} 
    \item[] Justification: Section \ref{sec:experiments} and Appendix \ref{app:experiments} provide all the information about the experiment procedures, environments, algorithms and their configuration in the experiment.
    \item[] Guidelines:
    \begin{itemize}
        \item The answer \answerNA{} means that the paper does not include experiments.
        \item If the paper includes experiments, a \answerNo{} answer to this question will not be perceived well by the reviewers: Making the paper reproducible is important, regardless of whether the code and data are provided or not.
        \item If the contribution is a dataset and\slash or model, the authors should describe the steps taken to make their results reproducible or verifiable. 
        \item Depending on the contribution, reproducibility can be accomplished in various ways. For example, if the contribution is a novel architecture, describing the architecture fully might suffice, or if the contribution is a specific model and empirical evaluation, it may be necessary to either make it possible for others to replicate the model with the same dataset, or provide access to the model. In general. releasing code and data is often one good way to accomplish this, but reproducibility can also be provided via detailed instructions for how to replicate the results, access to a hosted model (e.g., in the case of a large language model), releasing of a model checkpoint, or other means that are appropriate to the research performed.
        \item While NeurIPS does not require releasing code, the conference does require all submissions to provide some reasonable avenue for reproducibility, which may depend on the nature of the contribution. For example
        \begin{enumerate}
            \item If the contribution is primarily a new algorithm, the paper should make it clear how to reproduce that algorithm.
            \item If the contribution is primarily a new model architecture, the paper should describe the architecture clearly and fully.
            \item If the contribution is a new model (e.g., a large language model), then there should either be a way to access this model for reproducing the results or a way to reproduce the model (e.g., with an open-source dataset or instructions for how to construct the dataset).
            \item We recognize that reproducibility may be tricky in some cases, in which case authors are welcome to describe the particular way they provide for reproducibility. In the case of closed-source models, it may be that access to the model is limited in some way (e.g., to registered users), but it should be possible for other researchers to have some path to reproducing or verifying the results.
        \end{enumerate}
    \end{itemize}

\item {\bf Open access to data and code}
    \item[] Question: Does the paper provide open access to the data and code, with sufficient instructions to faithfully reproduce the main experimental results, as described in supplemental material?
    \item[] Answer: \answerYes{} 
    \item[] Justification: An anonymized .zip file containing the code, dependency specifications, and instructions necessary to faithfully reproduce the main experimental results is included in the supplementary material, in accordance with the NeurIPS Code and Data Submission Guidelines.
    \item[] Guidelines:
    \begin{itemize}
        \item The answer \answerNA{} means that paper does not include experiments requiring code.
        \item Please see the NeurIPS code and data submission guidelines (\url{https://neurips.cc/public/guides/CodeSubmissionPolicy}) for more details.
        \item While we encourage the release of code and data, we understand that this might not be possible, so \answerNo{} is an acceptable answer. Papers cannot be rejected simply for not including code, unless this is central to the contribution (e.g., for a new open-source benchmark).
        \item The instructions should contain the exact command and environment needed to run to reproduce the results. See the NeurIPS code and data submission guidelines (\url{https://neurips.cc/public/guides/CodeSubmissionPolicy}) for more details.
        \item The authors should provide instructions on data access and preparation, including how to access the raw data, preprocessed data, intermediate data, and generated data, etc.
        \item The authors should provide scripts to reproduce all experimental results for the new proposed method and baselines. If only a subset of experiments are reproducible, they should state which ones are omitted from the script and why.
        \item At submission time, to preserve anonymity, the authors should release anonymized versions (if applicable).
        \item Providing as much information as possible in supplemental material (appended to the paper) is recommended, but including URLs to data and code is permitted.
    \end{itemize}

\item {\bf Experimental setting/details}
    \item[] Question: Does the paper specify all the training and test details (e.g., data splits, hyperparameters, how they were chosen, type of optimizer) necessary to understand the results?
    \item[] Answer: \answerYes{} 
    \item[] Justification: Appendix \ref{app:experiments} clearly defines the environment setups, algorithm hyperparameters for both the main process and tuning process.
    \item[] Guidelines:
    \begin{itemize}
        \item The answer \answerNA{} means that the paper does not include experiments.
        \item The experimental setting should be presented in the core of the paper to a level of detail that is necessary to appreciate the results and make sense of them.
        \item The full details can be provided either with the code, in appendix, or as supplemental material.
    \end{itemize}

\item {\bf Experiment statistical significance}
    \item[] Question: Does the paper report error bars suitably and correctly defined or other appropriate information about the statistical significance of the experiments?
    \item[] Answer: \answerYes{} 
    \item[] Justification: The results are averaged over 50 seeds $\times$ 200 independent Monte Carlo runs, with standard deviation plotted in Figure \ref{fig:cvar-combined} and Table \ref{tab:final-budget-all-results} and empirical safety verified as detailed in Section~\ref{sec:experiments} and Appendix~\ref{app:experiments}.
    \item[] Guidelines:
    \begin{itemize}
        \item The answer \answerNA{} means that the paper does not include experiments.
        \item The authors should answer \answerYes{} if the results are accompanied by error bars, confidence intervals, or statistical significance tests, at least for the experiments that support the main claims of the paper.
        \item The factors of variability that the error bars are capturing should be clearly stated (for example, train/test split, initialization, random drawing of some parameter, or overall run with given experimental conditions).
        \item The method for calculating the error bars should be explained (closed form formula, call to a library function, bootstrap, etc.)
        \item The assumptions made should be given (e.g., Normally distributed errors).
        \item It should be clear whether the error bar is the standard deviation or the standard error of the mean.
        \item It is OK to report 1-sigma error bars, but one should state it. The authors should preferably report a 2-sigma error bar than state that they have a 96\% CI, if the hypothesis of Normality of errors is not verified.
        \item For asymmetric distributions, the authors should be careful not to show in tables or figures symmetric error bars that would yield results that are out of range (e.g., negative error rates).
        \item If error bars are reported in tables or plots, the authors should explain in the text how they were calculated and reference the corresponding figures or tables in the text.
    \end{itemize}

\item {\bf Experiments compute resources}
    \item[] Question: For each experiment, does the paper provide sufficient information on the computer resources (type of compute workers, memory, time of execution) needed to reproduce the experiments?
    \item[] Answer: \answerYes{} 
    \item[] Justification: Appendix \ref{app:experiments} explicitly details the compute resources (CPU Intel Core i9 14900K, 62 GB RAM) and the execution times for both the tuning experiment (2 to 4 hours per environment) and the main experiments (1 to 3 hours per environment).
    \item[] Guidelines:
    \begin{itemize}
        \item The answer \answerNA{} means that the paper does not include experiments.
        \item The paper should indicate the type of compute workers CPU or GPU, internal cluster, or cloud provider, including relevant memory and storage.
        \item The paper should provide the amount of compute required for each of the individual experimental runs as well as estimate the total compute. 
        \item The paper should disclose whether the full research project required more compute than the experiments reported in the paper (e.g., preliminary or failed experiments that didn't make it into the paper). 
    \end{itemize}
    
\item {\bf Code of ethics}
    \item[] Question: Does the research conducted in the paper conform, in every respect, with the NeurIPS Code of Ethics \url{https://neurips.cc/public/EthicsGuidelines}?
    \item[] Answer: \answerYes{} 
    \item[] Justification: The research focuses on foundational theoretical and algorithmic machine learning methodology (Sections~\ref{sec:theory}), which fully conforms to the NeurIPS Code of Ethics.
    \item[] Guidelines:
    \begin{itemize}
        \item The answer \answerNA{} means that the authors have not reviewed the NeurIPS Code of Ethics.
        \item If the authors answer \answerNo, they should explain the special circumstances that require a deviation from the Code of Ethics.
        \item The authors should make sure to preserve anonymity (e.g., if there is a special consideration due to laws or regulations in their jurisdiction).
    \end{itemize}

\item {\bf Broader impacts}
    \item[] Question: Does the paper discuss both potential positive societal impacts and negative societal impacts of the work performed?
    \item[] Answer: \answerNA{} 
    \item[] Justification: This is foundational theoretical research on risk-sensitive planning (Sections \ref{sec:algorithms} and \ref{sec:theory}) evaluated on synthetic distributions (Section \ref{sec:experiments}) presenting no immediate foreseeable negative societal impacts or direct paths to malicious applications.
    \item[] Guidelines:
    \begin{itemize}
        \item The answer \answerNA{} means that there is no societal impact of the work performed.
        \item If the authors answer \answerNA{} or \answerNo, they should explain why their work has no societal impact or why the paper does not address societal impact.
        \item Examples of negative societal impacts include potential malicious or unintended uses (e.g., disinformation, generating fake profiles, surveillance), fairness considerations (e.g., deployment of technologies that could make decisions that unfairly impact specific groups), privacy considerations, and security considerations.
        \item The conference expects that many papers will be foundational research and not tied to particular applications, let alone deployments. However, if there is a direct path to any negative applications, the authors should point it out. For example, it is legitimate to point out that an improvement in the quality of generative models could be used to generate Deepfakes for disinformation. On the other hand, it is not needed to point out that a generic algorithm for optimizing neural networks could enable people to train models that generate Deepfakes faster.
        \item The authors should consider possible harms that could arise when the technology is being used as intended and functioning correctly, harms that could arise when the technology is being used as intended but gives incorrect results, and harms following from (intentional or unintentional) misuse of the technology.
        \item If there are negative societal impacts, the authors could also discuss possible mitigation strategies (e.g., gated release of models, providing defenses in addition to attacks, mechanisms for monitoring misuse, mechanisms to monitor how a system learns from feedback over time, improving the efficiency and accessibility of ML).
    \end{itemize}
    
\item {\bf Safeguards}
    \item[] Question: Does the paper describe safeguards that have been put in place for responsible release of data or models that have a high risk for misuse (e.g., pre-trained language models, image generators, or scraped datasets)?
    \item[] Answer: \answerNA{} 
    \item[] Justification: The paper does not introduce or release high-risk models, language models, or scraped datasets (all data are simulated as per Appendix \ref{app:experiments}).
    \item[] Guidelines:
    \begin{itemize}
        \item The answer \answerNA{} means that the paper poses no such risks.
        \item Released models that have a high risk for misuse or dual-use should be released with necessary safeguards to allow for controlled use of the model, for example by requiring that users adhere to usage guidelines or restrictions to access the model or implementing safety filters. 
        \item Datasets that have been scraped from the Internet could pose safety risks. The authors should describe how they avoided releasing unsafe images.
        \item We recognize that providing effective safeguards is challenging, and many papers do not require this, but we encourage authors to take this into account and make a best faith effort.
    \end{itemize}

\item {\bf Licenses for existing assets}
    \item[] Question: Are the creators or original owners of assets (e.g., code, data, models), used in the paper, properly credited and are the license and terms of use explicitly mentioned and properly respected?
    \item[] Answer: \answerYes{} 
    \item[] Justification: All the environments represented in the paper, if adapted from existing works, are cited properly in Section \ref{sec:experiments}.
    \item[] Guidelines:
    \begin{itemize}
        \item The answer \answerNA{} means that the paper does not use existing assets.
        \item The authors should cite the original paper that produced the code package or dataset.
        \item The authors should state which version of the asset is used and, if possible, include a URL.
        \item The name of the license (e.g., CC-BY 4.0) should be included for each asset.
        \item For scraped data from a particular source (e.g., website), the copyright and terms of service of that source should be provided.
        \item If assets are released, the license, copyright information, and terms of use in the package should be provided. For popular datasets, \url{paperswithcode.com/datasets} has curated licenses for some datasets. Their licensing guide can help determine the license of a dataset.
        \item For existing datasets that are re-packaged, both the original license and the license of the derived asset (if it has changed) should be provided.
        \item If this information is not available online, the authors are encouraged to reach out to the asset's creators.
    \end{itemize}

\item {\bf New assets}
    \item[] Question: Are new assets introduced in the paper well documented and is the documentation provided alongside the assets?
    \item[] Answer: \answerYes{} 
    \item[] Justification: All new environments and algorithms are documented in detail in supplementary and appendix \ref{app:experiments}.
    \item[] Guidelines:
    \begin{itemize}
        \item The answer \answerNA{} means that the paper does not release new assets.
        \item Researchers should communicate the details of the dataset\slash code\slash model as part of their submissions via structured templates. This includes details about training, license, limitations, etc. 
        \item The paper should discuss whether and how consent was obtained from people whose asset is used.
        \item At submission time, remember to anonymize your assets (if applicable). You can either create an anonymized URL or include an anonymized zip file.
    \end{itemize}

\item {\bf Crowdsourcing and research with human subjects}
    \item[] Question: For crowdsourcing experiments and research with human subjects, does the paper include the full text of instructions given to participants and screenshots, if applicable, as well as details about compensation (if any)? 
    \item[] Answer: \answerNA{} 
    \item[] Justification: The paper studies risk awareness in MCTS planning algorithm and does not involve human subjects.
    \item[] Guidelines:
    \begin{itemize}
        \item The answer \answerNA{} means that the paper does not involve crowdsourcing nor research with human subjects.
        \item Including this information in the supplemental material is fine, but if the main contribution of the paper involves human subjects, then as much detail as possible should be included in the main paper. 
        \item According to the NeurIPS Code of Ethics, workers involved in data collection, curation, or other labor should be paid at least the minimum wage in the country of the data collector. 
    \end{itemize}

\item {\bf Institutional review board (IRB) approvals or equivalent for research with human subjects}
    \item[] Question: Does the paper describe potential risks incurred by study participants, whether such risks were disclosed to the subjects, and whether Institutional Review Board (IRB) approvals (or an equivalent approval/review based on the requirements of your country or institution) were obtained?
    \item[] Answer: \answerNA{} 
    \item[] Justification: The research utilizes synthetic environments (Appendix \ref{app:experiments}) and does not involve human subjects.
    \item[] Guidelines:
    \begin{itemize}
        \item The answer \answerNA{} means that the paper does not involve crowdsourcing nor research with human subjects.
        \item Depending on the country in which research is conducted, IRB approval (or equivalent) may be required for any human subjects research. If you obtained IRB approval, you should clearly state this in the paper. 
        \item We recognize that the procedures for this may vary significantly between institutions and locations, and we expect authors to adhere to the NeurIPS Code of Ethics and the guidelines for their institution. 
        \item For initial submissions, do not include any information that would break anonymity (if applicable), such as the institution conducting the review.
    \end{itemize}

\item {\bf Declaration of LLM usage}
    \item[] Question: Does the paper describe the usage of LLMs if it is an important, original, or non-standard component of the core methods in this research? Note that if the LLM is used only for writing, editing, or formatting purposes and does \emph{not} impact the core methodology, scientific rigor, or originality of the research, declaration is not required.
    \item[] Answer: \answerNA{} 
    \item[] Justification: LLMs are not used as a component of the core theoretical methodology or algorithm design (sections \ref{sec:algorithms} and \ref{sec:theory})
    \item[] Guidelines:
    \begin{itemize}
        \item The answer \answerNA{} means that the core method development in this research does not involve LLMs as any important, original, or non-standard components.
        \item Please refer to our LLM policy in the NeurIPS handbook for what should or should not be described.
    \end{itemize}

\end{enumerate}

\end{document}